\documentclass[twoside,11pt]{article}

\usepackage{blindtext}

\usepackage[abbrvbib, preprint]{jmlr2e}

\usepackage{soul}
\usepackage{xcolor}  
\usepackage{lipsum}
\usepackage{graphicx}
\usepackage{multicol}    	 
\usepackage{multirow}
\graphicspath{ {./images/} }

\usepackage{mathrsfs}
\usepackage{amsmath,amssymb,amsfonts,amsbsy}
\usepackage{algorithm}
\usepackage{algpseudocode}
\usepackage{comment}

\newtheorem{fact}{Fact}
\newtheorem{assumption}{Assumption}

\usepackage{colortbl}

\newcommand{\abs}[1]{\left\lvert #1 \right\rvert}
\newcommand{\norm}[1]{\left\lVert #1 \right\rVert}
\newcommand{\brk}[1]{\left[ #1 \right]}
\newcommand{\cbrk}[1]{\left\{ #1 \right\}}
\newcommand{\prt}[1]{\left( #1 \right)}

\newcommand{\cI}{\mathcal{I}}
\newcommand{\cF}{\mathcal{F}}
\newcommand{\cH}{\mathcal{H}}
\newcommand{\cM}{\mathcal{M}}
\newcommand{\cS}{\mathcal{S}}
\newcommand{\cA}{\mathcal{A}}

\newcommand{\cO}{\mathcal{O}}
\newcommand{\sA}{\mathscr{A}}
\newcommand{\sD}{\mathscr{D}}
\newcommand{\bP}{\mathbb{P}}
\newcommand{\bE}{\mathbb{E}}
\newcommand{\bR}{\mathbb{R}}

\newcommand{\rO}{\mathrm{O}}
\newcommand{\tb}[1]{\textbf{ #1}}

\newcommand{\eu}[1]{\text{EU}\prt{#1}}
\newcommand{\erm}[1]{U_{\beta}\prt{#1}}
\newcommand{\cT}{\mathcal{T}}

\usepackage{lastpage}
\jmlrheading{23}{2022}{1-\pageref{LastPage}}{1/21; Revised 5/22}{9/22}{21-0000}{Hao Liang and Zhi-Quan Luo}

\ShortHeadings{Bridging Distributional and Risk-sensitive Reinforcement Learning}{Hao Liang and Zhi-Quan Luo}
\firstpageno{1}

\begin{document}

\title{Bridging Distributional and Risk-sensitive Reinforcement Learning with Provable Regret Bounds}

\author{\name Hao Liang \email haoliang1@link.cuhk.edu.cn \\
       \addr School of Science and Engineering\\
       The Chinese University of Hong Kong, Shenzhen \\
       \AND
       \name Zhi-Quan Luo \email luozq@cuhk.edu.cn \\
       \addr School of Science and Engineering\\
       The Chinese University of Hong Kong, Shenzhen\\}

\editor{My editor}

\maketitle

\begin{abstract}
We study the regret guarantee for risk-sensitive reinforcement learning (RSRL) via distributional reinforcement learning (DRL) methods. In particular, we consider finite episodic Markov decision processes whose objective is the entropic risk measure (EntRM) of return. By leveraging a key property of the EntRM, the independence property, we establish the risk-sensitive distributional dynamic programming framework. We then propose two novel DRL algorithms that implement optimism through two different schemes, including a model-free one and a model-based one. 

We prove that they both attain $\tilde{\mathcal{O}}\prt{\frac{\exp(|\beta| H)-1}{|\beta|}H\sqrt{S^2AK}}$ regret upper bound, where $S$, $A$, $K$, and $H$ represent the number of states, actions, episodes, and the time horizon, respectively. It matches RSVI2 proposed in \cite{fei2021exponential}, with  novel distributional analysis. To the best of our knowledge, this is the first regret analysis that bridges DRL and RSRL in terms of sample complexity. 

Acknowledging the computational inefficiency associated with the model-free DRL algorithm, we propose an alternative DRL algorithm with distribution representation. This approach not only maintains the established regret bounds but also significantly amplifies computational efficiency.

We also  prove a tighter minimax lower bound of $\Omega\prt{\frac{\exp(\beta H/6)-1}{\beta H}H\sqrt{SAT}}$ for the $\beta>0$ case, which recovers the tight lower bound $\Omega(H\sqrt{SAT})$ in the risk-neutral setting. 
\end{abstract}


\section{Introduction}
Standard reinforcement learning (RL) seeks to find an optimal policy that maximizes the expected return \citep{sutton2018reinforcement}. This approach is often referred to as risk-neutral RL, as it focuses on the mean functional of the return distribution. However, in high-stakes applications, such as finance \citep{davis2008risk,bielecki2000risk}, medical treatment \citep{ernst2006clinical}, and operations research \citep{delage2010percentile}, decision-makers are often risk-sensitive and aim to maximize a risk measure of the return distribution.

Since the pioneering work of \citet{howard1972risk}, risk-sensitive reinforcement learning (RSRL) based on the exponential risk measure (EntRM) has been applied to a wide range of domains \citep{shen2014risk, nass2019entropic, hansen2011robustness}. EntRM offers a trade-off between the expected return and its variance and allows for the adjustment of risk-sensitivity through a risk parameter. However, existing approaches typically require complicated algorithmic designs to handle the non-linearity of EntRM.

Distributional reinforcement learning (DRL) has demonstrated superior performance over traditional methods in some challenging tasks under a risk-neutral setting \citep{bellemare2017distributional, dabney2018distributional, dabney2018implicit}. Unlike value-based approaches, DRL learns the entire return distribution instead of a real-valued value function. Given the distributional information, it is natural to leverage it to optimize a risk measure other than expectation \citep{dabney2018implicit, singh2020improving, ma2020dsac}. Despite the intrinsic connection between DRL and RSRL, existing works on RSRL via DRL approaches lack regret analysis \citep{dabney2018implicit, ma2021conservative, achab2021robustness}.  Consequently, it is challenging to evaluate and improve these DRL algorithms in terms of sample efficiency. Additionally, DRL can be computationally demanding as return distributions are typically infinite-dimensional objects. This complexity raises a pertinent question:
\begin{center}
\textit{Is it feasible for DRL to attain near-optimal regret in RSRL while preserving computational efficiency?}
\end{center}
In this work, we answer this question positively by providing computationally efficient DRL algorithms with regret guarantee. We have developed two types of DRL algorithms, both designed to be computationally efficient, and equipped with principled exploration strategies tailored for tabular EntRM-MDP. Notably, these proposed algorithms apply the principle of optimism in the face of uncertainty (OFU) at a distributional level, effectively balancing the exploration-exploitation dilemma. By conducting the first regret analysis in the field of DRL, we  bridge the gap between computationally efficient DRL and RSRL, especially in terms of sample complexity. Our work  paves the way for deeper understanding and improving the efficiency of RSRL through the lens of distributional approaches.

\subsection{Related Work}  
Related work in DRL has rapidly grown since \cite{bellemare2017distributional}, with numerous studies aiming to improve performance in the risk-neutral setting \citep[see][]{rowland2018analysis,dabney2018distributional,dabney2018implicit, barth2018distributed,yang2019fully,lyle2019comparative,zhang2021distributional}. However, only a few works have considered risk-sensitive behavior, including \cite{dabney2018implicit,ma2021conservative,achab2021robustness}. None of these works have addressed sample complexity.

A large body of work has investigated RSRL using the EntRM in various settings \citep{borkar2001sensitivity, borkar2002q,borkar2002risk,borkar2010learning,bauerle2014more,di2000infinite,di2007infinite,cavazos2011discounted, jaskiewicz2007average,ma2020dsac,mihatsch2002risk,osogami2012robustness,patek2001terminating,shen2013risk,shen2014risk}. However, these works either assume known transition and reward or consider infinite-horizon settings without considering sample complexity.

Our work is related to two recent studies by \cite{fei2020risk} and \cite{fei2021exponential} in the same setting. \cite{fei2020risk} introduced the first regret-guaranteed algorithms for risk-sensitive episodic Markov decision processes (MDPs), but their regret upper bounds contain an unnecessary factor of $\exp(|\beta| H^2)$ and their lower bound proof contains errors, leading to a weaker bound. While \cite{fei2021exponential} refined their algorithm by introducing a doubly decaying bonus that effectively removes the $\exp(|\beta| H^2)$ factor \footnote{A detailed comparison with \cite{fei2021exponential} is given in Section \ref{sec:dis}.}, the issue with the lower bound was not resolved. \cite{achab2021robustness} proposed a risk-sensitive deep deterministic policy gradient framework, but their work is fundamentally different from ours as they consider the conditional value at risk  and focus on discounted MDP with infinite horizon settings.  Moreover, \cite{achab2021robustness} assumes that the model is known. 

\subsection{Contributions}
This paper makes the following primary contributions: \\
1. Formulation of a Risk-Sensitive Distributional Dynamic Programming (RS-DDP) framework. This framework introduces a distributional Bellman optimality equation tailored for EntRM-MDP, leveraging the independence property of the EntRM. \\
2. Proposal of computationally efficient DRL algorithms that enforce the OFU principle in a distributional fashion, along with regret upper bounds of $\tilde{\mathcal{O}}\prt{\frac{\exp(|\beta| H)-1}{|\beta|}H\sqrt{S^2AK}}$. The DRL algorithms not only outperform existing methods empirically but are also supported by theoretical justifications. Furthermore, this marks the first instance of analyzing a DRL algorithm within a finite episodic EntRM-MDP setting.  \\
3. Filling of gaps in the lower bound in \cite{fei2020risk}, resulting in a tight lower bound of $\Omega\prt{\frac{\exp(\beta H/6)-1}{\beta }\sqrt{SAT}}$ for $\beta>0$. This lower bound is dependent of $S$ and $A$ and recovers the tight lower bound in the risk-neutral setting (as $\beta\rightarrow0$).
\begin{table}[h!]
    \centering
\begin{tabular}{|c|c|c|c|}
\hline  Algorithm & Regret bound& Time & Space \\
\hline   
\texttt{RSVI}   &  $\tilde{\mathcal{O}}\prt{\exp(|\beta|H^2)\frac{\exp(|\beta| H)-1}{|\beta|}\sqrt{HS^2AT}}$ & \multirow{3}{*}{$\mathcal{O}\left(T S^2 A\right)$} & \multirow{3}{*}{$\mathcal{O}\left(HSA+T\right)$} \\
\cline{1-2}
\texttt{RSVI2}   &  \multirow{4}{*}{$\tilde{\mathcal{O}}\prt{\frac{\exp(|\beta| H)-1}{|\beta|}\sqrt{HS^2AT}}$} &  &  \\
\cline{1-1}
  \cellcolor{lightgray} \texttt{RODI-Rep}  &  & & \\
  \cline{1-1} \cline{3-4}
  \cellcolor{lightgray} \texttt{RODI-MF}  &  & $\cO(KS^H)$ &$\cO(S^H)$ \\
  \cline{1-1} \cline{3-4}
   \cellcolor{lightgray} \texttt{RODI-MB}  &  & $\mathcal{O}\left(T S^2 A\right)$ &$\cO(HS^2A)$ \\
\hline  lower bound &  $\Omega\prt{\frac{\exp(\beta H/6)-1}{\beta }\sqrt{SAT}}$ & - & - \\
\hline
\end{tabular}
    \caption{Regret bounds and computational complexity comparisons.}
    \label{tab:comp}
\end{table}

\section{Preliminaries}
\label{sec:pre}
We provide the technical background in this section. 
\subsection{Notations}
We write $[M:N]\triangleq\{M,M+1,...,N\}$ and $[N]\triangleq[1:N]$ for any positive integers $M\leq N$. We adopt the convention that $\sum^m_{i=n}a_i\triangleq0$ if $n>m$ and $\prod^m_{i=n}a_i\triangleq1$ if $n>m$. We use $\mathbb{I}\{\cdot\}$  to denote the indicator function. For any $x\in\mathbb{R}$, we define $[x]^+\triangleq\max\{x,0\}$. We denote by $\delta_c$ the Dirac measure at $c$. We denote by $\mathscr{D}(a,b)$, $\mathscr{D}_M$ and $\mathscr{D}$  the set of distributions supported on $[a,b]$, $[0,M]$ and the set of all distributions respectively. We use $(x_1,x_2;p)$ to denote a binary r.v. taking values $x_1$ and $x_2$ with probability $1-p$ and $p$. For a discrete set $x=\cbrk{x_1,\cdots,x_n}$ and a probability vector $p=(p_1,\cdots,p_n)$, the notation $(x,p)$ represents the discrete distribution with $\bP(X=x_i)=p_i$. For a discrete distribution $\eta=(x,p)$, we use $|\eta|=|x|$ to denote the number of atoms of the distribution $\eta$. We use $\tilde{\mathcal{O}}(\cdot)$ to denote $\mathcal{O}(\cdot)$ omitting logarithmic factors. A table of notation is provided in Appendix \ref{app:not}.
\subsection{Episodic MDP}
An episodic MDP is identified by  $\mathcal{M}\triangleq(\mathcal{S},\mathcal{A},(P_h)_{h\in[H]},(r_h)_{h\in[H]},H)$, where $\mathcal{S}$ is the state space, $\mathcal{A}$ the action space, $P_h:\mathcal{S}\times\mathcal{A}\rightarrow\Delta(\mathcal{S})$ the probability transition kernel at step $h$, $r_h:\mathcal{S}\times\mathcal{A}\rightarrow[0,1]$  the collection of  reward functions at step $h$ and $H$  the length of one episode. The agent interacts with the environment for $K$ episodes. At the beginning of episode $k$, Nature selects an initial state $s^k_1$ arbitrarily. In step $h$, the agent takes action $a^k_h$ and observes random reward $r_h(s^k_h,a^k_h)$ and reaches the next state $s^k_{h+1}\sim P_h(\cdot|s^k_h,a^k_h)$. The episode terminates at $H+1$ with $r_{H+1}=0$, then the agent proceeds to next episode.

For each $(k,h)\in[K]\times[H]$, we denote by $\cH^k_h\triangleq\left(s_{1}^{1}, a_{1}^{1}, s_{2}^{1}, a_{2}^{1}, \ldots, s_{H}^{1},a_{H}^{1}, \ldots, s_{h}^{k}, a_{h}^{k}\right)$ the (random) history up to step $h$ of episode $k$. We define $\cF_k\triangleq\mathcal{H}^{k-1}_H$  as the history up to episode $k-1$. We describe the interaction between the algorithm and MDP in two levels. In the level of episode, we define an algorithm as a sequence of function $\sA\triangleq(\sA_k)_{k\in[K]}$, each mapping $\cF_k$ to a policy $\sA_k(\cF_k)\in\Pi$. We denote by  $\pi^k\triangleq\sA_k(\mathcal{F}_k)$ the policy at episode $k$. In the level of step, a (deterministic) policy $\pi$  is  a sequence of functions $\pi=(\pi_h)_{h\in[H]}$ with $\pi_h:\mathcal{S}\rightarrow \mathcal{A}$.  
\subsection{Risk Measure}
Consider two random variables $X\sim F$ and $Y\sim G$. We assert that $Y$ dominates $X$, and correspondingly, $G$ dominates $F$, denoted as $Y\succeq X$ and $G\succeq F$, if and only if for every real number $x$, the inequality $F(x)\ge G(x)$ holds true.  A risk measure, $\rho$, is a function mapping a set of random variables, denoted as $\mathscr{X}$, to the real numbers. This mapping adheres to several crucial properties:
\begin{itemize}
    \item \textbf{M}onotonicity (\textbf{M}):  $X \preceq Y \Rightarrow \rho(X) \leq \rho(Y),\ \forall X, Y\in\mathscr{X}$,
    \item \textbf{T}ranslation-\textbf{i}nvariance (\textbf{TI}): $\rho(X+c)=\rho(X)+c,\ \forall X\in\mathscr{X},\ \forall c\in\mathbb{R}$,
    \item \textbf{D}istribution-\textbf{i}nvariance (\textbf{DI}): $F_{X_1}=F_{X_2} \Rightarrow \rho(X_1)=\rho(X_2)$.
\end{itemize}
A mapping $\rho: \mathscr{X} \rightarrow \mathbb{R}$ qualifies as a risk measure if it satisfies both (\textbf{M}) and (\textbf{TI}). Additionally, a risk measure that also adheres to \textbf{(DI)} is termed a distribution-invariant risk measure. This paper focuses exclusively on distribution-invariant risk measures, allowing us to simplify notation by denoting $\rho(F_{X})$ as $\rho(X)$.

We direct our attention to EntRM, a prominent risk measure in domains requiring risk-sensitive decision-making, such as mathematical finance \citep{follmer2016stochastic}, Markovian decision processes \citep{bauerle2014more}. For a random variable $X \sim F$ and a non-zero coefficient $\beta$, the EntRM is defined as:
\[  U_{\beta}(X)\triangleq \frac{1}{\beta}\log\prt{\bE_{X\sim F}\brk{e^{\beta X}}}=\frac{1}{\beta}\log\prt{\int_{\bR}e^{\beta X}dF(x)}.\]
Given that EntRM satisfies \textbf{(DI)}, we can denote $U_{\beta}(F)$ as $U_{\beta}(X)$ for $X \sim F$. When $\beta$ possesses a small absolute value, employing Taylor's expansion yields
\begin{equation}
    \label{eqt:erm_exp}
    U_{\beta}(X)=\mathbb{E}[X]+\frac{\beta}{2}\mathbb{V}[X]+\mathcal{O}(\beta^2).
\end{equation}
Therefore, a decision-maker aiming to maximize the EntRM value demonstrates risk-seeking behavior (preferring higher uncertainty in $X$) when $\beta > 0$, and risk-averse behavior (preferring lower uncertainty in $X$) when $\beta < 0$. The absolute value of $\beta$ dictates the sensitivity to risk, with the measure converging to the mean functional as $\beta$ approaches zero.

\subsection{Risk-neutral Distributional Dynamic Programming Revisited}
\cite{bellemare2017distributional, rowland2018analysis} have discussed the \emph{infinite-horizon} distributional dynamic programming in the \emph{risk-neutral} setting, which will be referred to as the classical DDP. Now we adapt their results to the finite horizon setting. We define the return for a policy $\pi$ starting from state-action pair $(s,a)$ at step $h$ 
\begin{align*}
Z^{\pi}_h(s,a)\triangleq\sum^H_{h^{\prime}=h}r_{h^{\prime}}(s_{h^{\prime}},a_{h^{\prime}}),\ s_{h}=s, a_{h^{\prime}}=\pi_{h^{\prime}}(s_{h^{\prime}}),  s_{h^{\prime}+1} \sim P_{h^{\prime}}(\cdot|s_{h^{\prime}},a_{h^{\prime}}).
\end{align*}
Define $Y^{\pi}_h(s)\triangleq Z^{\pi}_h(s,\pi_h(s))$, then it is immediate that
\[ Z^{\pi}_h(s,a)=r_h(s,a)+ Y^{\pi}_{h+1}(S^{\prime}), S^{\prime}\sim P_h(\cdot|s,a).\]
There are two sources of randomness in $Z^{\pi}_h(s,a)$: the transition $P_h^{\pi}$ and the next-state return $Y^{\pi}_{h+1}$. Denote by $\nu^{\pi}_h(s)$ and $\eta^{\pi}_h(s,a)$  the cumulative distribution function (CDF) corresponding to $Y^{\pi}_h(s)$ and $Z^{\pi}_h(s,a)$ respectively. Rewriting the random variable in the form of CDF, we have the distributional Bellman equation
\begin{align*}
    \eta^{\pi}_h(s,a)=\sum_{s^{\prime}} P_h(s^{\prime}|s,a)\nu^{\pi}_{h+1}(s^{\prime})(\cdot-r_h(s,a)), \nu^{\pi}_h(s)=\eta^{\pi}_h(s,\pi_h(s)).
\end{align*}
The distributional Bellman equation outlines the backward recursion of the return distribution under a fixed policy. Our focus is primarily on risk-neutral control, aiming to maximize the mean value of the return, as represented by:
\[  \pi^{*}(s)\triangleq\arg\max_{(\pi_1,...,\pi_H)\in\Pi} \bE[Z^{\pi}_1(s)].   \]
Here, $\pi=(\pi_1,...,\pi_H)$ signifies that this is a multi-stage maximization problem. An exhaustive search approach is impractical due to its exponential computational complexity. However, the principle of optimality applies, suggesting that the optimal policy for any tail sub-problem coincides with the tail of the optimal policy, as discussed in \citep{bertsekas2000dynamic}. This principle enables the reduction of the multi-stage maximization problem into several single-stage maximization problems. Let $Z^{*}=Z^{\pi^*}$ and $Y^{*}=Y^{\pi^*}$ denote the optimal return. The risk-neutral Bellman optimality equation can be expressed as follows:
\begin{align*}
    Z^{*}_h(s,a)&=r_h(s,a)+ Y^{*}_{h+1}(S^{\prime}), S^{\prime}\sim P_h(\cdot|s,a) \\
    \pi^*_h(s)&=\arg\max_a\bE[Z^{*}_h(s,a)], Y^*_h(s)=Z^{*}_h(s,\pi^*_h(s)).
\end{align*}
Rewriting this in the form of distributions, we get:
\begin{align*}
    \eta^{*}_h(s,a)&=[P_h \nu^{*}_{h+1}](s,a)(\cdot-r_h(s,a)) \\
    \pi^*_h(s)&=\arg\max_a\bE[\eta^{*}_h(s,a)], \nu^*_h(s)=\eta^{*}_h(s,\pi^*_h(s)).
\end{align*}

\section{Risk-sensitive Distributional Dynamic Programming} 
\label{sec:ddp}
In this section, we establish a novel DDP framework for risk-sensitive control. For the risk-sensitive purpose, we define the action-value function of a policy $\pi$  at step $h$  as $Q^{\pi}_h(s,a)\triangleq U_{\beta}(Z^{\pi}_h(s,a))$, which is the EntRM value of the return distribution, for each $(s,a,h)\in\mathcal{S}\times\mathcal{A}\times[H]$. 
The value function is  defined as $V^{\pi}_h(s)\triangleq Q^{\pi}_h(s,\pi_h(s))=U_{\beta}(Y^{\pi}_h(s))$.
We focus on the control setting, in which the goal is to find an optimal policy to maximize the value function, that is,  
\[  \pi^{*}(s)\triangleq\arg\max_{(\pi_1,...,\pi_H)\in\Pi}V^{\pi_1...\pi_H}_1(s).   \]
In the risk-sensitive setting, however, the principle of optimality does not always hold for general risk measures. For example, the optimal policy for CVaR may be non-Markovian or history-dependent \citep{shapiro2021lectures}.

EntRM satisfies \textit{independence property} (also known as the independence axiom) in economics and decision theory \citep{von1947theory,dentcheva2013common}. For better illustration, we introduce some additional notations. We write $X\ge Y$ or $F\ge G$ if $\erm{X}\ge \erm{Y}$ or $\erm{F}\ge \erm{G}$. This is different from the notion of  stochastic dominance $X\succeq Y$. In fact,  \textbf{(M)} of EntRM implies that
\[ X \succeq Y \Longrightarrow \erm{X}\ge \erm{Y} \Longleftrightarrow X\ge Y.\]
\vspace{-4ex}
\begin{fact}[Independence property]
    $\forall F, G, H \in \mathscr{D}, \theta\in[0,1],$ the following holds
		\begin{align*}
		F\leq G\Longrightarrow \theta F+(1-\theta)H\leq \theta G+(1-\theta)H.
		\end{align*}
\end{fact}
\begin{proof}
	We only prove the case that $\beta>0$. The case that $\beta<0$ follows analogously. For any two distributions $F$ and $G$ such that $U_{\beta}(F)>U_{\beta}(G)$, we have
	\[ U_{\beta}(F)=\frac{1}{\beta}\log\int_{\mathbb{R}} \exp(\beta x)dF(x)>\frac{1}{\beta}\log\int_{\mathbb{R}} \exp(\beta x)dG(x)= U_{\beta}(G),\]
	which implies $\int_{\mathbb{R}}\exp(\beta x)d F(x)> \int_{\mathbb{R}}\exp(\beta x)d G(x)$.
	Thus for any distribution $H$, 
	\begin{align*}
	U_{\beta}(\theta F+(1-\theta)H)&= \frac{1}{\beta}\log\int_{\mathbb{R}} \exp(\beta x)d(\theta F(x)+(1-\theta)H(x))\\
	&=\frac{1}{\beta}\log\prt{\theta\int_{\mathbb{R}}\exp(\beta x)d F(x) + (1-\theta)\int_{\mathbb{R}}\exp(\beta x)d H(x) } \\
	&> \frac{1}{\beta}\log\prt{\theta\int_{\mathbb{R}}\exp(\beta x)d G(x) + (1-\theta)\int_{\mathbb{R}}\exp(\beta x)d H(x) }\\
	&=U_{\beta}(\theta G+(1-\theta)H).
	\end{align*}
This finishes the proof.
\end{proof}
Moreover, the property (\textbf{TI}) entails that the EntRM value of the current return $Z^{\pi}_h(s,a)$ equals the sum of the immediate reward $r_H(s,a)$ and the value of the future return $Y^{\pi}_h(s^{\prime})$
\begin{align*}
    U_{\beta}(Z^{\pi}_h(s,a))= U_{\beta}(r_h(s,a)+ Y^{\pi}_h(s^{\prime})) &= r_h(s,a)+U_{\beta}(Y^{\pi}_h(s^{\prime}))=r_h(s,a) + U_{\beta}( [P_h\nu^{\pi}_{h+1}](s,a)),
\end{align*}
where we write $[P_h\nu^{\pi}_{h+1}](s,a)=\sum_{s^{\prime}} P_h(s^{\prime}|s,a)\nu^{\pi}_{h+1}(s^{\prime})$.
We will show that \textbf{(I)} and \textbf{(T)} suggest that the optimal future return $Y^*_h(s^{\prime})$ consists in the optimal current return $Z^*_h(s,a)$
\[	Z^*_h(s,a)=r_h(s,a)+Y^*_h(s^{\prime}).	\] 
These observations implies the  principle of optimality. For notational simplicity, we write $\pi_{h_1:h_2}=\{\pi_{h_1},\pi_{h_1+1},\cdots,\pi_{h_2}\}$ for two positive integers $h_1<h_2\leq H$.
\begin{proposition}[Principle of optimality]
\label{prop:princ}
Let $\pi^{*}=\pi^*_{1:H}$ be an optimal policy. Fixing $h\in[H]$, then the truncated optimal policy $\pi^*_{h:H}$ is optimal for the sub-problem $\max_{\pi_{h:H}\in\Pi_{h:H}}V^{\pi}_h$.
\end{proposition}
\begin{proof}
Suppose that the truncated policy $\pi^*_{h:H}$ is not optimal for this subproblem, then there exists an optimal policy $\tilde{\pi}_{h:H}$  such that 
	$$
	\exists \tilde{s}_h \quad \text{occurring with positive probability}, \quad V^{\tilde{\pi}_{h:H}}_h(\tilde{s}_h ) > V^{\pi^*_{h:H}}_h(\tilde{s}_h ).
	$$ 
There exists a state $\tilde{s}_{h-1} $ with $P_{h-1}(\tilde{s}_h|\tilde{s}_{h-1},\pi^{*}_{h-1}(\tilde{s}_{h-1}))>0$ such that
	\begin{align*}
	\erm{\nu^{\pi^{*}_{h-1}, \tilde{\pi}_{h:H}}_{h-1}(\tilde{s}_{h-1})}&=r_{h-1}(\tilde{s}_{h-1},\pi^{*}_{h-1}(\tilde{s}_{h-1}) ) + U_{\beta}\prt{  \brk{P_{h-1}\nu^{\tilde{\pi}_{h:H}}_{h}}(\tilde{s}_{h-1},\pi^{*}_{h-1}(\tilde{s}_{h-1}) } \\
	&> r_{h-1}(\tilde{s}_{h-1},\pi^{*}_{h-1}(\tilde{s}_{h-1}) ) + U_{\beta}\prt{  \brk{P_{h-1}\nu^{\pi^*_{h:H}}_{h}}(\tilde{s}_{h-1},\pi^{*}_{h-1}(\tilde{s}_{h-1}) } \\
	&=U_{\beta}\prt{\nu^{\pi^{*}_{h-1:H}}_{h-1}(\tilde{s}_{h-1})},
	\end{align*}
where the inequality is due to \textbf{(I)} of $U_{\beta}$. It follows that  $(\pi^{*}_{h-1}, \tilde{\pi}_{h:H})$ is a strictly better policy than $\pi^{*}_{h-1:H}$  for the subproblem from $h-1$ to $H$.  Using induction, we deduce that $(\pi^{*}_{1:h-1}, \tilde{\pi}_{h:H})$ is a strictly better policy than $\pi^{*}=\pi^{*}_{1:H}$. This is contradicted to the assumption that $\pi^{*}$ is an optimal policy. 
\end{proof}
Furthermore, the principle of optimality induces the \textit{distributional Bellman optimality equation} in the risk-sensitive setting.
\begin{proposition}[Distributional Bellman optimality equation]
\label{prop:DP_alg}
The optimal policy $\pi^*$ is given by the following backward recursions:
\begin{equation}
\label{eqt:dbu}
\begin{aligned}
    &\nu^*_{H+1}(s)=\psi_0,\  \eta^*_h(s,a)=[P_h\nu^*_{h+1}](s,a) (\cdot- r_h(s,a)),\\
    &\pi^*_h(s)=\arg\max_{a\in\mathcal{A}}Q^*_h(s,a)=U_{\beta}(\eta^*_h(s,a)),\nu^*_h(s)=\eta^*_h(s,\pi^*_h(s)),
\end{aligned}
\end{equation}
where $F(\cdot-c)$ denotes the CDF obtained by shifting $F$ to the right by $c$. Furthermore, the sequence $(\eta^*_h)_{h\in[H]}$ and $(\nu^*_h)_{h\in[H]}$ represent the sequence of distributions corresponding to the optimal returns at each step.
\end{proposition}
\begin{proof}
	Throughout the proof we omit $*$ for the ease of notation. The proof follows from induction. Notice that $\eta_h(s_h)$ and $V_h(s_h)$ are the return distribution and value function for state $s_h$ at step $h$ following policy $\pi_{h:H}$ respectively. At step $H$, it is obvious that $\pi_H$ is the optimal policy that maximizes the EntRM value at the final step. Fixing $h \in [H-1]$, we assume that $\pi_{h+1:H}$ is the optimal policy for the  subproblem
	$$ 
	V^{\pi_{h+1:H}}_{h+1}(s_{h+1}) =\max_{\pi^{\prime}_{h+1:H}} V^{\pi^{\prime}_{h+1:H}}_{h+1}(s_{h+1}), \forall s_{h+1}.
	$$
	In other words, $\forall \pi^{\prime}_{h+1:H}, \forall s_{h+1}$:
	\begin{align*}
	U_{\beta}(\nu_{h+1}(s_{h+1}))&=U_{\beta}(\nu^{\pi_{h+1:H}}_{h+1}(s_{h+1}))  \ge U_{\beta}(\nu^{\pi^{\prime}_{h+1:H}}_{h+1}(s_{h+1})).
	\end{align*}
	It follows that $\forall s_{h}$,
	\begin{align*}
	V_h(s_h)&=Q_h(s_h,\pi_h(s_h)) = U_{\beta}(\nu^{\pi_{h:H}}_{h}(s_{h})) = \max_{a_h}U_{\beta}(\eta_h(s_h,a_h)) \\
	&= \max_{a_h} \{ r_h(s_h,a_h) + U_{\beta}\prt{\brk{P_h\nu_{h+1}}(s_h,a_h)} \} \\
	&\ge \max_{a_h} \cbrk{ r_h(s_h,a_h) + \max_{\pi^{\prime}_{h+1:H}} U_{\beta}\prt{\brk{P_h\nu^{\pi^{\prime}_{h+1:H}}_{h+1}}(s_h,a_h)} } \\
	&=  \max_{\pi^{\prime}_h} \cbrk{ r_h(s_h,a_h) + \max_{\pi^{\prime}_{h+1:H}} U_{\beta}\prt{\brk{P_h\nu^{\pi^{\prime}_{h+1:H}}_{h+1}}(s_h,a_h)} } \\
	&= \max_{\pi^{\prime}_{h:H}} \cbrk{ r_h(s_h,a_h) +  U_{\beta}\prt{\brk{P_h\nu^{\pi^{\prime}_{h+1:H}}_{h+1}}(s_h,\pi^{\prime}_h(s_h))} } \\
	&=\max_{\pi^{\prime}_{h:H}} U_{\beta}\prt{\nu^{\pi^{\prime}_{h+1:H}}_{h}(s_{h})}.
	\end{align*}
	Hence  $V_h$ is the optimal value function at step $h$ and $\pi_{h:H}$ is the optimal policy for the  sub-problem from $h$ to $H$. The induction is completed.
\end{proof}
For simplicity, we define the \emph{distributional Bellman operator} $\mathcal{T}(P,r):\mathscr{D}^{\mathcal{S}}\rightarrow\mathscr{D}^{\mathcal{S}\times\mathcal{A}}$ with associated model $(P,r)=(P(s,a), r(s,a))_{(s,a)\in\mathcal{S}\times\mathcal{A}}$ as 
\begin{equation*}
\label{eqt:bell}
        [\mathcal{T}(P,r)\nu](s,a)\triangleq[P\nu](s,a)(\cdot-r(s,a)),\ \forall (s,a)\in\mathcal{S}\times\mathcal{A}.
\end{equation*}
Denote by $\cT_h\triangleq\cT(P_h,r_h)$, then we can rewrite  Equation \ref{eqt:dbu} in a compact form:
\begin{equation}
\begin{aligned}
&\nu^{*}_{H+1}(s)=\psi_0,\ \eta^{*}_h(s,a)=[\cT_h\nu^{*}_{h+1}](s,a), \\
&\pi^*_h(s)=\arg\max_{a\in\mathcal{A}}U_{\beta}(\eta^*_h(s,a)),\ \nu^{*}_h(s)=\eta^{*}_h(s,\pi^*_h(s)).
\end{aligned}
\end{equation}
\vspace{-2ex}
\paragraph{Discussion about the independence property} Another property closely related to \textbf{(I)} is the tower (\textbf{T}) property \citep{kupper2009representation}. 
\vspace{-0ex}
\begin{definition}[Tower property]
        A risk measure $\rho$ satisfies the tower property if for two r.v.s $X$ and $Y$, we have
        \[ \rho(X) = \rho(\rho(X|Y)),\]
        where $\rho(\cdot|Y)$ is taken w.r.t. the conditional distribution.
\end{definition}
We can show that the following implications hold
\[ \textbf{(T)}\Longrightarrow\textbf{(I)}\Longrightarrow \text{DP}. \]
$\textbf{(T)}\Longrightarrow\textbf{(I)}$: Suppose \textbf{(T)} holds. Let $X_1\sim F, X_2\sim H, Y_1\sim G, Y_2\sim H$. Let $I \sim (1,2;1-\theta)$ be a binary r.v. independent of $X$ and $Y$. Given $F\leq G$, we have
    \begin{align*}
    \textbf{(DI)} &\Longrightarrow U_{\beta}(X_1)=U_{\beta}(F)\leq U_{\beta}(G)=U_{\beta}(Y_1) \\
    &\Longrightarrow U_{\beta}(X_I|I)\sim (U_{\beta}(X_1),U_{\beta}(X_2);1-\theta) \preceq (U_{\beta}(Y_1),U_{\beta}(Y_2);1-\theta)\sim U_{\beta}(Y_I|I).
    \end{align*}
    Next, Fact \ref{fct:mix} implies
    \[ \text{Fact \ref{fct:mix}} \Longrightarrow X_I\sim \theta F +(1-\theta)H, Y_I\sim \theta G +(1-\theta)H. \]
    \vspace{-5ex}
    \begin{fact}[Mixture distribution]
    \label{fct:mix}
        Let $X_i\sim F_i$ for $i\in[n]$. Let $I$ be a discrete r.v. \emph{independent of $(X_i)_i$} with $\bP(I=i)=\theta_i$. Then $X_I\sim \sum_{i\in[n]}\theta_i F_i$.
    \end{fact}
    It follows that 
    \begin{align*}
        U_{\beta}(X_I|I) \preceq U_{\beta}(Y_I|I) &\Longrightarrow U_{\beta}(U_{\beta}(X_I|I))\leq  U_{\beta}(U_{\beta}(Y_I|I)) \\
        &\Longrightarrow U_{\beta}(X_I) \leq U_{\beta}(Y_I) \\
        &\Longrightarrow U_{\beta}(\theta F+(1-\theta)H))\leq U_{\beta}(\theta G+(1-\theta)H),
    \end{align*}
    where the first and second implication is due to \textbf{(M)} and \textbf{(T)} of $U_{\beta}$. \\
$\textbf{(T)} \Longrightarrow \text{DP}$: Fix  $h\in[H-1]$ and $(s,a)\in\cS\times\cA$. Using \textbf{(T)} leads to a decomposition of the \emph{risk-sensitive value function} as follows:
\begin{align*}
    Q_h(s,a)&=\erm{Z_h(s,a)}=\erm{r_h(s,a)+Y_{h+1}(S)}=r_h(s,a)+\erm{Y_{h+1}(S)}\\
    &=r_h(s,a)+\erm{\erm{Y_{h+1}(S)|S}}\\
    &=r_h(s,a)+\erm{V_{h+1}(S)}.
\end{align*}
We call it the risk-sensitive \emph{Value} Bellman equation, which relates the action-value function $Q_h$ at step $h$ to the next-step value functions $V_{h+1}$. We can further derive the \textit{Value Bellman optimality equation} with \textbf{(M)} and \textbf{(T)}. Observe that
\begin{align*}
    V_{h+1}(s)\ge V'_{h+1}(s),\forall s \Longrightarrow V_{h+1}(S)|S\succeq V'_{h+1}(S)|S \\
    \Longrightarrow  \erm{V_{h+1}(S)}\ge \erm{V'_{h+1}(S)} \Longrightarrow Q_{h}(s,a)\ge Q'_{h}(s,a),
\end{align*}
which implies the Value Bellman optimality equation
\begin{align*}
     Q^*_h(s,a)&=r_h(s,a)+\erm{V^*_{h+1}(S)}\\
     \pi^*_h(s)&=\arg\max Q^*_h(s,a), V^*_h(s)= Q^*_h(s,\pi^*_h(s)).
\end{align*}
We make the following summary. \\
(i) Both \textbf{(T)} and \textbf{(I)} imply the principle of dynamic programming, but \textbf{(I)} is considered a weaker assumption than \textbf{(T)}. This indicates that while both properties support the formulation of DP, \textbf{(I)} does so under less stringent conditions. \\
(ii) \textbf{(I)} inspires a distributional perspective in RSRL, leading to the concept of DDP. This perspective involves running DP in the language of random variables or distributions, as opposed to traditional scalar values. In contrast,  \textbf{(T)} primarily supports the classical Value Bellman equation. It's important to note that both distributional and classical DP contribute to the derivation of optimal policies. However, in our work, \textbf{(I)} plays a crucial and irreplaceable role. DDP, enabled by \textbf{(I)}, facilitates the algorithm design and regret analysis in distributional RL, which is not achievable solely with \textbf{(T)}.

Finally, the regret of an algorithm $\mathscr{A}$ interacting with an MDP $\mathcal{M}$ for $K$ episodes is defined as
\begin{align*}
    \text{Regret}(\mathscr{A},\mathcal{M},K)&\triangleq \sum_{k=1}^K V^*_1(s^k_1)-V^{\pi^k}_h(s^k_1).
\end{align*}
Note that the regret is a random variable since $\pi^k$ is a random quantity. We denote by $\mathbb{E}[\text{Regret}(\mathscr{A},\mathcal{M},K)]$ the expected regret.  
We will omit $\mathscr{A}$ and $\mathcal{M}$ for simplicity.

\section{RODI-MF}
\label{sec:rodi_mf}
In this section, we introduce\tb{M}odel-\hspace{-1ex}\tb{F}ree \textbf{R}isk-sensitive  \textbf{O}ptimistic \textbf{D}istribution \textbf{I}teration (\texttt{RODI-MF}), as detailed in Algorithm \ref{alg:RODI-MF}.
\begin{algorithm}[H]
	\caption{\texttt{RODI-MF}}
	\label{alg:RODI-MF}
	\begin{algorithmic}[1]
		\State{Input: $T$ and $\delta$}
		\State{Initialize $N_h(\cdot,\cdot)\leftarrow0$; $\eta_h(\cdot,\cdot), \nu_h(\cdot)\leftarrow \delta_{H+1-h}$  $\forall h\in[H]$}
		\For{$k = 1:K$}
		\For{$h = H:1$}
		\If{$N_h(\cdot,\cdot)>0$}
		\State{$\eta_h(\cdot,\cdot)\leftarrow
			\frac{1}{N_h(\cdot,\cdot)}\sum_{\tau\in[k-1]}\mathbb{I}^{\tau}_h(\cdot,\cdot)\nu_{h+1}(s^{\tau}_{h+1})(\cdot-r_h(\cdot,\cdot))$}
		\EndIf{}
		\State{$c_h(\cdot,\cdot)\leftarrow \sqrt{\frac{2S}{N_h(\cdot,\cdot)\vee 1}\iota}$}
		\State{$\eta_h(\cdot,\cdot)\leftarrow
			\rO^{\infty}_{c_h(\cdot,\cdot)}\eta_h(\cdot,\cdot)$}
		\State{$\pi_h(\cdot)\leftarrow\arg\max_{a}U_{\beta}(\eta_h(\cdot,a))$}
		\State{$\nu_h(\cdot)\leftarrow\eta_h(\cdot,\pi_h(\cdot))$}
		\EndFor
		\State{Receive $s^k_1$}
		\For{$h = 1:H$}
		\State{$a^k_h\leftarrow\pi_h(s^k_h)$ and transit to $s^k_{h+1}$}
		\State{$N_h(s^k_h,a^k_h)\leftarrow N_h(s^k_h,a^k_h)+1$}
		\EndFor
		\EndFor
	\end{algorithmic} 
\end{algorithm}
We begin by establishing additional notations. For two Cumulative Distribution Functions (CDFs) $F$ and $G$, the supremum distance between them is defined as $\Vert F-G \Vert_{\infty}\triangleq\sup_x |F(x)-G(x)|$. We define the $\ell_1$ distance between two Probability Mass Functions (PMFs) with the same support $P$ and $Q$ as $\norm{P-Q}_1\triangleq \sum_i|P_i-Q_i|$. Furthermore, the set $B_{\infty}(F,c):=\{G\in\mathscr{D}|\Vert G-F\Vert_{\infty}\leq c\}$  denotes the supremum norm ball of CDFs centered at $F$ with radius $c$. Analogously, $B_{1}(P,c)$ represents the $\ell_1$ norm ball of PMFs centered at $P$ with radius $c$.

In each episode, Algorithm \ref{alg:RODI-MF} comprises two distinct phases: the planning phase and the interaction phase. During the planning phase, the algorithm executes an optimistic variant of the approximate Risk-Sensitive Distributional Dynamic Programming (RS-DDP), progressing backward from step $H+1$ to step 1 within each episode. This process results in a policy to be employed during the subsequent interaction phase. We offer further details about the two phases as follows:

\emph{Planning phase (Line 4-12)}. The algorithm undertakes a sample-based distributional Bellman update in Lines 5-7. To clarify, we append the episode index $k$ to the variables in Algorithm \ref{alg:RODI-MF} corresponding to episode $k$. For instance, $\eta_h^k$ represents $\eta_h$ in episode $k$. Specifically, for visited state-action pairs, Line 6 essentially performs an approximate DDP.  Let $\mathbb{I}^k_h(s,a)\triangleq\mathbb{I}\{(s^{k}_h,a^{k}_h)=(s,a)\}$ and $N^k_h(s,a)\triangleq \sum_{\tau\in[k-1]}\mathbb{I}^{\tau}_h(s,a) $. For a given tuple  $(s,a,k,h)$ with $N^k_h(s,a)\ge1$, , the empirical transition model $\hat{P}^k_h(\cdot|s,a)$ is defined as:
\vspace{-0ex}
\[  \hat{P}^k_h(s^{\prime}|s,a)\triangleq\frac{1}{N^k_h(s,a)}\sum_{\tau\in[k-1]}\mathbb{I}^{\tau}_h(s,a)\cdot\mathbb{I}\{s^{\tau}_{h+1}=s^{\prime}\}. \]
For any $\nu\in\mathscr{D}^{\mathcal{S}}$, the following holds:
\vspace{-0ex}
\begin{align*}
\brk{\hat{P}^k_h \nu}(s,a)&=\sum_{s^{\prime}\in\mathcal{S}}\hat{P}^k_h(s^{\prime}|s,a)\nu(s^{\prime})=\frac{1}{N^k_h(s,a)}\sum_{s^{\prime}\in\mathcal{S}}\sum_{\tau\in[k-1]}\mathbb{I}^{\tau}_h(s,a)\cdot\mathbb{I}\{s^{\tau}_{h+1}=s^{\prime}\}\nu(s^{\prime})\\
&=\frac{1}{N^k_h(s,a)}\sum_{\tau\in[k-1]}\mathbb{I}^{\tau}_h(s,a)\cdot\sum_{s^{\prime}\in\mathcal{S}}\mathbb{I}\{s^{\tau}_{h+1}=s^{\prime}\}\nu(s^{\tau}_{h+1})\\
&=\frac{1}{N^k_h(s,a)}\sum_{\tau\in[k-1]}\mathbb{I}^{\tau}_h(s,a)\nu(s^{\tau}_{h+1}).
\end{align*}
Thus, the update formula in Line 6 of Algorithm \ref{alg:RODI-MF} can be reformulated as:
\begin{align*}
\eta^k_h(s,a) 
&=\brk{\hat{P}^k_h\nu^k_{h+1}}(s,a)(\cdot-r_h(s,a))=\brk{\hat{\cT}^k_h\nu^k_{h+1}}(s,a).
\end{align*}
Conversely, for unvisited state-action pairs, the return distribution remains aligned with the highest plausible reward $H+1-h$. Subsequently, the algorithm calculates the optimism constants $c^k_h$ (Line 8) and applies the distributional optimism operator $\rO^{\infty}_{c^k_h}$ (Line 9) to obtain the optimistically plausible return distribution $\eta^k_h$. The selection of $c^k_h$ will be elaborated later. The optimistic return distribution yields the optimistic value function, from which the algorithm derives the greedy policy $\pi^k_h$ to be applied during the interaction phase.

\emph{Interaction phase (Line 14-17).}  In Lines 15-16, the agent interacts with the environment under policy $\pi^k$ and refreshes the counts $N^k_h$ based on newly gathered observations. 

\subsection{Connection to Exponential Utility}
Our analysis explores the relationship between EntRM and Exponential Utility (EU). The EU is defined as follows:
\[ E_{\beta}(F) \triangleq e^{\beta U_{\beta}(F)}=\int_{\bR} e^{\beta x} dF(x),\]
where it serves as an exponential transformation of the EntRM.  Notably, this transformation preserves the order in the sense that for any non-zero $\beta$,
$\forall \beta\neq0$,
\[  U_{\beta}(F)\ge U_{\beta}(G) \Longleftrightarrow \text{sign}(\beta) E_{\beta}(F)\ge \text{sign}(\beta)E_{\beta}(G).   \]
Leveraging this property, we derive the distributional Bellman optimality equation in terms of EU as follows:
\begin{equation}
\label{eqt:eerm_opt}
\begin{aligned}
\nu^*_{H+1}(s)&=\psi_0,\  \eta^*_h(s,a)=[P_h\nu^*_{h+1}](s,a) (\cdot-r_h(s,a)),\\
\pi^*_h(s)&=\arg\max_{a\in\mathcal{A}} \text{sign}(\beta) E_{\beta}(\eta^*_h(s,a)),\ \nu^*_h(s)=\eta^*_h(s,\pi^*_h(s)).
\end{aligned}
\end{equation}
\begin{proposition}[Equivalence between EntRM and EU]
\label{prop:eq_eerm}
The policy $\pi^*$, generated by Equation \ref{eqt:eerm_opt}, is optimal for both the EntRM and EU. Moreover, the return distribution generated aligns with the optimal return distribution for EntRM.
\end{proposition}
\begin{proof}
The proof employs induction. The only difference between Equation \ref{eqt:eerm_opt} and Equation \ref{eqt:dbu} lies in the policy generation step. For clarity, denote the quantities generated by the respective equations as $(\cdot)^*$ and $\tilde{(\cdot)}^*$. The base case with  $\eta^*_H=\tilde{\eta}^*_H$ is evident. It follows that $\pi^*_H(s)=\tilde{\pi}^*_H(s)$ for each $s$, due to the preserved order under the exponential transformation. Subsequently, it holds that $\nu^*_H=\tilde{\nu}^*_H$. Assuming $\nu^*_{h+1}=\tilde{\nu}^*_{h+1}$ for $h+1\in[2:H]$, we establish that $\eta^*_h=\tilde{\eta}^*_h$, $\pi^*_h=\tilde{\pi}^*_h$ and $\nu^*_{h}=\tilde{\nu}^*_{h}$. This completes the induction process.
\end{proof}

We further present two important properties of EU, instrumental in formulating the regret upper bounds: \emph{Lipschitz continuity} and \emph{linearity}. Denote by $L_{M}$  the Lipschitz constant of $E_{\beta}:\mathscr{D}_M\rightarrow\mathbb{R}$ with respect to the infinity norm $\norm{\cdot}_{\infty}$, which satisfies: 
\[  E_{\beta}(F)-E_{\beta}(G) \leq L_M \norm{F-G}_{\infty}, \forall F,G \in \mathscr{D}_M.    \]
Lemma \ref{lem:eu_lip} establishes a \emph{tight} Lipschitz constant for EU, linking the distance between distributions to the difference in their EU values.  
\begin{lemma}[Lipschitz property of EU]
\label{lem:eu_lip}
$E_{\beta}$ is Lipschitz continuous with respect to the supremum norm over $\mathscr{D}_M$ with $L_M=\abs{\exp(\beta M)-1}$, Moreover,  $L_{M}$ is tight with respect to both $\beta$ and $M$.
\end{lemma}
The proof is deferred to Appendix \ref{app:pf}. It's worth noting that as $\beta$ approaches zero, $L_M$ tends to zero, aligning with the observation that $\lim_{\beta\rightarrow0}E_{\beta}=1$. Another key property of EU is the linearity:
\[ E_{\beta}(\theta F + (1-\theta) G)= \theta E_{\beta}(F) + (1-\theta) E_{\beta}(G).\]        
This property significantly refines the regret bounds. In contrast, the non-linearity of EntRM may result in an exponential factor of $\exp(|\beta| H)$ in error propagation across time steps, potentially leading to a compounded factor of $\exp(|\beta| H^2)$ in the regret bound.

\subsection{Distributional Optimism over the Return Distribution} 
For the purpose of clarity, we will focus on the scenario where $\beta>0$ in the subsequent discussion. The case for $\beta<0$ can be approached using analogous reasoning. We commence by formally defining optimism at the distributional level.
\begin{definition}
 Given two CDFs $F$ and $G$, we say that $F$ is more optimistic than $G$  if $F\ge G$.
\end{definition}
This definition aligns with the intuitive notion that a more optimistic distribution should possess a larger EntRM value. Given that the exponential transformation preserves order, $F$ is more optimistic than $G$ if and only if $E_{\beta}(F)\ge E_{\beta}(G)$. Following the methodology of \cite{keramati2020being}, we introduce the distributional optimism operator $\rO^{\infty}_c:\sD(a,b)\mapsto\sD(a,b)$ for a level $c\in(0,1) $ as
\[  (\rO^{\infty}_c F)(x)\triangleq[F(x)-c\mathbb{I}_{[a,b)}(x)]^+.\]
This operator shifts the distribution $F$ down by a maximum of $c$ over $[a,b)$,  ensuring that the resulting function $\rO^{\infty}_c F$ remains a valid CDF in $\mathscr{D}(a,b)$ and optimistically dominates all other CDFs within the same space\footnote{For a more comprehensive explanation, please refer to \cite{liang2023distribution}.}.
\begin{lemma}
\label{lem:opt_opt}
Let $\rho$ be a functional (not necessarily a risk measure) satisfying \textbf{(M)}. For any $G\in\mathscr{D}(a,b)$, it holds that if $G\in B_{\infty}(F,c)$, then $G \preceq  \rO^{\infty}_{c} F$. Moreover, it holds that
	\[\rO^{\infty}_c F \in \arg\max_{G\in B_{\infty}(F,c)\cap\mathscr{D}(a,b)}\rho(G).\]
\end{lemma}
\begin{proof}
Consider any $G\in\mathscr{D}([a,b])\cap B_{\infty}(F,c)$. By the definition of $B_{\infty}(F,c)$, we have $\sup_{x\in[a,b]} |F(x)-G(x)|\leq c$. Therefore, for any $x\in[a,b]$, we have $G(x)\ge\max(F(x)-c,0)=(\rO^{\infty}_c F)(x)$. Since $\rO^{\infty}_c F$ dominates any $G$ in $\mathscr{D}([a,b])\cap B_{\infty}(F,c)$ and considering \textbf{(M)} of $\rho$, we arrive at the conclusion. 
\end{proof}
We define the EU value produced by the algorithm as $W^k_h(s)\triangleq E_{\beta}(\nu^k_h(s))$ and $J^k_h(s,a)\triangleq E_{\beta}(\eta^k_h(s,a))$ for all $(s,a,k,h)$. Similarly, we define $W^*_h(s)\triangleq E_{\beta}(\nu^*_h(s))$ and $J^*_h(s,a)\triangleq E_{\beta}(\eta^*_h(s,a))$ for all $(s,a,h)$. We define $\iota=\log(SAT/\delta)$ for any $\delta\in(0,1)$ and introduce the \emph{good event}  as follows:
\begin{align*}
\mathcal{G}_{\delta}\triangleq&\left\{ \norm{\hat{P}^k_h(\cdot|s,a)-P_h(\cdot|s,a)}_{1}\leq\sqrt{\frac{2S}{N^k_h(s,a)\vee 1}\iota}, \forall (s,a,k,h)\in\mathcal{S}\times\mathcal{A}\times[K]\times[H]\right\},  
\end{align*}
This event encapsulates the scenario where the empirical distributions concentrates around the true distributions with respect to the $\ell_1$ norm. Leveraging Lemma \ref{lem:bd_opt_cons}, \textbf{(M)} of EU, and inductive reasoning, we arrive at Proposition \ref{prop:main_opt_r}, which asserts that the sequence ${W^k_1(s^k_1)}_{k\in[K]}$ is consistently optimistic compared to the optimal value sequence ${W^*_1(s^k_1)}_{k\in[K]}$.
\begin{proposition}[Optimism]
\label{prop:main_opt_r} 
Conditioned on event $\mathcal{G}_{\delta}$, the sequence  $\{W^k_1(s^k_1)\}_{k\in[K]}$ produced by Algorithm \ref{alg:RODI-MF} are all greater than or equal to $W^*_1(s^k_1)$, i.e., 
\[ W^k_1(s^k_1)=E_{\beta}(\nu^k_1(s^k_1)) \ge E_{\beta}(\nu^*_1(s^k_1))=W^*_1(s^k_1), \forall k \in [K]. \]
\end{proposition}
We first present a series of lemmas, specifically Lemma \ref{lem:main_event} through Lemma \ref{lem:bd_opt_cons}, which is used in the proof of Proposition \ref{prop:main_opt_r}.
\begin{lemma}[High probability good event]
	\label{lem:main_event}
	For any $\delta\in(0,1)$, the event $\mathcal{G}_{\delta}$ is true with probability at least $1-\delta$.
\end{lemma}
We will verify the distributional optimism conditioned on $\mathcal{G}_{\delta}$.
\begin{lemma}
	\label{lem:mix_dis}
	For any $F_i\in\mathscr{D}$ and any $\theta,\theta^{\prime}\in\Delta_n$ with any $n\ge2$, it holds that
	\[  \norm{\sum_{i=1}^n \theta_i F_i -\sum_{i=1}^n \theta^{\prime}_i F_i}_{\infty}\leq \norm{\theta-\theta^{\prime}}_1.\]
\end{lemma}
\begin{proof}
	\begin{align*}
	\norm{\sum_{i=1}^n \theta_i F_i -\sum_{i=1}^n \theta^{\prime}_i F_i}_{\infty}&=\sup_{x\in\mathbb{R}}\abs{\sum_{i=1}^n(\theta_i F -\theta^{\prime}_i) F_i(x)} \leq \sup_{x\in\mathbb{R}} \sum_{i=1}^n|\theta_i  -\theta^{\prime}_i| F_i(x) \\
	&\leq \sum_{i=1}^n|\theta_i  -\theta^{\prime}_i| =\norm{\theta-\theta^{\prime}}_1.
	\end{align*}
\end{proof}
\begin{lemma}[Bound on the optimistic constant]
	\label{lem:bd_opt_cons}
	For any bounded distributions $\{F_i\}_{i\in[n]}$ and any $\theta,\theta^{\prime}\in\Delta_n$ it holds that 
	if $c\ge \Vert\theta-\theta^{\prime}\Vert_1$, then
	\[ \sum_{i=1}^n\theta_i F_i \preceq \rO^{\infty}_c\prt{\sum_{i=1}^n\theta^{\prime}_i F_i}. \]
\end{lemma}
\begin{proof}
Without loss of generality assume $F\in\mathscr{D}_M^n$. By Lemma \ref{lem:mix_dis}, for any $x$
\[ \abs{\sum_{i=1}^n(\theta^{\prime}_i-\theta_i) F_i(x)}\leq  \norm{\theta^{\prime}-\theta}_1.   \]
For any $x\in[0,M+1)$,
	\begin{align*}
     \rO^{\infty}_c\prt{\sum_{i=1}^n\theta^{\prime}_i F_i}(x)&=\brk{\sum_{i=1}^n\theta^{\prime}_i F_i(x)-c}^+ 
	=\brk{\sum_{i=1}^n\theta_i F_i(x)+\sum_{i=1}^n(\theta^{\prime}_i-\theta_i) F_i(x)-c}^+ \\
	&\leq \brk{\sum_{i=1}^n\theta_i F_i(x)+\norm{\theta^{\prime}-\theta}_1-c}^+ 
	\leq \brk{\sum_{i=1}^n\theta_i F_i(x)}^+ = \sum_{i=1}^n\theta_i F_i(x).
	\end{align*}
\end{proof}
Now we give the proof of Proposition \ref{prop:main_opt_r}. 
\begin{proof}
The proof proceeds by induction. We fix $k\in[K]$ and consider each stage $h$ in reverse order. Consider the base case: for any $(s,a)$ such that $N^k_H(s,a)>0$
	\begin{align*}
	J^k_H(s,a)=E_{\beta}(\eta^k_H(s,a))=E_{\beta}(\delta_{r_H(s,a)}) = \exp(\beta r_H(s,a)) =J^*_H(s,a).
	\end{align*}
This equality holds because the reward received at stage $H$ is deterministic and hence the EU is simply the exponential of the scaled reward. For unvisited state-action pairs $(s,a)$ with $N^k_H(s,a)=0$, the EU is given by:
	\begin{align*}
	J^k_H(s,a)=E_{\beta}(\eta^k_H(s,a))=E_{\beta}(\delta_{1}) = \exp(\beta ) \ge J^*_H(s,a).
	\end{align*}
Here, the EU value defaults to $\exp(\beta)$, which is greater than or equal to the optimal EU value for any $(s,a)$. Given these calculations, for any state $s$, the  EU value at stage $H$ satisfies $W^k_H(s)=\max_a J^k_H(s,a)\ge\max_a J^*_H(s,a) =W^*_H(s)$, establishing the base case. Assuming that for stage $h+1$, $W^k_{h+1}(s)\ge W^*_{h+1}(s)$ holds for all states $s$, we now consider stage $h$. For any visited state-action pair $(s,a)$ with $N^k_h(s,a)>0$, we apply Lemma \ref{lem:bd_opt_cons} with $\theta=P_h(s,a),\theta^{\prime}=\hat{P}^k_h(s,a)$, $F=\nu^k_{h+1}$ to obtain
\[[P_h\nu^k_{h+1}](s,a) \preceq \rO^{\infty}_{c^k_h(s,a)}([\hat{P}^k_h\nu^k_{h+1}](s,a)), \]
given that $c^k_h(s,a)=\sqrt{\frac{2S}{N^k_h(s,a)\vee1}\iota}\ge\norm{P_h(\cdot|s,a)-\hat{P}^k_h(\cdot|s,a)}_1$ for $h\in[H-1]$. We have
	\begin{align*}
	J^k_h(s,a)&=E_{\beta}(\rO^{\infty}_{c^k_h(s,a)}([\hat{P}^k_h\nu^k_{h+1}](s,a)(\cdot-r_h(s,a))))\\
	&= \exp(\beta r_h(s,a))E_{\beta}(\rO^{\infty}_{c^k_h(s,a)}([\hat{P}^k_h\nu^k_{h+1}](s,a) ))\\
	&\ge \exp(\beta r_h(s,a)) E_{\beta}([P_h\nu^k_{h+1}](s,a) )\\
	&= \exp(\beta r_h(s,a))\cdot [P_h W^k_{h+1}](s,a)\\
	&\ge \exp(\beta r_h(s,a)) \cdot [P_h W^*_{h+1}](s,a)\\
	&=J^{*}_h(s,a),  
	\end{align*}
where the first inequality is due to \textbf{(M)}, and the second inequality follows from the induction assumption. For unvisited state-action pairs $(s,a)$ at stage $h$, the EU is calculated based on the maximum possible reward, ensuring that:
	\[  J^k_h(s,a) = E_{\beta}(\delta_{H+1-h})=\exp(\beta(H+1-h)) \ge  J^{*}_h(s,a).  \]
Finally, aggregating these values for any state $s$ at stage $h$, we obtain:
	\begin{align*}
	W^k_h(s)=\max_a J^k_h(s,a) \ge \max_a J^*_h(s,a)=W^*_h(s),
	\end{align*}
completing the induction step and thereby the proof.
\end{proof}
\vspace{-1ex}
 
\subsection{Regret Upper Bound of \texttt{RODI-MF}}
\begin{theorem}[Regret upper bound of \texttt{RODI-MF}]
	\label{thm:rodi-mf}
	For any $\delta\in(0,1)$, with probability $1-\delta$, the regret of Algorithm \ref{alg:RODI-MF}  is bounded  as
	\begin{align*}
	\text{Regret}(\texttt{RODI-MF},K) &\leq \mathcal{O}\prt{L_{H}(U_{\beta})H\sqrt{S^2AK\iota}}=\tilde{\mathcal{O}}\prt{\frac{\exp(|\beta| H)-1}{|\beta|}H\sqrt{S^2AK}},
	\end{align*}
 where $L_H(U_{\beta})=\frac{\exp(|\beta| H)-1}{|\beta|}$ is the Lipschitz constant of EntRM over $\sD(0,H)$.
\end{theorem}
\begin{remark}
The regret bounds achieved by \texttt{RODI-MF} match the best-known results in \cite{fei2021exponential}. In particular, \texttt{RODI-MF} attains exponentially improved regret bounds compared to RSVI and RSQ  in \cite{fei2020risk} with a factor of $\exp(|\beta| H^2)$. 
\end{remark}
\begin{remark}
For values of $\beta$ that are close to zero, an expansion using Taylor's series reveals that the EntRM, $U_{\beta}(Z^{\pi})$, can be approximated by the sum of the expected cumulative reward and a term proportional to the variance of the cumulative reward, with higher-order terms contributing insignificantly. Considering that the reward $r_h$ lies in the interval $[0,1]$, both the expected cumulative reward and its variance are bounded by terms linear and quadratic in $H$, respectively. To balance the expected reward and the risk (as quantified by the variance), it is prudent to choose $\beta=\mathcal{O}(1/H)$.
\end{remark}
\begin{remark}
    By choosing $|\beta|=\cO(1/H)$, we can eliminate the exponential dependency on $H$ and achieve polynomial regret bound akin to the risk-neutral setting. Therefore, DRL can achieve $\rO\prt{H\sqrt{HS^2AT}}$ regret bound for RSRL with reasonable risk-sensitivity.
\end{remark}
\begin{proof}
We first prove the case $\beta>0$. Define $\Delta^k_h\triangleq W^k_h-W^{\pi^k}_h=E_{\beta}(\nu^k_h)-E_{\beta}\prt{\nu^{\pi^k}_h} \in D_h^{S}$ 
\[ D_h\triangleq[1-\exp(\beta(H+1-h)), \exp(\beta(H+1-h))-1] \]
and $\delta^k_h\triangleq\Delta^k_h(s^k_h)$. For any $(s,h)$ and any $\pi$, we let $P^{\pi}_h(\cdot|s)\triangleq P_h(\cdot|s,\pi_h(s))$. 
The regret can be bounded as
\begin{align*}
\text{Regret}(K)&=\sum^K_{k=1}\frac{1}{\beta}\log\prt{W^{*}_1(s^k_1)}-\frac{1}{\beta}\log\prt{W^{\pi^k}_1(s^k_1)} \\
&= \sum^K_{k=1}\frac{1}{\beta}\log\prt{W^{*}_1(s^k_1)}-\frac{1}{\beta}\log\prt{V^{k}_1(s^k_1)}+\frac{1}{\beta}\log\prt{W^{k}_1(s^k_1)}-\frac{1}{\beta}\log\prt{V^{\pi^k}_1(s^k_1)}\\
&\leq \sum^K_{k=1}\frac{1}{\beta}\log\prt{W^{k}_1(s^k_1)}-\frac{1}{\beta}\log\prt{W^{\pi^k}_1(s^k_1)}\\
&\leq \frac{1}{\beta}\sum^K_{k=1} W^{k}_1(s^k_1)-W^{\pi^k}_1(s^k_1)=\frac{1}{\beta}\sum^K_{k=1}\delta^k_1,
\end{align*}
where the last inequality follows from Lemma \ref{lem:log} and that both $W^{k}_1(s^k_1)$ and $W^{\pi^k}_1(s^k_1)$ are larger than 1. We can decompose $\delta^k_h$ as follows
\begin{align*}
\delta^k_h &= E_{\beta}\prt{\nu^k_h(s^k_h)}-E_{\beta}\prt{\nu^{\pi^k}_h(s^k_h)}  \\
&=E_{\beta}\prt{O_{c^k_h}\prt{\brk{\hat{P}^{\pi^k}_h\eta^k_{h+1}}(s^k_h)(\cdot-r^k_h) }}-E_{\beta}\prt{\brk{P^{\pi^k}_h\nu^{\pi^k}_{h+1}}(s^k_h)(\cdot-r^k_h)}  \\
&= \exp(\beta r^k_h) \prt{  E_{\beta}\prt{O_{c^k_h}\prt{\brk{\hat{P}^{\pi^k}_h\eta^k_{h+1}}(s^k_h) }}-E_{\beta}\prt{\brk{P^{\pi^k}_h\nu^{\pi^k}_{h+1}}(s^k_h)}    }\\
&=\underbrace{ \exp(\beta r^k_h) \prt{ E_{\beta}\prt{O_{c^k_h}\prt{\brk{\hat{P}^{\pi^k}_h\nu^k_{h+1}}(s^k_h)}}- E_{\beta}\prt{\brk{\hat{P}^{\pi^k}_h\nu^k_{h+1}}(s^k_h)} } }_{(a)}\\
&+ \underbrace{ \exp(\beta r^k_h) \prt{ E_{\beta}\prt{\brk{\hat{P}^{\pi^k}_h\nu^k_{h+1}}(s^k_h)}-E_{\beta}\prt{\brk{P^{\pi^k}_h\nu^k_{h+1}}(s^k_h)} } }_{(b)} \\ 
&+\underbrace{ \exp(\beta r^k_h) \prt{ E_{\beta}\prt{\brk{P^{\pi^k}_h\nu^k_{h+1}}(s^k_h)}-E_{\beta}\prt{\brk{P^{\pi^k}_h\nu^{\pi^k}_{h+1}}(s^k_h)} } }_{(d)}.
\end{align*}
Using the Lipschitz property of EU, we have
\begin{align*}
(a) &\leq  \exp(\beta r^k_h) \cdot L_{H-h}\norm{\rO^{\infty}_{c^k_h}\prt{\brk{\hat{P}^{\pi^k}_h\nu^k_{h+1}}(s^k_h)}-\brk{\hat{P}^{\pi^k}_h\nu^k_{h+1}}(s^k_h)  }_{\infty}\\
&\leq \exp(\beta r^k_h) \cdot L_{H-h} c^k_h\\
&\leq \exp(\beta) (\exp(\beta(H-h))-1) c^k_h \\
&\leq (\exp(\beta(H+1-h))-1) \sqrt{\frac{2S}{(N^k_h\vee1)}\iota}.
\end{align*}
We can bound $(b)$ as
\begin{align*}
(b)&=\exp(\beta r^k_h) \prt{E_{\beta}\prt{\brk{\hat{P}^{\pi^k}_h\nu^k_{h+1}}(s^k_h)}-E_{\beta}\prt{\brk{P^{\pi^k}_h\nu^k_{h+1}}(s^k_h)}}\\
&\leq \exp(\beta) L_{H-h}\norm{\brk{\hat{P}^{\pi^k}_h\nu^k_{h+1}}(s^k_h)-\brk{P^{\pi^k}_h\nu^k_{h+1}}(s^k_h)}_{\infty}\\
&\leq \exp(\beta) (\exp(\beta (H-h))-1)\norm{\hat{P}^{\pi^k}_h(s^k_h)-P^{\pi^k}_h(s^k_h)}_1\\
&\leq (\exp(\beta(H+1-h))-1)\sqrt{\frac{2S}{(N^k_h\vee1)}\iota},
\end{align*}
where the second inequality is due to Lemma \ref{lem:mix_dis}. Observe that
\begin{align*}
(c) = \exp(\beta r^k_h)\brk{P^{\pi^k}_h(V^k_{h+1}-V^{\pi^k}_{h+1})}(s^k_h)= \exp(\beta r^k_h)\brk{P^{\pi^k}_h\Delta^k_{h+1}}(s^k_h) = \exp(\beta r^k_h)(\epsilon^k_h+\delta^k_{h+1}),
\end{align*}
where $\epsilon^k_h\triangleq[P^{\pi^k}_h \Delta^k_{h+1}](s^k_h) -\Delta^k_{h+1}(s^k_{h+1})$ is a  martingale difference sequence with $\epsilon^k_h \in 2D_{h+1}$ a.s. for all $(k,h)\in[K]\times[H]$, and $e^k_h\triangleq\norm{\hat{P}^k_h(s^k_h)-P^{\pi^k}_h(s^k_h)}_1$. Since $(a)+(b)\leq  2L_{H+1-h} c^k_h$, we can bound $\delta^k_h$ recursively as 
\begin{align*}
\delta^k_h &\leq 2L_{H+1-h}c^k_h + \exp(\beta r^k_h)(\epsilon^k_h+\delta^k_{h+1}).
\end{align*}
Repeating the procedure, we get
\begin{align*}
\delta^k_1 &\leq 2\sum^{H-1}_{h=1}L_{H+1-h} \prod^{h-1}_{i=1}\exp(\beta r^k_i) c^k_{h} + \sum^{H-1}_{h=1} \prod^{h}_{i=1}\exp(\beta r^k_i) \epsilon^k_{h}+\prod^{H-1}_{i=1}\exp(\beta r^k_i) \delta^k_{H} \\
&\leq  2\sum^{H-1}_{h=1}(\exp(\beta(H+1-h))-1) \exp(\beta (h-1)) c^k_{h} + \sum^{H-1}_{h=1} \prod^{h}_{i=1}\exp(\beta r^k_i) \epsilon^k_{h}+ \exp(\beta (H-1)) \delta^k_{H}\\
&\leq 2\sum^{H-1}_{h=1}(\exp(\beta H)-1)c^k_{h} + \sum^{H-1}_{h=1} \prod^{h}_{i=1}\exp(\beta r^k_i) \epsilon^k_{h}+ \exp(\beta (H-1)) \delta^k_{H}.
\end{align*}
Thus
\[  \sum_{k=1}^K\delta^k_1\leq 2(\exp(\beta H)-1)\sum^K_{k=1}\sum^{H-1}_{h=1}c^k_{h} + \sum^K_{k=1}\sum^{H-1}_{h=1} \prod^{h}_{i=1}\exp(\beta r^k_i) \epsilon^k_{h}+\sum^K_{k=1} \exp(\beta (H-1)) \delta^k_{H}.\]
Now we bound each term separably. The first  term can be bounded  as
\begin{align*}
2(\exp(\beta (H+1))-1)\sum^K_{k=1}\sum^{H-1}_{h=1}c^k_h &=
2(\exp(\beta (H+1))-1)\sum^{H-1}_{h=1}\sum^K_{k=1} \sqrt{\frac{2S}{(N^k_h\vee1)}\iota}\\
&\leq 4(\exp(\beta (H+1))-1)\sum^{H-1}_{h=1}\sqrt{2S^2AK\iota}\\
&= 4(\exp(\beta (H+1))-1)(H-1)\sqrt{2S^2AK\iota}.
\end{align*}
Observe that
\begin{align*}
\prod^{h}_{i=1}\exp(\beta r^k_i) \epsilon^k_{h}\in \exp(\beta h)D_h&=\exp(\beta h)[1-\exp(\beta (H+1-h)),\exp(\beta (H+1-h))-1]\\
&\subseteq [1-\exp(\beta (H+1)),\exp(\beta (H+1))-1], 
\end{align*}
thus we can bound the second term by Azuma-Hoeffding inequality: with probability at least $1-\delta^{\prime}$, the following holds
\[ \sum^K_{k=1}\sum^{H-1}_{h=1}\prod^{h}_{i=1}\exp(\beta r^k_i)\epsilon^k_{h}\leq (\exp(\beta (H+1))-1)\sqrt{2KH\log(1/\delta^{\prime})}.\]
The third term can be bounded as
\begin{align*}
    &\sum^K_{k=1} \exp(\beta (H-1)) \delta^k_{H}=\exp(\beta (H-1))\sum^K_{k=1} W^k_H(s^k_H) -W^{\pi^k}_H(s^k_H)\\
    &=\exp(\beta (H-1)) \sum^K_{k=1} \mathbb{I}\{N^k_H=0\}\exp(\beta)+\mathbb{I}\{N^k_H>0\}\exp(\beta r_H(s^k_H))-\exp(\beta r_H(s^k_H))\\
    &\leq \exp(\beta (H-1))  (\exp(\beta)-1)\sum^K_{k=1} \mathbb{I}\{N^k_H=0\} \\
    &\leq (\exp(\beta H)-1) SA < (\exp(\beta (H+1))-1)SA
\end{align*}
Using a union bound and let $\delta=\delta^{\prime}=\frac{\tilde{\delta}}{2}$, we have that with probability at least $1-\delta$,
\begin{align*}
\text{Regret}(K)&\leq \frac{\exp(\beta (H+1))-1}{\beta}\prt{  4(H-1)\sqrt{2S^2AK\iota} +  \sqrt{2KH\iota} +  SA}\\
&=\tilde{\mathcal{O}}\prt{\frac{\exp(\beta H)-1}{\beta }\sqrt{HS^2AT}},
\end{align*}
where $\iota\triangleq\log(2SAT/\delta)$.

Now we consider the case $\beta<0$. Using similar arguments, we arrive at
\begin{align*}
\text{Regret}(K)&\leq \sum^K_{k=1}\frac{1}{\beta}\log\prt{W^{k}_1(s^k_1)}-\frac{1}{\beta}\log\prt{W^{\pi^k}_1(s^k_1)}\\
&= \sum^K_{k=1}\frac{1}{-\beta}\prt{\log\prt{W^{\pi^k}_1(s^k_1)} - \log\prt{W^{k}_1(s^k_1)} } \\
&\leq \frac{1}{-\beta\exp(\beta H)}\sum^K_{k=1}  W^{\pi^k}_1(s^k_1)-W^{k}_1(s^k_1),
\end{align*}
where the last inequality is due to that both $W^{\pi^k}_1(s^k_1)$ and $W^{k}_1(s^k_1)$ is larger than or equal to $\exp(\beta H)$. We can finally get 
\begin{align*}
    \text{Regret}(K)&\leq \tilde{\mathcal{O}}\prt{\frac{1-\exp(\beta H)}{-\beta\exp(\beta H)}\sqrt{HS^2AT}}=\tilde{\mathcal{O}}\prt{\frac{\exp(|\beta| H)-1}{|\beta|}\sqrt{HS^2AT}}.
\end{align*}
\end{proof}
\subsection{Computational inefficiency of \texttt{RODI-MF}} 
While \texttt{RODI-MF} enjoys near-optimal regret guarantee, it suffers from computational inefficiency, especially in contexts with a large number of states or a long horizon. For better illustration, let's consider a \textit{Markov Reward Process} with $S$ states at each step. In this setup, the transition kernel is uniform ($P_h(s'|s)=1/S$) for any $(h,s')\in[H-1]\times\cS$, and the reward function is bounded between 0 and 1 ($r_h(s)\in[0,1]$). Starting from the final step $H$, the return distribution $\eta_H(s)$ is a Dirac delta function centered at $r_H(s)$.  Applying the distributional Bellman equation at step at step $H-1$, we get
\[ \eta_{H-1}(s) = \sum_{s'} p_{H-1}(s'|s)\delta_{r_H(s')+r_{H-1}(s)}. \]
Recall that  $|\eta|$ represents the number of atoms (distinct elements) in a discrete distribution $\eta$, indicating the memory required to store this distribution.  Since $|\eta_H(s)|=|\delta_{r_H(s)}|=1$ for each $s\in\cS$, and $\eta_{H-1}(s)$ is a uniform mixture of all $\eta_{H}(s)$ shifted by $r_{H-1}(s)$, we find
\[  |\eta_{H-1}(s)|= \abs{\prt{r_{H-1}(s)+r_H(s'),\frac{1}{S}}_{s'\in\cS}}=\cO(S).\]
Continuing this process backwards through the time steps:
\begin{align*}
    &\abs{\eta_{H-2}(s)} = \cO(S^2) \\
    &\cdots \\\
    &|\eta_1(s)|=\cO(S^{H-1}).
\end{align*}
This analysis shows that the number of atoms in the return distribution \textit{exponentially} increases with the horizon $H$, scaled by the number of states $S$ at each application of the distributional Bellman operator. As a result, the memory and computational requirements to implement an \textit{exact} distributional RL algorithm like \texttt{RODI-MF} become prohibitive, particularly for problems with many states or a long horizon. This exponential growth in complexity highlights the computational challenges associated with \texttt{RODI-MF} and underscores the need for \textit{approximations} for practical implementations. 

\section{DRL with Distribution Representation}
\label{sec:rep}
To address the computational challenges in implementing the distributional Bellman equation, we introduce two versions of \texttt{RODI-MF} in the revised paper that utilize \emph{distribution representation}. A widely used method of distribution representation is the \emph{categorical representation}, as discussed in recent literature \citep{bdr2023}. This approach parameterizes the probability distribution at fixed locations. Specifically, we consider the simplest form of categorical representation that uses only two atoms. We refer to this as the \textit{Bernoulli representation}. It represents the set of all discrete distributions with two distinct atoms, denoted as $\theta=(\theta_1,\theta_2)$. The Bernoulli representation is formally defined as:
$$
\mathscr{F}_{\mathrm{B}}(\theta)=\left\{(1-p) \delta_{\theta_1}+p \delta_{\theta_2}: p \in[0,1]\right\} .
$$
With the Bernoulli representation in mind, let's consider distributional Bellman operator
\[ \eta_{h}(s,a) = [\mathcal{T}_h \nu_{h+1}](s,a)=\sum_{s'}P_h(s'|s,a)\nu_{h+1}(s')(\cdot-r_h(s,a)). \]
This operator essentially performs two basic operations: shifting and mixing. Specifically, it shifts each next-step return distribution by the reward $r_h(s,a)$ and then takes a mixture of these shifted distributions with the mixture coefficients $P_h(s,a)$. However, these operations might change and expand the support of the distributions, leading to:
\[ |\eta_{h}(s,a)| = \cO(S) |\nu_{h+1}|.\]
To improve computational efficiency, we introduce the Bernoulli representation for $\mathcal{T}_h$. Let
\[ \bar{\nu}_{h+1}(s)=(L_{h+1}(s),R_{h+1}(s);q_{h+1}(s))\in \mathscr{F}_{\mathrm{B}}(L_{h+1}(s),R_{h+1}(s))  \]
be a Bernoulli representation of the true return distribution $\nu_{h+1}(s)$, where $L_{h+1}(s)$ and $R_{h+1}(s)$ are the left and right atoms, and $q_{h+1}(s)$ is the probability at $R_{h+1}(s)$. Applying $\mathcal{T}_h$ to $\bar{\nu}_{h+1}$, we obtain
\[ [\mathcal{T}_h \bar{\nu}_{h+1}](s,a)=\prt{r_h(s,a)+L_{h+1}(s'),r_h(s,a)+R_{h+1}(s');p_h(s'|s,a)q_{h+1}(s')}_{s'\in \cS} \not\in \mathscr{F}_{\mathrm{B}}. \]
The result is no longer a Bernoulli distribution but a categorical distribution with $2S$ atoms. This demonstrates that the Bernoulli representation is not \emph{closed} under $\mathcal{T}_h$
\[  \nu \in \mathscr{F}_{\mathrm{B}} \not\Longrightarrow \mathcal{T}_h \nu \in \mathscr{F}_{\mathrm{B}}.\]
To overcome this issue, we introduce the \textit{Bernoulli projection operator}.  This operator serves as a mapping from the space of probability distributions to $\mathscr{F}_{\mathrm{B}}$, and we denote it as $\Pi:\mathscr{P}(\bR)\mapsto\mathscr{F}_{\mathrm{B}}$. Algorithmically, we add a projection step immediately after the application of $\mathcal{T}$, resulting in a \textit{projected distributional Bellman operator} $\Pi \mathcal{T}$. This projection ensures that each iteration of $\eta_{h}=\Pi \mathcal{T}_h \nu_{h+1}$ is representable using a limited amount of memory.

Note that the projection operator is not unique. Previous works \citep{bdr2023} have developed projection operators aiming to find the best approximation to a given probability distribution, as measured by a specific probability metric. Our approach introduces a novel type of Bernoulli projection that \emph{preserves the ERM value}, an essential aspect in risk-sensitive settings. Starting from a Dirac measure $\delta_c$, we define the \emph{value-equivalent Bernoulli projection operator} as:
\[ \Pi\delta_c\triangleq (1-q(c;\theta))\delta_{\theta_1}+q(c;\theta)\delta_{\theta_2}=(\theta_1,\theta_2;q(c;\theta)),\]
where the probability is defined as
\begin{equation}
\label{eqt:ber_prj}
   q(c;\theta)=\frac{e^{\beta c}-e^{\beta \theta_1}}{e^{\beta \theta_2}-e^{\beta \theta_1}}\in[0,1]. 
\end{equation}
It is easy to verify that $U_{\beta}(\Pi\delta_c)=U_{\beta}(\delta_c)=c, \forall c \in [\theta_1,\theta_2]$. Now we extend the definition to the categorical distributions as:
\begin{align*}
\Pi\prt{c_i,p_i}_{i\in[n]}&=\Pi\prt{\sum_{i\in[n]}p_i\delta_{c_i}}\triangleq\sum_i p_i \Pi \delta_{c_i} =\sum_i p_i \prt{ (1-q(c_i;\theta))\delta_{\theta_1}+q(c_i;\theta)\delta_{\theta_2} }\\
    &= \prt{\sum_i p_i (1-q(c_i;\theta))}\cdot\delta_{\theta_1}+\prt{\sum_i p_i q(c_i;\theta)}\cdot\delta_{\theta_2}\\
    &=\prt{\theta_1,\theta_2;\sum_i p_i q(c_i;\theta)}.
\end{align*}
Given that $\eu{\delta_{c_i}} =\eu{\Pi \delta_{c_i}}, \forall i \in[n]$, the linearity of EU implies
\begin{align*}
    \eu{\sum_i p_i\delta_{c_i}}=&\sum_i p_i\eu{\delta_{c_i}}=\sum_i p_i\eu{\Pi \delta_{c_i}} =\eu{\sum_i p_i\Pi \delta_{c_i}}=\eu{\Pi \sum_i p_i\delta_{c_i}}.
\end{align*}
This verifies the value equivalence of $\Pi$. 

To ensure the preservation of the value, the only requirement is that the interval  $[\theta_1,\theta_2]$ covers the support of the input distribution, i.e., $\theta_1\leq \min c_i\leq \max c_i\leq\theta_2$. The projection preserves the risk value of the original distribution, making it a powerful tool for efficient and accurate representation in DRL for RSRL.

Without the knowledge of MDP, \texttt{RODI-MF} deviates from the DDP in two crucial updates:
\begin{align*}
    \hat{\eta}_h &\leftarrow \hat{\mathcal{T}}_h \nu_{h+1} \\
    \tilde{\eta}_h &\leftarrow \rO_c \hat{\eta}_h.
\end{align*}
In \texttt{RODI-MF}, the \textit{approximate distributional Bellman operator} $\hat{\mathcal{T}}$ is applied first, which relies on the empirical transition $\hat{P}$ rather than the true transition $P$. Then, the \textit{distributional optimism operator} $\rO_c$ is used to generate an optimistic return distribution. Drawing from these observations, we propose two DRL algorithms with Bernoulli representation, differing in the order of projection and optimism operator. We term the two algorithms as \texttt{RODI-Rep}.

\subsection{DRL with Bernoulli Representation}
Given that $\eta_h\in D_{H+1-h}$, we set the \emph{uniform} location parameters as
\[  L_{h}\triangleq0, \quad R_{h}\triangleq H+1-h,\]
which is independent of $(s,a)$. We represent each iterate by a Bernoulli distribution
\[ \eta^k_h(s,a) = (1-q^k_h(s,a))\delta_{L_h}+q^k_h(s,a)\delta_{R_h}, \nu^k_h(s) = (1-q^k_h(s))\delta_{L_h}+q^k_h(s)\delta_{R_h},\]
where we overload the notation for $q^k_h(s,a)$ and $q^k_h(s)$. Applying the approximate $\cT_h$ to the Bernoulli represented $\nu^k_{h+1}\in \mathscr{F}_{\mathrm{B}}$ yields
\begin{align*}
    \eta^k_h(s,a)&=[\hat{\cT}_h\nu^k_{h+1}](s,a)= \sum_{s'\in\cS}\hat{P}^k_h(s'|s,a)\nu^k_{h+1}(s')(\cdot-r_h(s,a))\\
    &= \sum_{s'\in\cS}\hat{P}^k_h(s'|s,a)((1-q^k_h(s'))\delta_{L_{h+1}}+q^k_h(s')\delta_{R_{h+1}})(\cdot-r_h(s,a)) \\
    &= (1-[\hat{P}^k_h q^k_{h+1}](s,a))\cdot\delta_{r_h(s,a)+L_{h+1}} + [\hat{P}^k_h q^k_{h+1}](s,a)\cdot\delta_{r_h(s,a)+R_{h+1}} \\
    &= \prt{r_h(s,a)+L_{h+1},r_h(s,a)+R_{h+1};[\hat{P}^k_h q^k_{h+1}](s,a)}. 
\end{align*}
With slight abuse of notation, we let 
\[ L_h(s,a)\triangleq r_h(s,a)+L_{h+1}, \quad R_h(s,a)\triangleq r_h(s,a)+R_{h+1}. \]
$\eta^k_h(s,a)=\prt{L_h(s,a),R_h(s,a);[\hat{P}^k_h q^k_{h+1}](s,a)}$ is a Bernoulli distribution with support not corresponding to $L_h$ and $R_h$. Now, we propose two different algorithms differing in the order of projection and optimism operator.
\paragraph{Optimism-Then-Projection.} 
\texttt{RODI-OTP} applies the optimism operator first, followed by the projection operator:
\begin{align*}
    \eta^k_h &\leftarrow \Pi \rO_c \hat{\mathcal{T}}_h \nu^k_{h+1}. 
\end{align*}
Note that $\eta^k_h \leftarrow\hat{\mathcal{T}}_h \nu^k_{h+1}\in\mathscr{F}_{\mathrm{B}}(r_h(s,a)+L_{h+1},r_h(s,a)+R_{h+1})$. For Bernoulli distribution, the optimism operator admits a simple form
$$\rO_c \prt{a,b;p} = \prt{a,b;\min(p+c,1)}.$$
Applying optimism operator to $\eta^k_h$ yields
\begin{align*}
    \rO_{c^k_h(s,a)}\prt{\eta^k_h(s,a)}&=\rO_{c^k_h(s,a)}\prt{L_h(s,a), R_h(s,a); [\hat{P}^k_h q^k_{h+1}](s,a)}\\
        &=\prt{L_h(s,a), R_h(s,a); \min\prt{[\hat{P}^k_h q^k_{h+1}](s,a)+c^k_h(s,a),1}}.
\end{align*}
We can simplify the update in a parametric form 
\begin{align*}
    q^k_{h}(s,a) &\leftarrow [\hat{P}^k_h q^k_{h+1}](s,a)  \quad \text{parametric Bellman update}\\
    q^k_{h}(s,a) &\leftarrow \min(q^k_{h}(s,a)+c^k_h(s,a),1) \quad \text{optimism operator}.
\end{align*}
Finally, we apply the projection rule (cf. Equation \ref{eqt:ber_prj}) to obtain
\begin{align*}
    q^k_{h}(s,a) &\leftarrow (1-q^k_{h}(s,a))q(L_h(s,a);L_h,R_h)+q^k_{h}(s,a)q(R_h(s,a);L_h,R_h) \\
    &= (1-q^k_{h}(s,a))q^L_h(s,a)+q^k_{h}(s,a)q^R_h(s,a),
\end{align*}
where 
\begin{align*}
    &q^R_h(s,a)\triangleq q(L_h(s,a);L_h,R_h) = \frac{e^{\beta(r_h(s,a)+H-h)}-1}{e^{\beta (H+1-h)}-1}, \\
    &q^L_h(s,a)\triangleq q(R_h(s,a);L_h,R_h) = \frac{e^{\beta r_h(s,a)}-1}{e^{\beta (H+1-h)}-1}.
\end{align*}
\vspace{-2ex}
\begin{remark}
$q^R_h(s,a)$ and $q^L_h(s,a)$ are \emph{fixed} (independent of $k$) and \emph{known}. Therefore we can compute their values for all $(h,s,a)$ in advance.
\end{remark}
\paragraph{Projection-Then-Optimism.} \texttt{RODI-PTO} applies the projection operator first, followed by the optimism operator:
\begin{align*}
    \eta^k_h &\leftarrow \rO_c \Pi \hat{\mathcal{T}}_h \nu^k_{h+1}. 
\end{align*}
The update can be represented in a parametric form
\begin{align*}
    q^k_{h}(s,a) &\leftarrow [\hat{P}^k_h q^k_{h+1}](s,a)  \quad \text{parametric Bellman update}\\
    q^k_{h}(s,a) &\leftarrow (1-q^k_{h}(s,a))q^L_h(s,a)+q^k_{h}(s,a)q^R_h(s,a) \quad \text{projection operator} \\
    q^k_{h}(s,a) &\leftarrow \min(q^k_{h}(s,a)+c^k_h(s,a),1) \quad \text{optimism operator}.
\end{align*}
After applying optimism operator and projection operator, both \texttt{RODI-OTP} and \texttt{RODI-PTO} update the value functions and policies accordingly
\begin{align*}
    Q^k_h(s,a)&\leftarrow \frac{1}{\beta}\log\prt{1-q^k_{h}(s,a)+q^k_{h}(s,a)e^{\beta(H+1-h)}} \\
    \pi^k_h(s)&\leftarrow \arg\max_a Q^k_h(s,a), V^k_h(s)\leftarrow Q^k_h(s,\pi^k_h(s)) \\
    q^k_h(s)&\leftarrow q^k_h(s,\pi^k_h(s)).
\end{align*}
\textbf{Computational complexity.} The \emph{time complexity} of \texttt{RODI-OTP} and \texttt{RODI-PTO} is given as follows: i) computation of $q^L$ and $q^R$: $\cO(HSA)$; ii) parametric Bellman update: $KHSA\cdot\cO(S)$; iii) projection: $KHSA\cdot\cO(1)$; iv) optimism operator: $KHSA\cdot\cO(1)$; v) computation of $Q$-function: $KHSA\cdot\cO(1)$; vi) greedy policy: $KHS\cdot\cO(A\log A)$. Therefore, the total time complexity is given by 
\[  \cO(KHSA(S+\log A), \]
which is the same as that of \texttt{RSVI2}. The \emph{space complexity} of both algorithm is given as follows: i) $q^L$ and $q^R$: $\cO(HSA)$; ii) $N_h(s,a)$: $\cO(HSA)$; iii) trajectory $(s^k_h,a^k_h)_{k,h}$: $\cO(T)$; iv) probabilities $q_h(s,a)$: $\cO(HSA)$; v) action-value function: $\cO(HSA)$. Therefore, their total space complexity is $\cO(HSA+T)$.

\subsection{Optimism of DRL with Representation}
While \texttt{RODI-OTP} and \texttt{RODI-PTO} adapts \texttt{RODI-MF} by Bernoulli representation, they maintain the optimism mainly due to the value-equivalence property of the projection operator.
\paragraph{Optimism of \texttt{RODI-OTP}.} For simplicity, we rewrite the update formula of \texttt{RODI-OTP} as
\begin{align*}
    \hat{\eta}_h(s,a)&=[\hat{\cT}_h\nu_{h+1}](s,a)=\brk{\hat{P}_h \nu_{h+1}}[s,a](\cdot-r_h(s,a))\\
    \tilde{\eta}_h(s,a)&=\rO_{c_h(s,a)}\hat{\eta}_h(s,a)\\
    \eta_h(s,a)&=\Pi \tilde{\eta}_h(s,a)\\
    Q_h(s,a)&=\erm{\eta_h(s,a)},\pi_h(s)=\arg\max_a Q_h(s,a)\\
    \nu_h(s)&=\eta_h(s,\pi_h(s)).
\end{align*}
Define
\[ \check{\eta}_h(s,a)\triangleq[\cT_h\nu_{h+1}](s,a)=\brk{P_h \nu_{h+1}}[s,a](\cdot-r_h(s,a)), \]
which is the Bellman target that replaces $\hat{P}_h$ by the true model $P_h$. Note that $\nu_{h+1}\in \mathscr{F}_{\mathrm{B}}$ is the distribution generated by the algorithm, which is Bernoulli represented, rather than the optimal distribution $\nu^*_{h+1}$. Since 
\begin{align*}
    \norm{\check{\eta}_h(s,a)-\hat{\eta}_h(s,a)}_{\infty} &= \norm{\brk{\hat{P}_h \nu_{h+1}}[s,a](\cdot-r_h(s,a)) - \brk{P_h \nu_{h+1}}[s,a](\cdot-r_h(s,a))}_{\infty} \\
    &= \norm{\brk{\hat{P}_h \nu_{h+1}}[s,a] - \brk{P_h \nu_{h+1}}[s,a]}_{\infty} \\
    &\leq \norm{\hat{P}_h(s,a)-P_h(s,a)}_1\leq c_h(s,a),
\end{align*}
we have
\[  \tilde{\eta}_h(s,a)=\rO_{c_h(s,a)}\hat{\eta}_h(s,a)\succeq  \check{\eta}_h(s,a).\]
We can prove the argument by induction. Fix $h+1\in[2:H+1]$. Suppose $V_{h+1}=\erm{\eta_{h+1}}\ge\erm{\eta^*_{h+1}}=V^*_{h+1}$ for any $s$. It follows that 
\begin{align*}
    Q_h(s,a)&=\erm{\eta_h(s,a)}=\erm{\Pi \tilde{\eta}_h(s,a)}=\erm{\tilde{\eta}_h(s,a)}=\erm{\rO_{c_h(s,a)}\hat{\eta}_h(s,a)}\\
    &\ge\erm{\check{\eta}_h(s,a)}=\erm{\cT_h \nu_{h+1}}\\
    &\ge \erm{\cT_h \nu^*_{h+1}}=Q^*_h(s,a),
\end{align*}
which implies $V_h(s)\ge V_h^*(s)$ for any $s$. The induction is completed.
\paragraph{Optimism of \texttt{RODI-PTO}.}
We rewrite the update of $q_h(s,a)$ in \texttt{RODI-PTO} as:
\begin{align*}
    \hat{q}_{h}(s,a) &\leftarrow [\hat{P}_h q_{h+1}](s,a), \hat{\eta}_{h}(s,a)=\prt{L_h(s,a),R_h(s,a);\hat{q}_h(s,a)}   \\
    \bar{q}_{h}(s,a) &\leftarrow (1-\hat{q}_{h}(s,a))q^L_h(s,a)+\hat{q}_{h}(s,a)q^R_h(s,a),\bar{\eta}_{h}(s,a)=\prt{L_h,R_h;\bar{q}_h(s,a)}  \\
    q_{h}(s,a) &\leftarrow \min(\bar{q}_{h}(s,a)+c_h(s,a),1), \eta_h(s,a)=\prt{L_h,R_h;q_h(s,a)}.
\end{align*}
Define 
\begin{align*}
\check{q}_h(s,a)\triangleq \brk{P_h q_{h+1}}[s,a], \quad
\check{\eta}_h(s,a)\triangleq\prt{L_h(s,a),R_h(s,a);\check{q}_h(s,a)},
\end{align*}
then we have
\[ \Pi \check{\eta}_h(s,a) = \prt{L_h,R_h; (1-\check{q}_{h}(s,a))q^L_h(s,a)+\check{q}_{h}(s,a)q^R_h(s,a)}.\]
$\check{\eta}_h(s,a)$ and $\hat{\eta}_h(s,a)$ are both Bernoulli distributions with the same support, thus
\[ \norm{\check{\eta}_h(s,a)-\hat{\eta}_h(s,a)}_{\infty}=|\check{q}_h(s,a)-\hat{q}_h(s,a)|=\abs{\brk{(\hat{P}_h-P_h)q_{h+1}}(s,a)}\leq \norm{(\hat{P}_h-P_h)(s,a)}_1. \]
We have
\begin{align*}
   &\norm{\Pi\check{\eta}_h(s,a)-\Pi\hat{\eta}_h(s,a)}_{\infty}=\\
   &\abs{(1-\check{q}_{h}(s,a))q^L_h(s,a)+\check{q}_{h}(s,a)q^R_h(s,a) - (1-\hat{q}_{h}(s,a))q^L_h(s,a)-\hat{q}_{h}(s,a)q^R_h(s,a)}\\
   &=\abs{(\check{q}_{h}(s,a)-\hat{q}_{h}(s,a))(q^R_h(s,a)-q^L_h(s,a))}\\
   &=\abs{[(\hat{P}_h-P_h)q_{h+1}](s,a)(q^R_h(s,a)-q^L_h(s,a))}\\
   &= (q^R_h(s,a)-q^L_h(s,a))\norm{\check{\eta}_h(s,a)-\hat{\eta}_h(s,a)}_{\infty}\\
   &\leq (q^R_h(s,a)-q^L_h(s,a))\norm{\hat{P}_h(s,a)-P_h(s,a)}_1 \leq (q^R_h(s,a)-q^L_h(s,a))c_h(s,a)<c_h(s,a).
\end{align*}
Suppose $V_{h+1}=\erm{\eta_{h+1}}\ge\erm{\eta^*_{h+1}}=V^*_{h+1}$ for any $s$. Since $\eta_h(s,a)=\rO_{c_h(s,a)} \Pi\hat{\eta}_h(s,a) \succeq \Pi\check{\eta}_h(s,a)$, we have
\begin{align*}
    Q_h(s,a)&=\erm{\eta_h(s,a)}=\erm{\rO_{c_h(s,a)} \bar{\eta}_h(s,a)}=\erm{\rO_{c_h(s,a)} \Pi_h\hat{\eta}_h(s,a)}\ge\erm{\Pi_h\check{\eta}_h(s,a)}\\
    &=\erm{\check{\eta}_h(s,a)}=\erm{\brk{\cT_h \nu_{h+1}}(s,a))}\ge \erm{\brk{\cT_h \nu^*_{h+1}}(s,a))}=Q^*_h(s,a).
\end{align*}
which implies $V_h(s)\ge V_h^*(s)$ for any $s$. The induction is completed.

\section{RODI-MB}
\label{sec:rodi-mb}
We introduce the \hspace{-1ex}\tb{M}odel-\hspace{-1ex}\tb{B}ased \textbf{R}isk-sensitive  \textbf{O}ptimistic \textbf{D}istribution \textbf{I}teration algorithm (\texttt{RODI-MB}, cf. Algorithm \ref{alg:RODI-MB}). Unlike its model-free counterpart, \texttt{RODI-MB} explicitly maintains and updates an empirical transition model within each episode, making it a model-based approach. However, \texttt{RODI-MB} also encounters issues with computational inefficiency. Remarkably, \texttt{RODI-MB} is equivalent to a non-distributional RL algorithm (Algorithm \ref{alg:ROVI}). This equivalence results in computational efficiency, as it operates on one-dimensional values rather than full distributions. 

\emph{Planning phase} (Line 5-14) Mirroring the structure of Algorithm \ref{alg:RODI-MF}, \texttt{RODI-MB} also employs approximate DDP in conjunction with the OFU principle. Initially, it applies the distributional optimism operator to the empirical transition model $\hat{P}^k_h$, resulting in an optimistic transition model $\tilde{P}^k_h$. The algorithm then utilizes this optimistic model for the Bellman update, generating optimistic return distributions $\eta^k_h$. The subsequent steps remain consistent with those outlined in Algorithm \ref{alg:RODI-MF}.
\begin{algorithm}[H]
	\caption{\texttt{RODI-MB}}
	\label{alg:RODI-MB}
	\begin{algorithmic}[1]
		\State{Input: $T$ and $\delta$}
		\State{$N^1_h(\cdot,\cdot)\leftarrow0$; $\hat{P}^1_h(\cdot,\cdot)\leftarrow \frac{1}{S}\textbf{1}$  $\forall h\in[H]$}
		\For{$k = 1:K$}
		\State{$\nu^k_{H+1}(\cdot)\leftarrow\psi_0$}
		\For{$h = H:1$}
		\If{$N^k_h(\cdot,\cdot)>0$}
		    \State{$\tilde{P}^k_h(\cdot,\cdot)\leftarrow \rO^{1}_{c^k_{h}(\cdot,\cdot)}\prt{\hat{P}^k_h(\cdot,\cdot),\nu^k_{h+1}}$}
		    \State{$\eta^k_h(\cdot,\cdot)\leftarrow
			\brk{\cT\prt{\tilde{P}^k_h,r_h}\nu^{k}_{h+1}}(\cdot,\cdot)$}
		\Else{}
		    \State{$\eta^k_h(\cdot,\cdot)\leftarrow \delta_{H+1-h}$}
		\EndIf{}
		\State{$\pi^k_h(\cdot)\leftarrow\arg\max_{a}U_{\beta}(\eta^k_h(\cdot,a))$}
		\State{$\nu^k_h(\cdot)\leftarrow\eta^k_h(\cdot,\pi^k_h(\cdot))$}
		\EndFor
		\State{Receive $s^k_1$}
		\For{$h = 1:H$}
		\State{$a^k_h\leftarrow\pi^k_h(s^k_h)$ and transit to $s^k_{h+1}$}
		\State{Compute $N^{k+1}_h(\cdot,\cdot)$ and $\hat{P}^{k+1}_h(\cdot,\cdot)$ }
		\EndFor
		\EndFor
	\end{algorithmic}  
\end{algorithm}
\emph{Interaction phase} (Line 16-19) During the interaction phase, the agent engages with the environment using the policy $\pi^k$ and updates the counts $N^{k+1}_h$ and the empirical transition model $\hat{P}^{k+1}_h$ based on newly acquired observations.

\subsection{Distributional Optimism over the Model} 
In the \texttt{RODI-MB} algorithm, we introduce a nuanced approach to generating an optimistic transition model, $\tilde{P}^k_h(s,a)$, from the empirical transition model, $\widehat{P}^k_{h}(s,a)$. This approach is based on the concept of distributional optimism over the space of PMFs rather than CDFs. Specifically, the goal is to compute a return distribution, $\eta^k_h$, from $\tilde{P}^k_h(s,a)$ and the future return $\nu^k_{h+1}$, such that $\eta^k_h\ge\eta^*_h$ with high probability.

The distributional optimism operator, $\rO_c^1$, is defined for PMFs over the space $\sD(\cS)$ with a level $c$, and it operates differently from $\rO_c^{\infty}$ by also considering the future return distribution $\nu$:
\[	\rO_c^1\prt{\widehat{P}(s,a),\nu}\triangleq	\arg\max_{P\in B_1(\widehat{P}(s,a),c)}U_{\beta}([P\nu]).\]
This operator selects a model from within the $\ell_1$ norm ball, $B_1(\widehat{P}(s,a),c)$, that yields the largest EntRM value, $U_{\beta}([P\nu])$. This approach ensures that $\rO_c^1$ generates a model with optimistically biased estimates of the future returns, and it leverages an efficient method to achieve this (as detailed in Appendix \ref{app:property}).

Given that Lemma \ref{lem:main_event} assures the high-probability event $\mathcal{G}_{\delta}$, the analysis primarily focuses on scenarios conditioned on $\mathcal{G}_{\delta}$. Additionally, due to the equivalence between EntRM and EU, the verification of optimism is conducted in terms of EU for $\beta>0$.
\begin{lemma}[Optimistic model]
	\label{lem:opt_model}
For any $(s,a,k,h)$ and $P\in B_1(\hat{P}^k_h(s,a),c^k_{h}(s,a))$, we have
\[  E_{\beta}\prt{\brk{\tilde{P}^k_h\nu^k_{h+1}}(s,a)}\ge E_{\beta}\prt{\brk{P\nu^k_{h+1}}(s,a)}.\]
\end{lemma}
\begin{proof}
    Use the definition of $\rO_c^1$ and the  equivalence between EntRM and EU.
\end{proof}
\vspace{-3.5ex}
\begin{lemma}[Optimism]
Conditioned on event $\mathcal{G}_{\delta}$, the sequence $\{W^k_1(s^k_1)\}_{k\in[K]}$ produced by Algorithm \ref{alg:RODI-MB} are all greater than or equal to $W^*_1(s^k_1)$, i.e., 
\[ W^k_1(s^k_1)=E_{\beta}(\nu^k_1(s^k_1)) \ge E_{\beta}(\nu^*_1(s^k_1))=W^*_1(s^k_1), \forall k \in [K]. \]
\end{lemma}
The proof uses induction, paralleling the methodology in \texttt{RODI-MF}, and leverages Lemma \ref{lem:opt_model} to ensure that the return distributions are optimistically biased.
\begin{proof}
The induction begins with the terminal stage, $H$, and progresses backwards. For the visited $(s,a)$, we have
\begin{align*}
	J^k_H(s,a)=E_{\beta}(\eta^k_H(s,a))&=\exp(\beta r_H(s,a)) =J^*_H(s,a).
\end{align*}
For the unvisited $(s,a)$, it holds that $J^k_H(s,a)=\exp(\beta) \ge J^*_H(s,a)$. Thus $W^k_H(s)=\max_a J^k_H(s,a)\ge\max_a J^*_H(s,a) =W^*_H(s)$ for any $s$. Assuming $W^k_{h+1}(s)\ge W^*_{h+1}(s),\forall s$ for $h\in[H-1]$. It follows that for the $(s,a)$ with $N^k_h(s,a)>0$
	\begin{align*}
	J^k_h(s,a)&= \exp(\beta r_h(s,a))E_{\beta}\prt{\brk{\tilde{P}^k_h\nu^k_{h+1}}(s,a)} \ge  \exp(\beta r_h(s,a))E_{\beta}\prt{\brk{P_h\nu^k_{h+1}}(s,a)}\\
	&\ge  \exp(\beta r_h(s,a))E_{\beta}\prt{\brk{P_h\nu^*_{h+1}}(s,a)}=J^{*}_h(s,a).
	\end{align*}
The first inequality is due to Lemma \ref{lem:opt_model}. The second inequality follows from the induction assumption. For the unvisited $(s,a)$, we have $J^k_h(s,a)=\exp(\beta (H+1-h))\ge J^*_h(s,a)$. Since  $W^k_h(s)=\max_a J^k_h(s,a) \ge \max_a J^*_h(s,a)=W^*_h(s)$ for any $s$, the induction is completed.
\end{proof}
\subsection{Equivalence to \texttt{ROVI}}
The \textbf{R}isk-sensitive  \textbf{O}ptimistic \textbf{V}alue \textbf{I}teration (\texttt{ROVI}) algorithm, as outlined in Algorithm \ref{alg:ROVI}, is a non-distributional approach that processes value functions directly, as opposed to handling return distributions. The \texttt{RODI-MB} algorithm, however, can be demonstrated to be equivalent to \texttt{ROVI}. This equivalence signifies that both algorithms generate the same policy sequence, implying that their resulting trajectories, denoted as $\cF_{K+1}$, follow the same distribution. This relationship is grounded in the connection between the EntRM and EU, coupled with the linearity property of EU. To formalize this concept of algorithmic equivalence, we define:
\begin{algorithm}[H]
	\caption{\texttt{ROVI}}
	\label{alg:ROVI}
	\begin{algorithmic}[1]
		\State{Input: $T$ and $\delta$}
		\State{$N^1_h(\cdot,\cdot)\leftarrow0$; $\hat{P}^1_h(\cdot,\cdot)\leftarrow \frac{1}{S}\textbf{1}$ $\forall h\in[H]$}
		\For{$k = 1:K$}
		\State{$W^k_{H+1}(\cdot)\leftarrow1$}
		\For{$h = H:1$}
		\If{$N^k_h(\cdot,\cdot)>0$}
		    \State{$\tilde{P}^k_h(\cdot,\cdot)\leftarrow \rO^1_{c^k_h(\cdot,\cdot)}\prt{\hat{P}^k_h(\cdot,\cdot),W^k_{h+1}}$}
		    \State{$J^k_h(\cdot,\cdot)\leftarrow e^{\beta r_h(\cdot,\cdot)}\brk{\tilde{P}^k_h W^k_{h+1}}(\cdot,\cdot)$}
		\Else{}
		    \State{$J^k_h(\cdot,\cdot)\leftarrow \exp(\beta(H+1-h))$}
		\EndIf{}
		\State{$W^k_h(\cdot)\leftarrow\max_a J^k_h(\cdot,a)$}
		\EndFor
		\State{Receive $s^k_1$}
		\For{$h = 1:H$}
		\State{$a^k_h\leftarrow \arg\max_a \text{sign}(\beta)J^k_h(s^k_h,a)$ and transit to $s^k_{h+1}$}
		\State{Compute $N^{k+1}_h(\cdot,\cdot)$ and $\hat{P}^{k+1}_h(\cdot,\cdot)$}
		\EndFor
		\EndFor
	\end{algorithmic}
\end{algorithm}
\begin{definition}
For two algorithms $\sA$ and $\tilde{\sA}$, we say that $\sA$ is equivalent to $\tilde{\sA}$ (vice versa) if for any $k\in[K]$, any $\cF_k$, it holds that $\sA(\cF_k)=\tilde{\sA}(\cF_k)$.
\end{definition}
Under this definition, if two algorithms are equivalent, the trajectories or histories generated by their interactions with any MDP instance will follow the same distribution throughout the episodes. Consequently, these algorithms will enjoy the same regret. 
\begin{proposition}
	\label{prop:equiv}
	Algorithm \ref{alg:RODI-MB} is equivalent to Algorithm \ref{alg:ROVI}.
\end{proposition}
\begin{proof}
We focus on the case where $\beta>0$, noting that the case for $\beta<0$ can be argued in a similar manner. Fix an arbitrary $k\in[K]$ and $\cF_{k}=\{s^1_1,a^1_1,\cdots,s^{k-1}_H,a^{k-1}_H\}$.  Let $\sA$ (and ${\pi^k_h}$) represent Algorithm \ref{alg:ROVI} (and its corresponding policy sequence), while $\tilde{\sA}$ (and ${\tilde{\pi}^k_h}$) denote Algorithm \ref{alg:RODI-MB} (and its respective policy sequence). To establish equivalence, we need to show that $\pi^k$ aligns with $\tilde{\pi}^k$ for the given history $\cF_{k}$. By the definition of the two algorithms
	\[	\tilde{\pi}^k_h(s)=\arg\max_{a}Q^k_h(s,a)=U_{\beta}(\eta^k_h(s,a)),\ \pi^k_h(s)=\arg\max_{a}J^k_h(s,a).		\]
	If $J^k_h(s,a)=E_{\beta}(\eta^k_h(s,a))=\exp(\beta Q^k_h(s,a))$ for any $(s,a)$, then $\pi^k_h=\tilde{\pi}^k_h$ due to the monotonicity of the exponential function. We will prove that $J^k_h(s,a)=E_{\beta}(\eta^k_h(s,a))$ for any $(s,a)$ by the induction. The base case is evident as $J^k_{H}(s,a)=E_{\beta}(\eta^k_H(s,a))$. Assuming $J^k_h(s,a)=E_{\beta}(\eta^k_h(s,a))$ for all $(s,a)$ for some $h\in[H]$, we have $\pi^k_h=\tilde{\pi}^k_h$ and 
	\begin{align*}
	W^k_h(s)&=\max_a J^k_h(s,a)= J^k_h(s,\pi^k_h(s))= E_{\beta}(\eta^k_h(s,\pi^k_h(s)))=E_{\beta}(\eta^k_h(s,\tilde{\pi}^k_h(s)))=E_{\beta}(\nu^k_h(s)).
	\end{align*}
	Given the same history $\cF_k$, both algorithms share the empirical transition model $\hat{P}^k_{h-1}$, the count $N^k_{h-1}$, and the optimism constants $c^k_{h-1}$. Therefore, they also share the optimistic transition model $\tilde{P}^k_{h-1}$. According to the update formula of Algorithm \ref{alg:ROVI}, for any $(s,a)$ with $N^k_h(s,a)>0$, we have
	\begin{align*}
	J^k_{h-1}(s,a)&=\exp(\beta r_h(s,a))\brk{\tilde{P}^k_{h-1} W^k_h}(s,a)=\exp(\beta r_h(s,a))E_{\beta}\prt{\brk{\tilde{P}^k_{h-1} \nu^k_h}(s,a)}\\
	&=E_{\beta}\prt{\brk{\mathcal{B}(\tilde{P}^k_{h-1},r_{h-1})\nu^{k}_{h}}(s,a)}= E_{\beta}\prt{\eta^k_{h-1}(s,a)}.
	\end{align*}
This equality also holds for the unvisited state-action pairs, thereby completing the proof of equivalence between Algorithm \ref{alg:RODI-MB} and Algorithm \ref{alg:ROVI}.
\end{proof}

\subsection{Regret Upper Bound of \texttt{RODI-MB}/\texttt{ROVI}}
\begin{theorem}[Regret upper bound of \texttt{RODI-MB}/\texttt{ROVI}]
	\label{thm:rovi}
	For any $\delta\in(0,1)$, with probability $1-\delta$, the regret of Algorithm \ref{alg:RODI-MF} or Algorithm \ref{alg:ROVI} is bounded as
	\begin{align*}
	\text{Regret}(\texttt{RODI-MF},K)=\text{Regret}(\texttt{ROVI},K) &\leq \mathcal{O}(L_{H}H\sqrt{S^2AK\log(4SAT/\delta)})\\
	&=\tilde{\mathcal{O}}\prt{\frac{\exp(|\beta| H)-1}{|\beta|}H\sqrt{S^2AK}}. 
	\end{align*}
\end{theorem}
\begin{proof}
The regret can be bounded as
\begin{align*}
\text{Regret}(K)\leq \frac{1}{\beta}\sum^K_{k=1} W^{k}_1(s^k_1)-W^{\pi^k}_1(s^k_1)=\frac{1}{\beta}\sum^K_{k=1}\delta^k_1.
\end{align*}
We can decompose $\delta^k_h$ as follows
\begin{align*}
\delta^k_h &= E_{\beta}(\nu^k_h(s^k_h))-E_{\beta}(\nu^{\pi^k}_h(s^k_h))  \\
&=\exp(\beta r^k_h) E_{\beta}\prt{\brk{\tilde{P}^k_h\nu^k_{h+1}}(s^k_h)}-\exp(\beta r^k_h)E_{\beta}\prt{\brk{P^{\pi^k}_h\nu^{\pi^k}_{h+1}}(s^k_h)}\\
&=\underbrace{ \exp(\beta r^k_h) E_{\beta}\prt{\brk{\tilde{P}^k_h\nu^k_{h+1}}(s^k_h)}-\exp(\beta r^k_h)E_{\beta}\prt{\brk{P^{\pi^k}_h\nu^{k}_{h+1}}(s^k_h) }}_{(a)}\\
&+\underbrace{ \exp(\beta r^k_h)E_{\beta}\prt{\brk{P^{\pi^k}_h\nu^{k}_{h+1}}(s^k_h) }-\exp(\beta r^k_h)E_{\beta}\prt{\brk{P^{\pi^k}_h\nu^{\pi^k}_{h+1}}(s^k_h) } }_{(b)}\\
&=\underbrace{ \exp(\beta r^k_h)\brk{\tilde{P}^k_h W^k_{h+1}}(s^k_h)-\exp(\beta r^k_h)\brk{P^{\pi^k}_h W^k_{h+1}}(s^k_h) }_{(a)}\\ 
&+\underbrace{ \exp(\beta r^k_h)\brk{P^{\pi^k}_h W^k_{h+1}}(s^k_h) -\exp(\beta r^k_h)\brk{P^{\pi^k}_h W^{\pi^k}_{h+1}}(s^k_h) }_{(b)}.
\end{align*}
Using the Lipschitz property of EU
\begin{align*}
(a) &\leq  L_{H+1-h}\norm{\brk{\tilde{P}^k_h\nu^k_{h+1}}(s^k_h)(\cdot-r_h^k)-\brk{P^{\pi^k}_h\nu^k_{h+1}}(s^k_h)(\cdot-r_h^k)  }_{\infty}\\
&= L_{H+1-h}\norm{\brk{\tilde{P}^k_h\nu^k_{h+1}}(s^k_h)-\brk{P^{\pi^k}_h\nu^k_{h+1}}(s^k_h) }_{\infty} \\
&\leq L_{H+1-h}\norm{\tilde{P}^k_h-P^{\pi^k}_h}_1 \leq L_{H+1-h} c^k_{h}\\
&= (\exp(\beta(H+1-h))-1)c^k_{h},
\end{align*}
where the second inequality is due to Lemma \ref{lem:mix_dis}.  Term $(b)$ is bounded as
\begin{align*}
(b) &= \exp(\beta r^k_h)\brk{P^{\pi^k}_h(W^k_{h+1}-W^{\pi^k}_{h+1})}(s^k_h)
= \exp(\beta r^k_h)\brk{P^{\pi^k}_h\Delta^k_{h+1}}(s^k_h)\\
&= \exp(\beta r^k_h)(\epsilon^k_h+\delta^k_{h+1}),
\end{align*}
where $\epsilon^k_h\triangleq[P^{\pi^k}_h \Delta^k_{h+1}](s^k_h) -\Delta^k_{h+1}(s^k_{h+1})$ is a  martingale difference sequence with $\epsilon^k_h \in 2D_{h+1}$ a.s. for all $(k,h)\in[K]\times[H]$. In summary, we can bound $\delta^k_h$ recursively as 
\begin{align*}
\delta^k_h &\leq L_{H+1-h}c^k_h + \exp(\beta r^k_h)(\epsilon^k_h+\delta^k_{h+1}).
\end{align*}
Repeating the procedure, we can get
\begin{align*}
\delta^k_1 &\leq \sum^{H-1}_{h=1}L_{H+1-h} \prod^{h-1}_{i=1}\exp(\beta r^k_i) c^k_{h} + \sum^{H-1}_{h=1} \prod^{h}_{i=1}\exp(\beta r^k_i) \epsilon^k_{h}+\prod^{H-1}_{i=1}\exp(\beta r^k_i) \delta^k_{H} \\
&\leq  \sum^{H-1}_{h=1}(\exp(\beta(H+1-h))-1) \exp(\beta (h-1)) c^k_{h} + \sum^{H-1}_{h=1} \prod^{h}_{i=1}\exp(\beta r^k_i) \epsilon^k_{h}+ \exp(\beta (H-1)) \delta^k_{H}\\
&\leq \sum^{H-1}_{h=1}(\exp(\beta H)-1)c^k_{h} + \sum^{H-1}_{h=1} \prod^{h}_{i=1}\exp(\beta r^k_i) \epsilon^k_{h}+ \exp(\beta (H-1)) \delta^k_{H}.
\end{align*}
It follows that
\[  \sum_{k=1}^K\delta^k_1\leq (\exp(\beta H)-1)\sum^K_{k=1}\sum^{H-1}_{h=1}c^k_{h} + \sum^K_{k=1}\sum^{H-1}_{h=1} \prod^{h}_{i=1}\exp(\beta r^k_i) \epsilon^k_{h}+\sum^K_{k=1} \exp(\beta (H-1)) \delta^k_{H}.\]

The following follows analogously: with probability at least $1-\delta$,
\begin{align*}
\text{Regret}(K)&\leq \frac{\exp(\beta (H+1))-1}{\beta}\prt{  2(H-1)\sqrt{2S^2AK\iota} +  \sqrt{2KH\iota} +  SA}\\
&=\tilde{\mathcal{O}}\prt{\frac{\exp(\beta H)-1}{\beta H}H\sqrt{HS^2AT}},
\end{align*}
where $\iota\triangleq\log(2SAT/\delta)$.
\end{proof}
\vspace{-3ex}
\begin{remark}
Compared to the traditional/non-distributional analysis dealing with scalars, our analysis is distribution-centered, and we call it the \emph{distributional analysis}. The distributional analysis deals with the distributions of the return rather than the risk measure values of the return. In particular, it involves the operations of the distributions, the optimism between different distributions, the error caused by estimation of distribution, etc. These distributional aspects fundamentally differ from the traditional analysis that deals with the  scalars (value functions). 
\end{remark}

\section{Regret Lower Bound}
\label{sec:lb}
The section establishes a regret lower bound for EntRM-MDP, serving to understand the fundamental limitations of any learning algorithm in such settings. While previous works, like \cite{fei2020risk}, have approached this problem by drawing parallels to simpler models like the two-armed bandit, leading to lower bounds that are independent of $S, A$, and $H$, this approach does not capture the full complexity of MDPs.

In contrast, the approach motivated by \cite{domingues2021episodic} aims to derive a more comprehensive and tight minimax lower bound that incorporates these factors. For risk-neutral MDPs, the tight minimax lower bound has been established as $H\sqrt{SAT}$, but extending this to the risk-sensitive domain is challenging due to the non-linearity of EntRM. In risk-neutral scenarios, the linearity of expectation allows for the interchange of the risk measure (expectation) and summation, simplifying the analysis. However, in the risk-sensitive setting, the non-linear nature of EntRM precludes such straightforward manipulations, necessitating novel proof techniques.
\begin{assumption}
\label{asp:lb_exp}
Assume $S\ge6, A\ge 2$, and there exists an integer $d$ such that $S=3+\frac{A^d-1}{A-1}$. We further assume that $H\ge3d$ and $\bar{H}\triangleq\frac{H}{3}\ge1$.
\end{assumption}
\begin{theorem}[Tighter lower bound]
\label{thm:lb_exp}
Assume Assumption \ref{asp:lb_exp} holds and $\beta>0$. Let $\bar{L}\triangleq(1-\frac{1}{A})(S-3)+\frac{1}{A}$. Then for any algorithm $\sA$, there exists an MDP $\mathcal{M}_{\sA}$ such that for $K\ge2\exp(\beta(H-\bar{H}-d))\bar{H}\bar{L}A$ we have 
\[ \mathbb{E}[\text{Regret}(\sA,\mathcal{M}_{\sA},K)]\ge\frac{1}{72\sqrt{6}}\frac{\exp(\beta H/6)-1}{\beta H}H\sqrt{SAT}.\]
\end{theorem}
\begin{remark}
As $\beta\rightarrow0$, it recovers the \emph{tight} lower bound for risk-neutral episodic MDP $\Omega(H\sqrt{SAT})$ \citep[see][]{domingues2021episodic}.
\end{remark}
\vspace{-4ex}
\begin{remark}
The two conditions in Assumption \ref{asp:lb_exp} are used in the our paper and \cite{domingues2021episodic} to simply the proof. Technically, we can relax these conditions to any MDP with $S \geq 11, A \geq 4$ and $H \geq 6$, which is modestly large. In particular, condition (i) allows us to consider a full $A$-ary tree with $S-3$ nodes, which implies that all the leaves are at the same level $d-1$ in the tree. The proof can be generalized to any $S \geq 6$ by arranging the states in a balanced, but not necessarily full, $A$-ary tree. We can also technically relax condition (ii) to the case $H<3d$. In this case, the resulting bounds will replace $S$ by $\left[A^{\frac{H}{3}-2}\right]$.
\end{remark}
Before presenting the proof of Theorem \ref{thm:lb_exp}, we first fix the lower bound in \cite{fei2020risk}. 
\subsection{Fixing Lower Bound}
\label{sec:cor_lb}
\cite{fei2020risk} presents the following lower bound. 
\begin{proposition}[Theorem 3, \cite{fei2020risk}]
For sufficiently large $K$ and $H$, the regret of any algorithm obeys 
$$
\mathbb{E}[\operatorname{Regret}(K)] \gtrsim \frac{e^{|\beta| H / 2}-1}{|\beta|} \sqrt{T \log T}.
	$$
\end{proposition}
However, a critical reassessment of the proof reveals inaccuracies that necessitate a revision of the lower bound. The main issue lies in the derivation of the second inequality in the proof provided by \cite{fei2020risk}, specifically:
\begin{align*} 
\mathbb{E}[\operatorname{Regret}(K)] &\gtrsim \frac{\exp(\beta H/2)-1}{\beta}\sqrt{K\log(K)}\gtrsim \frac{\exp(\beta H/2)-1}{\beta}\sqrt{KH\log(KH)}.
\end{align*}
The authors establish the second inequality based on the following fact
\begin{fact}[Fact 5, \cite{fei2020risk} ]
	\label{fct:fei}
For any $\alpha>0$, the function $f_{\alpha}\triangleq\frac{e^{\alpha x}-1}{x}, x>0$ is increasing and satisfies $\lim_{x\rightarrow0}f_{\alpha}=\alpha$.
\end{fact}
In fact, we can only use Fact \ref{fct:fei} to derive $\frac{\exp(\beta H/2)-1}{\beta}\gtrsim H$, which combined with the first inequality yields \begin{align*}
\mathbb{E}[\operatorname{Regret}(K)] \gtrsim H\sqrt{KH\log(KH)},
\end{align*}
which, notably, does not capture the dependency on $\beta$ and $H$ as the original lower bound suggested. Consequently, the corrected version of the lower bound is more conservative and does not reflect the exponential influence of the risk factor and the horizon on the regret. This corrected proposition reads:
\begin{proposition}[Correction of Theorem 3, \cite{fei2020risk}]
	For sufficiently large $K$ and $H$, the regret of any algorithm obeys
	$$
	\mathbb{E}[\operatorname{Regret}(K)] \gtrsim \frac{e^{|\beta| H / 2}-1}{|\beta|} \sqrt{K \log K}.
	$$
\end{proposition}

\subsection{Proof of Theorem \ref{thm:lb_exp}}
\label{pf_lb}
We define $\mathrm{kl}(p,q)\triangleq p\log\frac{p}{q}+(1-p)\log\frac{1-p}{1-q}$ as the KL divergence between two Bernoulli distributions with parameters $p$ and $q$. We define the probability measure induced by an algorithm $\sA$ and an MDP instance $\mathcal{M}$ as 
\[   \bP_{\sA\cM}(\cF^{K+1})\triangleq\prod_{k=1}^K \bP_{\sA_k(\cF^k)\cM}(\cI^k_H|s^k_1),\]
where $\bP_{\pi\cM}$ is the  probability measure induced by a policy $\pi$ and $\mathcal{M}$, which is defined as
\[   \bP_{\pi\cM}(\cI_H|s_1)\triangleq\prod_{h=1}^H \pi_h(a_h|s_h)P^{\cM}_h(s_{h+1}|s_h,a_h). \]
The probability measure for the truncated history $\cH^k_h$ can be obtained by marginalization
\[   \bP_{\sA\cM}(\cH^{k}_h)=\bP_{\sA\cM}(\cF^{k})\bP_{\sA_k(\cF^k)\cM}(\cI^k_h) .\]  
We denote by $\mathbb{P}_{\sA\mathcal{M}}$ and $\mathbb{E}_{\sA\mathcal{M}}$ the probability measure and expectation induced by $\sA$ and $\mathcal{M}$. We omit $\sA$  and $\cM$ if it is clear in the context. 
\begin{proof}
Fix an arbitrary algorithm $\sA$. We introduce three types of special states for the hard MDP class: a waiting state $s_w$ where the agent starts and may stay until stage $\bar{H}$, after that it has to leave; a good state $s_g$ which is absorbing and is the only rewarding state; a bad state $s_b$ that is absorbing and provides no reward. The rest $S-3$ states are part of a $A$-ary tree of depth $d-1$. The agent can only arrive $s_w$ from the root node $s_{root}$ and can only reach $s_g$ and $s_b$ from the leaves of the tree.
	
Let $\bar{H}\in[H-d]$ be the first parameter of the MDP class. We define $\tilde{H}\triangleq\bar{H}+d+1$ and $H^{\prime}\triangleq H+1-\tilde{H}$. We denote by  $\mathcal{L}\triangleq\{s_1,s_2,...,s_{\bar{L}}\}$  the set of $\bar{L}$ leaves of the tree. For each $u^*\triangleq\left(h^{*}, \ell^{*}, a^{*}\right) \in [d+1:\bar{H}+d]\times\mathcal{L}\times\mathcal{A}$, we define an MDP $\mathcal{M}_{u^*}$ as follows. The transitions in the tree are deterministic, hence taking action $a$ in state $s$ results in the $a$-th child of node $s$. The transitions from $s_w$ are defined as
\[
	P_{h}\left(s_{\mathrm{w}} \mid s_{\mathrm{w}}, a\right) \triangleq \mathbb{I}\left\{a=a_{\mathrm{w}}, h \leq \bar{H}\right\} \quad \text { and } \quad P_{h}\left(s_{\text {root }} \mid s_{\mathrm{w}}, a\right) \triangleq 1-P_{h}\left(s_{\mathrm{w}} \mid s_{\mathrm{w}}, a\right).
	\]
The transitions from any leaf $s_i\in\mathcal{L}$  are specified  as
	\[ P_{h}\left(s_{g} \mid s_{i}, a\right) \triangleq p+\Delta_{u^*}\left(h, s_{i}, a\right) \quad \text { and } \quad P_{h}\left(s_{b} \mid s_{i}, a\right) \triangleq 1- p-\Delta_{u^*}\left(h, s_{i}, a\right), \]
	where $\Delta_{u^*}\left(h, s_{i}, a\right)\triangleq\epsilon\mathbb{I}\{(h, s_{i}, a)=(h^{*}, s_{\ell^{*}}, a^{*})\}$ for some constants  $p\in[0,1]$ and $\epsilon\in[0,\min(1-p,p)]$ to be determined later. $p$ and $\epsilon$ are the second and third parameters of the MDP class. Observe that  $s_g$ and $s_b$ are absorbing, therefore we have $\forall a, P_{h}\left(s_{g} \mid s_{g}, a\right)\triangleq P_{h}\left(s_{b} \mid s_{b}, a\right)\triangleq1$. The reward  is a deterministic function of the state 
	\[   r_h(s,a)\triangleq \mathbb{I}\{s=s_g, h\ge\tilde{H}\}.\]
	Finally we define a reference MDP $\mathcal{M}_0$ which differs from the previous MDP instances only in that $\Delta_0(h,s_i,a)\triangleq 0$ for all $(h,s_i,a)$. For each $\epsilon,p$ and $\bar{H}$, we define the MDP class 
	\[  \mathcal{C}_{\bar{H},p,\epsilon}\triangleq\mathcal{M}_0\cup\{\mathcal{M}_{u^*}\}_{u^*\in[d+1:\bar{H}+d]\times\mathcal{L}\times\mathcal{A}}.\]
	The total expected ERM value of $\sA$ is given by 
	\begin{align*}
	&\ \ \ \ \mathbb{E}_{\sA,\mathcal{M}_{u^*}}\brk{\sum_{k=1}^K U_{\beta}\prt{\sum^H_{h=1}r_h(s^k_h,a^k_h)|\pi^k}}\\
	&=\mathbb{E}_{\sA,\mathcal{M}_{u^*}}\brk{ \sum_{k=1}^K \frac{1}{\beta}\log\mathbb{E}_{\sA,\mathcal{M}_{u^*}}\brk{\exp\prt{\beta \sum^H_{h=1}r_h(s^k_h,a^k_h)}}}\\
	&=\mathbb{E}_{\sA,\mathcal{M}_{u^*}}\brk{\sum_{k=1}^K \frac{1}{\beta}\log\mathbb{E}_{\pi^k,\mathcal{M}_{u^*}}\brk{\exp\prt{\beta \sum^H_{h=\Tilde{H}}\mathbb{I}\{s^k_h=s_g\}}}}\\
	&=\mathbb{E}_{\sA,\mathcal{M}_{u^*}}\brk{\sum_{k=1}^K \frac{1}{\beta}\log\mathbb{E}_{\pi^k,\mathcal{M}_{u^*}}\brk{\exp(\beta H^{\prime}\mathbb{I}\{s^k_{\tilde{H}}=s_g\})}}\\
	&=\mathbb{E}_{\sA,\mathcal{M}_{u^*}}\brk{\sum_{k=1}^K \frac{1}{\beta}\log(\exp(\beta H^{\prime})\mathbb{P}_{\pi^k,\mathcal{M}_{u^*}}(s^k_{\tilde{H}}=s_g)+\mathbb{P}_{\pi^k,\mathcal{M}_{u^*}}(s^k_{\tilde{H}}=s_b))},
	\end{align*} 
	where the second equality follows from the fact that the reward is non-zero only after step $\tilde{H}$, the third equality is due to that the agent gets into absorbing state when $h\ge\tilde{H}$. Define $x^k_h\triangleq(s^k_h,a^k_h)$ for each $(k,h)$ and $x^*\triangleq(s_{\ell^{*}},a^*)$, then it is not hard to obtain that
	\begin{align*}
	\mathbb{P}_{\pi^k,u^*}\left[s_{\tilde{H}}^{k}=s_{g}\right]&=\sum_{h=1+d}^{\bar{H}+d} p \mathbb{P}_{\pi^k,u^*}\left(s_{h}^{k} \in \mathcal{L}\right)+\mathbb{I}\left\{h=h^{*}\right\} \mathbb{P}_{\pi^k,u^*}(x^k_h=x^*) \varepsilon \\
	&=p+\epsilon\mathbb{P}_{\pi^k,u^*}(x^k_{h^*}=x^*).
	\end{align*}
	For an  MDP $\mathcal{M}_{u^*}$, the optimal policy $\pi^{*,\mathcal{M}_{u^*}}$  starts to traverse the tree at step $h^*-d$ then chooses to reach the leaf $s_{l^*}$ and performs action $a^*$. The corresponding optimal value in any of the MDPs is $V^{*,\mathcal{M}_{u^*}}=\frac{1}{\beta}\log(\exp(\beta H^{\prime})(p+\epsilon)+1-p-\epsilon)$. Define $p^k_{u^*}\triangleq\mathbb{P}_{\pi^k,u^*}(x^k_{h^*}=x^*)$, then the expected regret of $\sA$ in $\mathcal{M}_{u^*}$ can be bounded below as 
	\begin{align*}
	&\ \ \ \ \bE_{\sA,\cM_{u^*}}[\text{Regret}(\sA,\mathcal{M}_{u^*},K)]\\
	&=\bE_{\sA,\cM_{u^*}}\brk{ \sum^K_{k=1} V^{*,\mathcal{M}_{u^*}}-U_{\beta}\prt{\sum^H_{h=1}r_h(x^k_h)|\pi^k} } \\
	&=\bE_{\sA,\cM_{u^*}}\brk{ \sum_{k=1}^K \frac{1}{\beta}\log \frac{\exp(\beta H^{\prime})(p+\epsilon)+1-p-\epsilon}{\exp(\beta H^{\prime})(p+\epsilon p^k_{u^*})+1-p-\epsilon p^k_{u^*}} } \\
	&=\bE_{\sA,\cM_{u^*}} \brk{ \sum_{k=1}^K \frac{1}{\beta}\log\prt{1+\frac{\epsilon(1-p^k_{u^*})(\exp(\beta H^{\prime})-1)}{\exp(\beta H^{\prime})(p+\epsilon p^k_{u^*})+1-p-\epsilon p^k_{u^*}}} } \\
	&\ge \bE_{\sA,\cM_{u^*}} \brk{  \sum_{k=1}^K \frac{1}{\beta}\log\prt{1+\frac{\epsilon(1-p^k_{u^*})(\exp(\beta H^{\prime})-1)}{1+1}} }\\
	&\ge \bE_{\sA,\cM_{u^*}} \brk{  \frac{\exp(\beta H^{\prime})-1}{4\beta}\epsilon\sum_{k=1}^K (1-p^k_{u^*}) }\\
	&= \frac{\exp(\beta H^{\prime})-1}{4\beta}\epsilon\sum_{k=1}^K (1-\bE_{\sA,\cM_{u^*}}[p^k_{u^*}]) \\
	&= \frac{\exp(\beta H^{\prime})-1}{4\beta}K\epsilon\prt{1-\frac{1}{K}\mathbb{E}_{\sA,\cM_{u^*}}[N_K(u^*)] }.
	\end{align*}
	The first inequality holds by setting $p+\epsilon\leq\exp(-\beta H^{\prime})$. The second inequality holds by letting $\epsilon\leq2\exp(-\beta H^{\prime})$ since $\log(1+x)\ge\frac{x}{2}$ for $x\in[0,1]$. The last equality follows from the fact that 
	\[  \bE_{\sA,\cM_{u^*}}[p^k_{u^*}]=\bE_{\sA,\cM_{u^*}}[\mathbb{P}_{\pi^k,u^*}(x^k_{h^*}=x^*)]=\mathbb{P}_{\sA,u^*}(x^k_{h^*}=x^*)=\mathbb{E}_{\sA,u^*}[\mathbb{I}\{(x^k_{h^*}=x^*)\}] \]
	and the definition of $N_K(u^*)\triangleq\sum_{k=1}^K\mathbb{I}\{x^k_{h^*}=x^*\}$. 
	
	The maximum of the regret can be bounded below by the mean over all instances as
	\begin{align*}
	&\max_{u^*\in[d+1:\bar{H}+d]\times\mathcal{L}\times\mathcal{A}}\text{Regret}(\sA,\mathcal{M}_{u^*},K)\ge\frac{1}{\bar{H}\bar{L}A}\sum_{u^*\in[d+1:\bar{H}+d]\times\mathcal{L}\times\mathcal{A}}\text{Regret}(\sA,\mathcal{M}_{u^*},K)\\
	&\ge \frac{\exp(\beta H^{\prime})-1}{4\beta}K\epsilon\prt{1-\frac{1}{\bar{L}AK\bar{H}}\sum_{u^*\in[d+1:\bar{H}+d]\times\mathcal{L}\times\mathcal{A}}\mathbb{E}_{u^*}[N_K(u^*)]}.
	\end{align*}
	Observe that it can be further bounded if we can obtain an upper bound on \\ $\sum_{u^*\in[d+1:\bar{H}+d]\times\mathcal{L}\times\mathcal{A}}\mathbb{E}_{u^*}[N_K(u^*)]$, which can be done by relating each expectation to the expectation  under the reference MDP $\mathcal{M}_0$. \\
	By applying Fact \ref{fct:fund_inq} with $Z=\frac{N_K(u^*)}{K}\in[0,1]$, we have
	\[ \operatorname{kl}\prt{\frac{1}{K} \mathbb{E}_{0}\left[N_K(u^*)\right], \frac{1}{K} \mathbb{E}_{u^*}[N_K(u^*)]} \leq \operatorname{KL}\left(\mathbb{P}_{0}, \mathbb{P}_{ u^*}\right). \]
	By Pinsker's inequality, it implies that
	\[ \frac{1}{K} \mathbb{E}_{u^*}[N_K(u^*)]\leq \frac{1}{K} \mathbb{E}_{0}\left[N_K(u^*)\right] + \sqrt{\frac{1}{2}\operatorname{KL}\left(\mathbb{P}_{0}, \mathbb{P}_{ u^* }\right)}. \]
	Since $\mathcal{M}_0$ and $\mathcal{M}_{u^*}$ only differs at stage $h^*$ when $(s,a)=x^*$, it follows from Fact \ref{fct:div_dec} that
	\[ \operatorname{KL}\left(\mathbb{P}_{0}, \mathbb{P}_{u^*}\right)=\mathbb{E}_{0}\left[N_{K}(u^*)\right] \mathrm{kl}(p,p+\varepsilon). \]
	By Lemma \ref{lem:kl_bd}, we have $\operatorname{kl}(p,p+\epsilon)\leq\frac{\epsilon^2}{p}$ for $\epsilon\ge0$ and $p+\epsilon\in[0,\frac{1}{2}]$. Consequently,
	\begin{align*}
	&\ \ \ \ \frac{1}{K} \sum_{u^*\in[d+1:\bar{H}+d]\times\mathcal{L}\times\mathcal{A}}\mathbb{E}_{u^*}[N_K(u^*)]\\
	&\leq \frac{1}{K} \mathbb{E}_{0}\left[\sum_{u^*\in[d+1:\bar{H}+d]\times\mathcal{L}\times\mathcal{A}}N_K(u^*)\right] + \frac{\epsilon}{\sqrt{2p}}\sum_{u^*\in[d+1:\bar{H}+d]\times\mathcal{L}\times\mathcal{A}}\sqrt{\mathbb{E}_{0}\left[N_{K}(u^*)\right]}\\
	&\leq 1+\frac{\epsilon}{\sqrt{2p}}\sqrt{\bar{L}AK\bar{H}},
	\end{align*}
	where the second inequality is due to the Cauchy-Schwartz inequality and the fact that $\sum_{u^*\in[d+1:\bar{H}+d]\times\mathcal{L}\times\mathcal{A}}N_K(u^*)=K$.\\
	It follows that 
	\begin{align*}
	\max_{u^*\in[d+1:\bar{H}+d]\times\mathcal{L}\times\mathcal{A}}\text{Regret}(\sA,\mathcal{M}_{u^*},K)\ge \frac{\exp(\beta H^{\prime})-1}{4\beta}K\epsilon\prt{1-\frac{1}{\bar{L}A\bar{H}}-\frac{\frac{\epsilon}{\sqrt{2p}}\sqrt{\bar{L}AK\bar{H}}}{\bar{L}A\bar{H}}}.
	\end{align*}
	Choosing $\epsilon=\sqrt{\frac{p}{2}}(1-\frac{1}{LA\bar{H}})\sqrt{\frac{LA\bar{H}}{K}}$ maximizes the lower bound
	\begin{align*}
	\max_{u^*\in[d+1:\bar{H}+d]\times\mathcal{L}\times\mathcal{A}}\text{Regret}(\sA,\mathcal{M}_{u^*},K)\ge \frac{\sqrt{p}}{8\sqrt{2}}\frac{\exp(\beta H^{\prime})-1}{\beta}\prt{1-\frac{1}{\bar{L}A\bar{H}}}^2\sqrt{\bar{L}AK\bar{H}}.
	\end{align*}
	Since $S\ge6$ and $A\ge2$, we have $\bar{L}=(1-\frac{1}{A})(S-3)+\frac{1}{A}\ge\frac{S}{4}$ and $1-\frac{1}{\bar{L}A\bar{H}}\ge1-\frac{1}{\frac{6}{4}\cdot2}=\frac{2}{3}$. Choose $\bar{H}=\frac{H}{3}$ and use the assumption that $d\leq\frac{H}{3}$ to obtain that $H^{\prime}=H-d-\bar{H}\ge\frac{H}{3}$. Now we choose $p=\frac{1}{4}\exp(-\beta H^{\prime})$
	and $\epsilon=\sqrt{\frac{p}{2}}(1-\frac{1}{\bar{L}A\bar{H}})\sqrt{\frac{LA\bar{H}}{K}}\leq\frac{1}{2\sqrt{2}}\exp(-\beta H^{\prime}/2)\sqrt{\frac{\bar{L}A\bar{H}}{K}}\leq\frac{1}{4}\exp(-\beta H^{\prime})$ if $K\ge2\exp(\beta H^{\prime})\bar{L}A\bar{H}$. Such choice of $p$ and $\epsilon$ guarantees the assumption of Lemma \ref{lem:kl_bd} and  that $p+\epsilon\leq\exp(-\beta H^{\prime})$,
	$\epsilon\leq2\exp(-\beta H^{\prime})$. Finally we use the fact that $\sqrt{\bar{L}AK\bar{H}}\ge\frac{1}{2\sqrt{3}}\sqrt{SAKH}$ to obtain
	\[ \max_{u^*\in[d+1:\bar{H}+d]\times\mathcal{L}\times\mathcal{A}}\text{Regret}(\sA,\mathcal{M}_{u^*},K)\ge\frac{1}{72\sqrt{6}}\frac{\exp(\beta H/6)-1}{\beta}\sqrt{SAKH}.\]
\end{proof}
Theorem \ref{thm:lb_exp} recovers the tight lower bound for standard episodic MDP, implying that the exponential dependence on $|\beta|$ and $H$ in the upper bounds is indispensable. Yet, it is not clear whether a similar lower bound holds for $\beta<0$, which is left as a future direction.

\section{Discussion}
\label{sec:dis}
In this section, we provide a comprehensive comparison of DRL algorithms (\texttt{RODI-MB}, \texttt{RODI-MF}), DRL with distribution representation (\texttt{RODI-OTP}, \texttt{RODI-PTO}), \texttt{RSVI2} \citep{fei2021exponential}, \texttt{RSVI} \citep{fei2020risk}, and \texttt{UCBVI} \citep{azar2017minimax} in terms of regret guarantees and computational complexity. The comparison is also succinctly encapsulated in Table \ref{tab:comp}.

\subsection{Numerical Results}
To validate the empirical performance of our algorithms, we conducted numerical experiments comparing \texttt{RODI-MB}, \texttt{RODI-MF}, and \texttt{RODI-Rep} with the risk-neutral algorithm \texttt{UCBVI} \citep{azar2017minimax}, \texttt{RSVI} in \cite{fei2020risk}, and \texttt{RSVI2} in \cite{fei2021exponential}. 

The experimental setup involved an MDP with $S=5$ states, $A=5$ actions, and a horizon $H=5$, mirroring the setup in \cite{du2022provably}.  The MDP consists of a fixed initial state denoted as state 0, and $S$ additional states.  The agent started in state 0 and could take actions from the set $[A]$, transitioning to one of the states in $[S]$ in the next step. The transition probabilities and reward functions were defined as follows for $2\leq h\leq H$:
\begin{align*}
    \forall a \in [A-1]: P_h(s'|s,a)&=\frac{0.5}{S-2} \ \forall s'\in [2:S-1],
    P_h(1|s,a)=0.5,  \\
    P_h(s'|s,A)&=\frac{0.001}{S-1}, \forall s'\in [S-1],
    P_h(S|s,A)=0.999 \\
    \forall a \in [A]: r_h(1,a)&=1, r_h(S,a)=0.4, r_h(s,a)=0 \ \forall s \in [2:S-1].
\end{align*}
This MDP was designed to be highly risky, with the risk-neutral optimal policy leading to a mean reward of 0.5 but with a chance of receiving no reward. A risk-sensitive policy might prefer the last action $A$, which offers slightly less mean reward but a more consistent return, indicating lower risk.

We set $\delta=0.005$ and $\beta=-1.1$. The results, as illustrated in Figure \ref{fig:exp}, demonstrates the regret ranking of these algorithms :
\begin{align*} 
    \underbrace{\texttt{RODI-MB}<\texttt{RODI-MF}}_{\texttt{RODI}}<\underbrace{\texttt{RODI-OTP}
    <\texttt{RODI-PTO}}_{\texttt{RODI-Rep}}\lesssim\texttt{RSVI2}<\texttt{RSVI}<\texttt{UCBVI}.
\end{align*}
Figure \ref{fig:exp} includes the following key observations: \\
(i) Advantage of distributional over non-Distributional algorithms: DRL algorithms (\texttt{RODI} and \texttt{RODI-Rep}) outperforms non-distributional algorithms, demonstrating the effectiveness of distributional optimism over bonus-based optimism. \\
(ii) Performance of \texttt{RODI} vs. \texttt{RODI-Rep}: While \texttt{RODI} shows better performance than \texttt{RODI-Rep}, the latter offers a balance between statistical and computational efficiency. \\
(iii) Comparison of \texttt{RODI-Rep} with \texttt{RSVI2}: \texttt{RODI-Rep} demonstrates advantages over \texttt{RSVI2} in terms of sample efficiency, while also maintaining computational efficiency.
\begin{figure}[ht]
\vskip -0.1in
\begin{center}
\centering
	\includegraphics[width=.5 \textwidth]{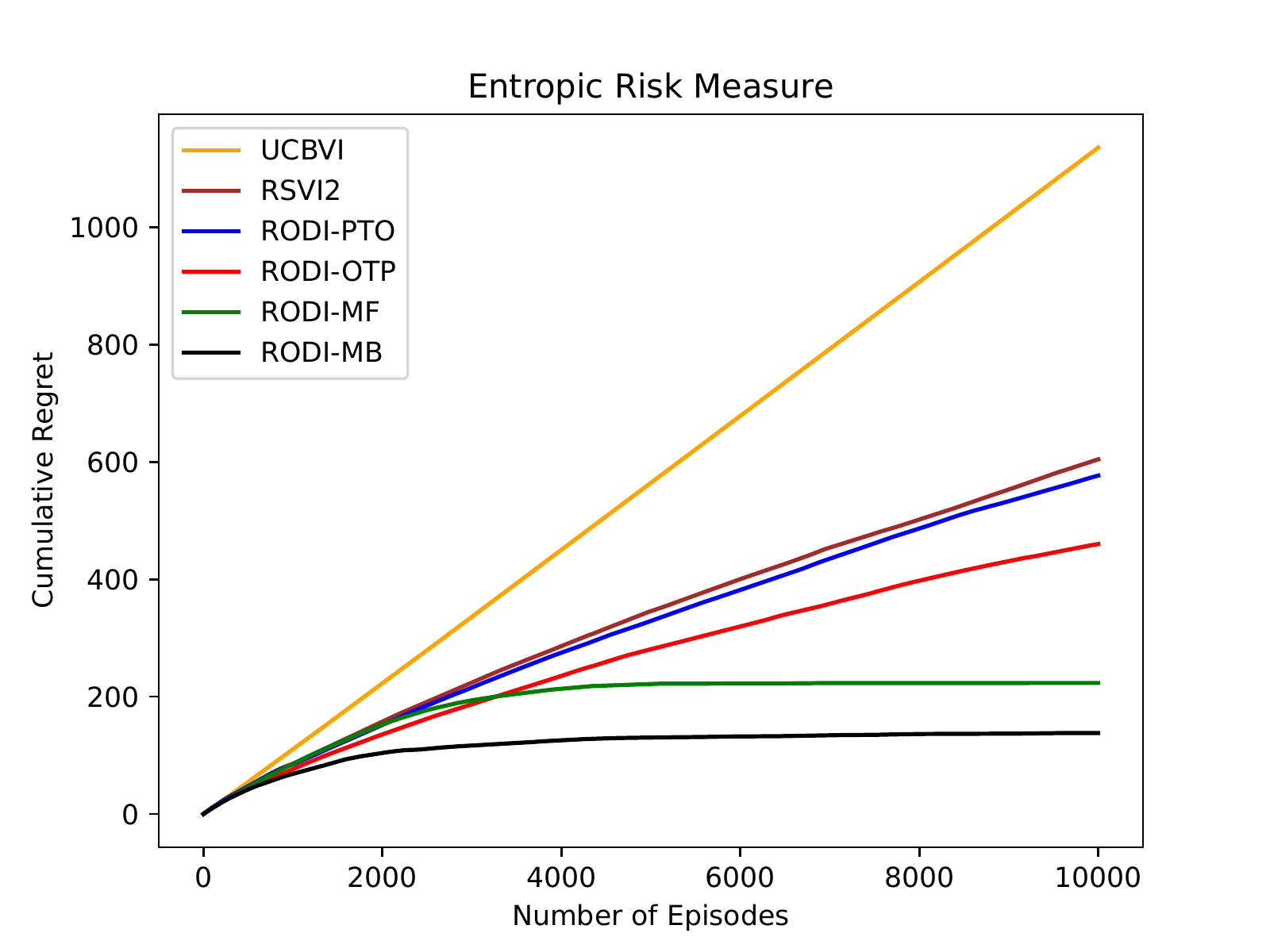}
	\caption{Comparison of regret for different algorithms.}
 \label{fig:exp}
 \end{center}
 \vskip -0.1in
\end{figure}
\subsection{Theoretical Comparisons}
\subsubsection{\texttt{RODI} vs. \texttt{RSVI2}}
We first provide theoretical justifications regarding the regret ranking of \texttt{RSVI} \citep{fei2020risk}, \texttt{RSVI2} \citep{fei2021exponential}, \texttt{RODI-MF}, and \texttt{RODI-MB}, which demonstrates the advantage of distributional optimism over bonus-based optimism used in \texttt{RSVI} and \texttt{RSVI2}. A key observation regarding the ranking of their value functions $V^k$ is that:
\[ \text{value functions}: \texttt{RSVI}>\texttt{RSVI2}>\texttt{RODI-MF}>\texttt{RODI-MB}\ge V^*.\]
This ordering will be formally presented in Equation \ref{eqt:v_comp}. The last part of this inequality sequence indicates that all these value functions are indeed optimistic. Given that the level of optimism is mirrored in the value functions, we can deduce:
\[ \text{optimism level}: \texttt{RSVI}>\texttt{RSVI2}>\texttt{RODI-MF}>\texttt{RODI-MB}.\]
Considering the relationship between regret and the optimistic value function $V^k$
\[ \text{Regret}=\sum_{k\in[K]} V^*_1-V^{\pi^k}_1 \leq \sum_{k\in[K]} V^{k}_1-V^{\pi^k}_1,\]
it is intuitive that a smaller $V^k$ or less optimism induces reduced regret. Consequently, their regret can be ranked as:
\[ \text{regret}:\texttt{RSVI}>\texttt{RSVI2}>\texttt{RODI-MF}>\texttt{RODI-MB},\]
which explains Figure \ref{fig:exp}. The regret bounds of \texttt{RODI} should at least match those of \texttt{RSVI2}, explaining the ranking of their regret bounds reported in Table \ref{tab:comp}:
\[ \text{regret bound}:\texttt{RSVI}>\texttt{RSVI2}=\texttt{RODI-MF}=\texttt{RODI-MB}.\]
Despite sharing same regret bounds with \texttt{RSVI2}, \texttt{RODI} outperforms \texttt{RSVI2} both theoretically and empirically. Formally speaking, let $V',V'',V$ denote the value functions generated by \texttt{RSVI}, \texttt{RSVI2}, and \texttt{RODI-MB} respectively. Let $\tilde{\eta}$ denote the distribution generated by \texttt{RODI-MF}. We omit $k$ for simplicity.
\begin{proposition}
\label{prop:v_com}
    Fix $(s,a,k,h)$. The comparison of their values is as follows:
\begin{equation}
\label{eqt:v_comp}
    \begin{aligned}
        \texttt{RSVI} \ \ \frac{1}{\beta}\log\prt{\brk{\hat{P}_h e^{\beta V'_{h+1}}}+b'_h}
    &\overset{(a)}{>}\frac{1}{\beta}\log\prt{\brk{\hat{P}_h e^{\beta V''_{h+1}}}+b'_h}\\
    &\overset{(b)}{>}\frac{1}{\beta}\log\prt{\brk{\hat{P}_h e^{\beta V''_{h+1}}}+b''_h} \ \ \texttt{RSVI2}\\
    &\overset{(c)}{>} \erm{\tilde{\eta}_h} \ \ \texttt{RODI-MF} \\
    &\overset{(d)}{>}\frac{1}{\beta}\log\prt{\brk{\tilde{P}_h e^{\beta V_{h+1}}}} \ \ \texttt{RODI-MB}\\
    &\overset{(e)}{>}\frac{1}{\beta}\log\prt{\brk{P_h e^{\beta V^*_{h+1}}}}.
    \end{aligned}
\end{equation}
\end{proposition}
The proof is detailed in  Appendix \ref{app:pf}. Both \texttt{RSVI} and \texttt{RSVI2} use exploration bonuses, defined as $b'_h=|e^{\beta H}-1|c_h$ and $b''_h=|e^{\beta (H+1-h)}-1|c_h$ respectively, where $c_h(s,a)$ represents the model estimation error
\[  \norm{\hat{P}_h(s,a)-P_h(s,a)}_1 \leq c_h(s,a)=\sqrt{\frac{S\iota}{N_h(s,a)}}.\]
Both $b''_h$ and $b'_h$ are formulated as a multiplier times $c_h$. Notably, $b''_h$, referred to as the the \emph{doubly decaying bonus} \citep{fei2021exponential}, decreases its multiplier exponentially across stages $h$, contrasting with $b'_h$ in \texttt{RSVI}.  In comparison, \texttt{RODI} directly incorporates optimism into the return distribution using an optimism constant $c_h$. Our \textit{distributional analysis} establishes a connection between $c_h$ and the bonus via the Lipschitz constant of EU:
\[  b''_h = L(E_{\beta},H-h)c_h < L(E_{\beta},H) c_h =b'_h,\]
where $L(E_{\beta},M)$ denotes the Lipschitz constant of EU over the distributions supported in $[0,M]$. This distributional perspective posits that \texttt{RSVI} and \texttt{RSVI2} design bonuses to offset the error in value estimates, which is bounded by the product of the Lipschitz constant of EU and the error in the return distribution:
\[  V^k_h-V_h \leq L(E_{\beta},H-h)\norm{\eta^k_h-\eta_h} \leq L(E_{\beta},H-h)\norm{P^k_h-P_h} \leq L(E_{\beta},H-h)c^k_h. \]
Under the distributional perspective, the multiplier in the bonus $b''_h$ is interpreted as the Lipschitz constant that links the return estimation error $c_h$ to the value estimation error $b''_h$. The Lipschitz constant decreases exponentially in $h$ as the range $[0,H-h]$ of the return distribution narrows. Furthermore, $b''_h$ used in \texttt{RSVI2} is not improvable in the sense that its corresponding Lipschitz constant is proven to be tight, as shown in Lemma \ref{lem:eu_lip}. 

In conclusion, bonus-based optimism requires an exponentially decaying multiplier or Lipschitz constant, whereas distributional optimism functions directly at the distributional level, obviating the need for a multiplier. Next, we theoretically justify the regret ranking of \texttt{RODI-OTP} and  \texttt{RODI-PTO}, which interpolates between \texttt{RODI} and \texttt{RSVI2}.
\subsubsection{\texttt{RODI-Rep} vs. \texttt{RSVI2}}
We delve into the analysis by first explaining why \texttt{RODI-PTO} achieves marginally lower regret compared to \texttt{RSVI2}, and subsequently, we justify the advantage of \texttt{RODI-OTP} over \texttt{RODI-PTO}.
\paragraph{Near-equivalence between \texttt{RSVI2} and \texttt{RODI-PTO}.} We can show the  near-equivalence between \texttt{RSVI2} and \texttt{RODI-PTO} using induction. Let $V$ and $V'$  denote the value functions generated by \texttt{RODI-PTO} and \texttt{RSVI2} respectively. We start with the base case that $h=H$. By the construction of \texttt{RODI-PTO}, we have 
\begin{align*}
    &q_H(s,a)=q(r_H(s,a);0,1) \Longrightarrow \\
    &Q_H(s,a)=\frac{1}{\beta}\log\prt{(1-q_H(s,a))e^0+q_H(s,a)e^{\beta}}=r_H(s,a)=Q'_H(s,a) \Longrightarrow\\
    &V_H(s)=\max_a Q_H(s,a) = \max_a Q'_H(s,a) = V'_H(s),
\end{align*}
verifying the equivalence at step $H$. Now fix $h\in[H-1]$. Suppose the following holds
\begin{align*}
    V_{h+1}(s) &= \frac{1}{\beta}\log\prt{1-q_{h+1}(s)+q_{h+1}(s)e^{\beta (H-h)}}\leq V'_{h+1}(s),\forall s\in\cS \Longrightarrow \\
    &1-q_{h+1}(s)+q_{h+1}(s)e^{\beta (H-h)} \leq e^{\beta V'_{h+1}(s)},\forall s\in\cS .
\end{align*}
Recall the recursion of $q_h(s,a)$ in \texttt{RODI-PTO}
\begin{align*}
    \hat{q}_{h}(s,a) &\leftarrow [\hat{P}_h q_{h+1}](s,a)  \\
    \bar{q}_{h}(s,a) &\leftarrow (1-\hat{q}_{h}(s,a))q^L_h(s,a)+\hat{q}_{h}(s,a)q^R_h(s,a)  \\
    q_{h}(s,a) &\leftarrow \min(\bar{q}_{h}(s,a)+c_h(s,a),1).  
\end{align*}
It follows that
\begin{align*}
    Q_h(s,a)&=\frac{1}{\beta}\log\prt{(1-q_{h}(s,a))e^0+q_{h}(s,a)e^{\beta (H+1-h)}}\\
    &=\frac{1}{\beta}\log\prt{1+q_{h}(s,a)(e^{\beta (H+1-h)}-1)}\\
    &\leq\frac{1}{\beta}\log\prt{1+\bar{q}_{h}(s,a)(e^{\beta (H+1-h)}-1)+c_h(s,a)(e^{\beta (H+1-h)}-1)},
\end{align*}
where the last inequality becomes equality if $\bar{q}_{h}(s,a)+c_h(s,a)\leq1$. By the definition of projection, we obtain
\begin{align*}
    1+\bar{q}_{h}(s,a)(e^{\beta (H+1-h)}-1) &= 1-\bar{q}_{h}(s,a)+\bar{q}_{h}(s,a)e^{\beta (H+1-h)} \\
    &=(1-\hat{q}_{h}(s,a))e^{\beta r_h(s,a)}+\hat{q}_{h}(s,a)e^{\beta(r_h(s,a)+H-h)} \\
    &=[\hat{P}_h (1-q_{h+1})](s,a)e^{\beta r_h(s,a)}+[\hat{P}_h q_{h+1}](s,a)e^{\beta(r_h(s,a)+H-h)}\\
    &=\sum_{s'}\hat{P}_h(s'|s,a)\prt{(1-q_{h+1}(s'))e^{\beta r_h(s,a)}+q_{h+1}(s')e^{\beta(r_h(s,a)+H-h)}}\\
    &=e^{\beta r_h(s,a)}\sum_{s'}\hat{P}_h(s'|s,a)\prt{(1-q_{h+1}(s'))+q_{h+1}(s')e^{\beta(H-h)}}\\
    &= e^{\beta r_h(s,a)}\sum_{s'}\hat{P}_h(s'|s,a)e^{\beta V_{h+1}(s')}\\
    &\leq e^{\beta r_h(s,a)}\sum_{s'}\hat{P}_h(s'|s,a)e^{\beta V'_{h+1}(s')},
\end{align*}
which implies 
\begin{align*}
    Q_h(s,a)\leq \frac{1}{\beta}\log\prt{e^{\beta r_h(s,a)}\sum_{s'}\hat{P}_h(s'|s,a)e^{\beta V'_{h+1}(s')}+c_h(s,a)(e^{\beta (H+1-h)}-1)} =Q'_h(s,a).
\end{align*}
Then we have $V_h(s)=\max_a Q_h(s,a) \leq \max_a Q'_h(s,a) = V'_h(s)$. The induction is completed.
Moreover, it holds that $V_h=V'_h$ for every $h\in[H]$ if $\bar{q}_{h}(s,a)+c_h(s,a)\leq1$ for every $(h,s,a)$. This condition is likely to be met for large values of $k$, considering that
\[  k\uparrow \Longrightarrow N^k_h \downarrow \Longrightarrow c^k_h\propto 1/\sqrt{N^k_h} \downarrow.\]
\paragraph{Advantage of \texttt{RODI-OTP} over \texttt{RSVI2}} Let $V$ and $V'$  denote the value functions generated by \texttt{RODI-OTP} and \texttt{RSVI2} respectively. The recursion of $q_h(s,a)$ in \texttt{RODI-OTP} writes
\begin{align*}
    \hat{q}_{h}(s,a) &\leftarrow [\hat{P}_h q_{h+1}](s,a)  \\
    \tilde{q}_{h}(s,a) &\leftarrow \min(\hat{q}_{h}(s,a)+c_h(s,a),1)   \\
    q_{h}(s,a) &\leftarrow (1-\tilde{q}_{h}(s,a))q^L_h(s,a)+\tilde{q}_{h}(s,a)q^R_h(s,a).  
\end{align*}
Fix $(h,s,a)\in[H-1]\times\cS\times\cA$. Note that
\begin{align*}
    V_{h+1}(s) = \frac{1}{\beta}\log\prt{1-q_{h+1}(s)+q_{h+1}(s)e^{\beta (H-h)}},\forall s\in\cS, \\
    \Longrightarrow [\hat{P}_h e^{\beta V_{h+1}}](s,a) = (1-\hat{q}_{h}(s,a))+\hat{q}_{h}(s,a)e^{\beta (H-h)}, \forall (s,a), 
\end{align*}
then we have
\begin{align*}
    Q_h(s,a) &= \frac{1}{\beta}\log\prt{1-q_{h}(s,a)+q_{h}(s,a)e^{\beta(H+1-h)}}\\
    &= \frac{1}{\beta}\log\prt{(1-\tilde{q}_{h}(s,a))e^{\beta r_h(s,a)}+\tilde{q}_{h}(s,a)e^{\beta(r_h(s,a)+H-h)}}\\
    &\leq \frac{1}{\beta}\log\prt{(1-\hat{q}_{h}(s,a))e^{\beta r_h(s,a)}+\hat{q}_{h}(s,a)e^{\beta(r_h(s,a)+H-h)}+c_h(s,a)(e^{\beta(r_h(s,a)+H-h)}-e^{\beta r_h(s,a)})} \\
    &= \frac{1}{\beta}\log\prt{e^{\beta r_h(s,a)}[\hat{P}_h e^{\beta V_{h+1}}](s,a)+c_h(s,a)e^{\beta r_h(s,a)}(e^{\beta(H-h)}-1)} \\
    &<\frac{1}{\beta}\log\prt{e^{\beta r_h(s,a)}[\hat{P}_h e^{\beta V'_{h+1}}](s,a)+c_h(s,a)e^{\beta r_h(s,a)}(e^{\beta(H-h)}-1)} \\
    &<\frac{1}{\beta}\log\prt{e^{\beta r_h(s,a)}[\hat{P}_h e^{\beta V'_{h+1}}](s,a)+c_h(s,a)(e^{\beta(H+1-h)}-1)} =Q'_h(s,a).
\end{align*}
\begin{remark}
This explains why \texttt{RODI-OTP} achieves an order of magnitude improvement in regret compared with \texttt{RSVI2} as well as \texttt{RODI-PTO}, as the "optimism level ratio" of \texttt{RODI-OTP} to  \texttt{RSVI2} at step $h$ is quantifiable by
    \[  \frac{e^{\beta(r_h(s,a)+H-h)}-e^{\beta r_h(s,a)}}{e^{\beta(H+1-h)}-1}<1. \]
\end{remark}
\begin{remark}
The difference in the optimism level between the two algorithms stems from originates from their respective approaches to bounding the estimation error:
    \[ \sum_{s'}\hat{P_h}(s'|s,a)e^{\beta(r_h(s,a)+V'_{h+1}(s'))}.\]
Specifically, \texttt{RSVI2} treats $e^{\beta(r_h(s,a)+V'_{h+1})}$ as a variable within the range $[1,e^{\beta(H+1-h)}]$. However, since $e^{\beta r_h(s,a)}$ is deterministic and known,  the bonus can be refined by acknowledging
\[  \sum_{s'}\hat{P}_h(s'|s,a)e^{\beta(r_h(s,a)+V'_{h+1}(s'))}=e^{\beta(r_h(s,a)}[\hat{P}_he^{\beta V_{h+1}}](s,a), \]
where $e^{\beta V_{h+1}}\in[1,e^{\beta(H-h)}]$.
\end{remark}

\paragraph{Why \texttt{OTP} is better than \texttt{PTO}.} The superiority of \texttt{OTP}  over \texttt{PTO} can be substantiated through an insightful observation about the optimization problem:
\begin{equation}
\begin{aligned}
& \underset{q}{\text{min}} & & \erm{L,R;q} \\
& \text{s.t.}
& & \erm{L,R;q}\ge\erm{\eta} \\
&&& \norm{\eta-\hat{\eta}}_{\infty} \leq c \\
&&& \eta=D(\text{Supp}(\hat{\eta}))
\end{aligned}
\end{equation} 
Let $(L,R;\tilde{q})$ be the optimal solution to this problem. It turns out that the optimal solution is given by $(L,R;\tilde{q})=\Pi\rO_c\hat{\eta}$, aligning with the \texttt{OTP} principle.  Fixing $(h,s,a)$, we interpret $\hat{\eta}\triangleq [\hat{\cT}_h\nu_{h+1}](s,a)$ as the empirical Bellman operator applied to $\nu_{h+1}$. Suppose $\nu_{h+1}$ is optimistic relative to the true distribution $\nu^*_{h+1}$, i.e., $\erm{\nu_{h+1}}\ge\nu^*_{h+1}$. Define $\check{\eta}\triangleq [\cT_h\nu_{h+1}](s,a)$, which is the exact Bellman operator applied to $\nu_{h+1}$. Given that
\[ \norm{\hat{\eta}-\check{\eta}}_{\infty}=\norm{[(\hat{\cT}_h-\cT_h)\nu_{h+1}](s,a) }_{\infty}\leq c_h(s,a), \]
the optimal solution satisfies 
\[ \erm{L,R;\tilde{q}} \ge \erm{\check{\eta}}=\erm{[\hat{\cT}_h\nu_{h+1}](s,a)}\ge \erm{[\hat{\cT}_h\nu^*_{h+1}](s,a)}=\erm{\eta^*_h(s,a)}=Q^*_h(s,a). \]
Hence, the optimal solution $(L,R;\tilde{q})$ is optimistic over $\eta^*_h(s,a)$. The nature of the optimization problem compels $(L,R;\tilde{q})$ to be the Bernoulli distribution with support $(L_h,R_h)$ that necessitates minimal optimism over $\eta^*_h(s,a)$. Notably, the \texttt{PTO} solution $\rO_c\Pi\hat{\eta}$ is also a feasible solution.  Consequently, \texttt{OTP} induces less optimism than \texttt{PTO}:
\[  \erm{\Pi\rO_c\hat{\eta}} < \erm{\rO_c\Pi\hat{\eta}}.\]
This analysis elucidates the inherent advantage of the \texttt{OTP} approach over \texttt{PTO}. By inverting the order of the projection and optimism operators, \texttt{OTP} not only ensures an optimism over the true distribution but also guarantees that the induced optimism is minimal and necessary.

\subsection{Distributional Perspective}
The distributional perspective is crucial in both the algorithm design and the regret analysis of \texttt{RODI}, offering  advantages and novel approaches.
\subsubsection{Algorithm design}
\textit{Revisiting \texttt{RSVI2}}: \texttt{RSVI2} effectively operates as a model-based algorithm, implicitly maintaining an empirical model through visiting counts. We rewrite the key step in \texttt{RSVI2} as:
\[  Q''_h =\min\cbrk{H+1-h,r_h+\frac{1}{\beta}\log\prt{\brk{\hat{P}_h e^{\beta V''_{h+1}}}+b''_h}}. \]
Here, $b''_h$ is chosen to ensure optimism:
\[ \brk{\hat{P}_h e^{\beta V''_{h+1}}}+ b''_h \ge \brk{P_h e^{\beta V''_{h+1}}} \ge \brk{P_h e^{\beta V^*_{h+1}}} \Longleftarrow b''_h \ge \brk{(P_h-\hat{P}_h) e^{\beta V''_{h+1}}}\]
\textit{Distributional perspective}: In contrast, the distributional perspective leads to a fundamentally different algorithm design. The primary distinction of \texttt{RODI} is its implementation of return distribution iterations based on approximate distributional Bellman equation. When $\beta\rightarrow0$, \texttt{RODI} transitions to a risk-neutral algorithm, unlike \texttt{RSVI2}, where the log term becomes constant. \texttt{RODI} also introduces \textit{distributional optimism}, yielding optimistic return distributions without needing a multiplier, unlike bonus-based optimism. This approach not only contrasts sharply with bonus-based methods but also demonstrates improved theoretical and empirical performance. \\
\subsubsection{Regret analysis} 
Our regret analysis, which we term \emph{distributional analysis}, stands apart from traditional scalar-focused approaches. This analysis is centered around the distributions of returns rather than the risk values of these returns. It involves various distributional operations, including understanding the optimism between different distributions and the errors caused by distribution estimation. These elements fundamentally differ from classical analysis methods that focus on scalars (value functions). Let's highlight some novel aspects of our distributional analysis compared to traditional approaches \citep{fei2020risk,fei2021exponential}. \\
\textbf{(i) Distributional optimism}. Traditional analysis typically employs OFU to construct a series of optimistic value functions. In contrast, our distributional approach implements optimism directly at the distribution level, leading to a sequence of optimistic return distributions. This involves defining a high probability event under which the true return distribution is close to the estimated one within a certain confidence radius, followed by the application of a distributional optimism operator.  \\
\textbf{(ii) Lipschitz continuity and linearity in EU.} We leverage key properties of EU, such as Lipschitz continuity and linearity, that are crucial in establishing regret upper bounds. The Lipschitz continuity of EU relates the distance between distributions to their EU values' difference. In contrast, EntRM is non-linear w.r.t. the distribution, potentially introducing a factor of $\exp(|\beta| H)$ in error propagation across time steps, leading to a compounded factor of $\exp(|\beta| H^2)$ in the regret bound. \\
\textbf{(iii) Better interpretability}. Both \texttt{RODI} and \texttt{RSVI2} share a same regret bound of 
\[ \tilde{\mathcal{O}}\prt{\frac{\exp(|\beta| H)-1}{|\beta|}H\sqrt{S^2AK}}. \]
From the distributional perspective, the exponential term $\frac{\exp(|\beta| H)-1}{|\beta|}$ is interpreted as the Lipschitz constant of EntRM, highlighting the impact of EntRM’s nonlinearity on sample complexity. A larger Lipschitz constant implies a greater estimation error in values, thus leading to a more unfavorable regret bound.

\section{Closing Remarks}
\label{sec:rk}
In this paper, we present a novel framework for risk-sensitive distributional dynamic programming. We then introduce two types of computationally efficient DRL algorithms, which implement the OFU principle at the distributional level to strike a balance between exploration and exploitation under the risk-sensitive setting. We provide  theoretical justification and numerical results demonstrating that these algorithms outperforms existing methods while maintaining computational efficiency compared. Furthermore, we prove that DRL can attain near-optimal regret
upper bounds compared with our improved lower bound. 

Looking forward, there are several promising avenues for future research. Our current regret upper bound has an additional factor of $\sqrt{HS}$ compared to the lower bound, and it may be possible to eliminate this factor through further algorithmic improvements or refined analysis techniques. Additionally, extending the DRL algorithm from tabular MDP to function approximation settings would be an interesting and valuable direction for future investigation. Lastly, it would be worthwhile to explore whether our DDP framework can be applied to other risk measures beyond the ones considered in this paper.

\appendix
\section{Table of Notation}
\label{app:not}
$$
\begin{array}{ll}
\hline \text { Symbol } & \text {Explanation } \\
\hline \mathscr{D} & \text {The space of all CDFs} \\
\mathscr{D}(a,b) & \text {The space of all CDFs supported on $[a,b]$} \\
\mathscr{D}_M & \text {The space of all CDFs supported on $[0,M]$} \\
B_{\infty}(F,c) & \text {The $\norm{\cdot}_{\infty}$ norm ball centered at $F$ with radius $c$ } \\
\delta_c & \text {the step function with parameter $c$}  \\
L_M & \text {The  Lipschitz constant of EntRM w.r.t. $\infty{\cdot}_{\infty}$ over $\mathscr{D}_M$ } \\
\boldsymbol{\rO_{c}^{\infty}} & \text {The optimism operator  w.r.t. $\norm{\cdot}_{\infty}$ with coefficient $c$ } \\
\mathcal{M} & \text {MDP instance } \\
\mathcal{S} & \text {finite state space} \\
\mathcal{A} & \text {finite action space} \\
r_h & \text {deterministic function at step $h$} \\
S & \text {number of states } \\
A & \text {number of actions } \\
H & \text {Number of time-steps per episode } \\
K & \text {Number of episodes } \\
Z^{\pi}_h(s,a) & \text {return of $(s,a)$ at step $h$ with policy $\pi$} \\
Y^{\pi}_h(s) & \text {return of $s$ at step $h$ with policy $\pi$} \\
Z^{*}_h(s,a) & \text {optimal return of $(s,a)$ at step $h$} \\
Y^{*}_h(s) & \text {optimal return of $s$ at step $h$} \\
\eta^{\pi}_h(s,a) & \text {distribution of $Z^{\pi}_h(s,a)$} \\
\nu^{\pi}_h(s) & \text {distribution of $Y^{\pi}_h(s)$} \\
\eta^{*}_h(s,a) & \text {distribution of $Z^{*}_h(s,a)$} \\
\nu^{*}_h(s) & \text {distribution of $Y^{*}_h(s)$} \\
Q^{\pi}_h(s,a) & \text {EntRM value of $Z^{\pi}_h(s,a)$} \\
V^{\pi}_h(s) & \text {EntRM value of $Y^{\pi}_h(s)$} \\
Q^{*}_h(s,a) & \text {EntRM value of $Z^{*}_h(s,a)$} \\
V^{*}_h(s) & \text {EntRM value of $Y^{*}_h(s)$} \\
J^{\pi}_h(s,a) & \text {EU value of $Z^{\pi}_h(s,a)$} \\
W^{\pi}_h(s) & \text {EU value of $Y^{\pi}_h(s)$} \\
J^{*}_h(s,a) & \text {EU value of $Z^{*}_h(s,a)$} \\
W^{*}_h(s) & \text {EU value of $Y^{*}_h(s)$} \\
\cH^k_h & \text {history up to step $h$ of episode $k$} \\
\cF_k & \text {history up to episode $k-1$} \\
\sA & \text {RL algorithm } \\
\pi & \text {policy} \\
N^k & \text {visiting count} \\
\hat{P}^k & \text {empirical transition function in episode $k$} \\
\cT & \text {distributional Bellman operator} \\
\Pi & \text {projection operator} \\
(x_1,x_2;p) & \text {a distribution taking values $x_1,x_2$ with probability $1-p$ and $p$} \\
(x;p) & \text {a discrete distribution with $\bP(X=x_i)=p_i$.} \\
|\eta| & \text {the number of atoms of the distribution $\eta$} \\
\hline
\end{array}
$$

\section{Missing Proofs}
\label{app:pf}
\subsection{Missing Proofs in Section \ref{sec:rodi_mf}}
\textbf{Proof of Lemma \ref{lem:eu_lip}}
\begin{proof}
	We first provide the proof for the case $\beta >0$. For any $F,G\in\mathscr{D}_M$, without loss of generality we assume $\int_0^M G(x)d\exp(\beta x)-\int_0^M F(x)d\exp(\beta x) \ge0$, otherwise we switch the order.
	\begin{align*}
	&|E_{\beta}(F) - E_{\beta}(G)| = \abs{\int_0^M \exp(\beta x)dF(x)-\int_0^M \exp(\beta x)dG(x)} \\
	&=\abs{\exp(\beta x)F(x)|_0^M-\int_0^M F(x)d\exp(\beta x)-\exp(\beta x)G(x)|_0^M+\int_0^M G(x)d\exp(\beta x) } \\
	&= \int_0^M (G(x)-F(x))d\exp(\beta x)\leq \int_0^M \abs{G(x)-F(x)} d\exp(\beta x)\\
	&\leq \norm{F-G}_{\infty}\int_0^M 1 d\exp(\beta x)=(\exp(\beta M)-1)\norm{F-G}_{\infty}.
	\end{align*}
	For the case $\beta<0$, we assume $\int_0^M G(x)d\exp(\beta x)-\int_0^M F(x)d\exp(\beta x) \ge0$. 
	\begin{align*}
	&|E_{\beta}(F) - E_{\beta}(G)| = \int_0^M (G(x)-F(x))d\exp(\beta x) =\int_0^M (G(x)-F(x)) \beta \exp(\beta x)dx \\
	&\leq \int_0^M \abs{G(x)-F(x)} |\beta|\exp(\beta x)dx  \\
	&\leq \norm{F-G}_{\infty}\int_0^M -1 d\exp(\beta x)=(1-\exp(\beta M))\norm{F-G}_{\infty}\\
	&=\abs{\exp(\beta M)-1}\norm{F-G}_{\infty}.
	\end{align*}
	Thus $L_M=\abs{\exp(\beta M)-1}$ for EU. To show the tightness of the constant, consider two scaled Bernoulli distributions $F=(1-\mu_1)\psi_0+\mu_1\psi_M$ and  $G=(1-\mu_2)\psi_0+\mu_2\psi_M$, where $\mu_1,\mu_2\in(0,1)$ are some constants. It holds that
	\begin{align*}
	\abs{E_{\beta}(F) - E_{\beta}(G)} &= \abs{\mu_1\exp(\beta M)+1-\mu_1-(\mu_2\exp(\beta M)+1-\mu_2)}\\
	&=\abs{\mu_1-\mu_2}\abs{\exp(\beta M)-1}=\norm{F-G}_{\infty}L_M,
	\end{align*}
	where the last  equality holds since $\|F-G\|_{\infty}=\abs{F(0)-G(0)}=\abs{\mu_1-\mu_2}$ (independent of $M$). More formally, we have
	\[\inf_{M>0,\beta>0}\sup_{F,G\in\mathscr{D}_M}\frac{|E_{\beta}(F) - E_{\beta}(G)|}{\|F-G\|_{\infty}}= \abs{\exp(\beta M)-1}=L_M.\]
\end{proof}
\textbf{Proof of Lemma \ref{lem:main_event}}
\begin{fact}[$\ell_1$ concentration bound, \cite{weissman2003inequalities}]
\label{fct:l1_con}
Let $P$ be a probability distribution over a finite discrete measurable space $(\mathcal{X}, \Sigma)$. Let $\widehat{P}_{n}$ be the empirical distribution of $P$ estimated from $n$ samples. Then with probability at least $1-\delta$,
$$
\left\|\widehat{P}_{n}-P\right\|_{1} \leq \sqrt{\frac{2|\mathcal{X}|}{n} \log \frac{1}{\delta}}.
$$
\end{fact}
Lemma \ref{lem:main_event} does not directly follow from a union bound together with Fact \ref{fct:l1_con} since the case $N^k_h(s,a)=0$ need to be checked.
\begin{proof}
	Fix some $(s,a,k,h)\in\mathcal{S}\times\mathcal{A}\times[K]\times[H]$. If $N^k_h(s,a)=0$, then we have $\hat{P}^k_h(\cdot|s,a)=\frac{1}{S}\textbf{1}$. A simple calculation yields that for  any $P_h(\cdot|s,a)$
	\begin{align*}
     \norm{\frac{1}{S}\textbf{1}-P_h(\cdot|s,a)}_{1}\leq2\leq\sqrt{2S\log(1/\delta)}.
	\end{align*}
	It follows that 
	\begin{align*}
	\mathbb{P} &\left(\norm{\hat{P}^k_h(\cdot|s,a)-P_h(\cdot|s,a)}_{1}\leq \left. \sqrt{\frac{2S}{N^k_h(s,a)\vee 1}\log(1/\delta)}\right| N^k_h(s,a)=0 \right) =1>1-\delta. 
	\end{align*}
	The event is true for the unseen state-action pairs. Now we consider the case that $N^k_h(s,a)>0$. By Fact \ref{fct:l1_con} , we have that for any integer $n\ge1$
	\begin{align*}
	\mathbb{P} & \left(\norm{\hat{P}^k_h(\cdot|s,a)-P_h(\cdot|s,a)}_{1}  \left.\leq\sqrt{\frac{2S}{N^k_h(s,a)}\log(1/\delta)} \right| N^k_h(s,a)=n \right)\ge1-\delta.
	\end{align*}
	Thus,
	\begin{align*}
	&\mathbb{P} \left(   \norm{\hat{P}^k_h(\cdot|s,a)-P_h(\cdot|s,a)}_{1}\leq\sqrt{\frac{2S\log(1/\delta)}{N^k_h(s,a)}} \right)\\
	&=\sum_{n=0,1,\cdots}\mathbb{P} \left(   \norm{\hat{P}^k_h(\cdot|s,a)-P_h(\cdot|s,a)}_{1} \leq \left.\sqrt{\frac{2S\log(1/\delta)}{N^k_h(s,a)\vee 1}} \right| N^k_h(s,a)=n\right)\mathbb{P}(N^k_h(s,a)=n ) \\
	&\ge (1-\delta) \sum_{n=0,1,\cdots}\mathbb{P}(N^k_h(s,a)=n )=1-\delta.
	\end{align*}
	Applying a union bound over all $(s,a,k,h)\in\mathcal{S}\times\mathcal{A}\times[K]\times[H]$ and rescaling $\delta$ leads to the result.
\end{proof}
\begin{lemma}
\label{lem:log}
    Let $0<m\leq a<b$, it holds that $\log(b)-\log(a) \leq \frac{1}{m}(b-a)$.
\end{lemma}
\subsection{Missing Proofs in Section \ref{sec:rodi-mb}}
\subsection{Missing Proofs in Section \ref{sec:lb}}
\textbf{Proof of Lemma \ref{lem:kl_bd}}
\begin{proof}
	Fix $q\in[0,1]$, let $h(p):=\operatorname{kl}(p,q)$. It is immediate that
	\begin{align*}
	h^{\prime}(p)&=\log\frac{p}{q}-\log\frac{1-p}{1-q},\\
	h^{\prime\prime}(p)&=\frac{1}{p(1-p)}>0.
	\end{align*}
	Therefore $h(p)$ is strictly convex, increasing in $(q,1)$ and decreasing in $(0,q)$. By Taylor's expansion, we have that
	\[ h(p)=h(q)+h^{\prime}(q)(p-q)+\frac{1}{2}h^{\prime\prime}(r)(p-q)^2=\frac{(p-q)^2}{2r(1-r)}\]
	for some $r\in[p,q]$ ($p<q$) or $r\in[q,p]$ ($p>q$). In particular, for any $\epsilon\ge0$ such that $q=p+\epsilon\leq\frac{1}{2}$ it follows that
	\[ \operatorname{kl}(p,p+\epsilon)=\frac{(p-q)^2}{2r(1-r)}\arrowvert_{q=p+\epsilon}=\frac{\epsilon^2}{2r(1-r)}\leq\frac{\epsilon^2}{2p(1-p)}\leq\frac{\epsilon^2}{p},\]
	where the first inequality follows from the fact that $r\mapsto r(1-r)$ is increasing in $[p,p+\epsilon]\subset[0,\frac{1}{2}]$ and the second inequality is due to that $1-p\ge\frac{1}{2}$.
\end{proof}
\begin{lemma}
	\label{lem:kl_bd}
	If $\epsilon\ge0$, $p\ge0$ and $p+\epsilon\in[0,\frac{1}{2}]$, then $\operatorname{kl}(p,p+\epsilon)\leq \frac{\epsilon^2}{2p(1-p)}\leq\frac{\epsilon^2}{p}$.
\end{lemma}
\begin{fact}[Lemma 1, \cite{garivier2019explore}]
	\label{fct:fund_inq}
	Consider a measurable space $(\Omega, \mathcal{F})$ equipped with two distributions $\mathbb{P}_{1}$ and $\mathbb{P}_{2}$. For any $\mathcal{F}$-measurable function $Z: \Omega \rightarrow[0,1]$, we have
	$$
	\mathrm{KL}\left(\mathbb{P}_{1}, \mathbb{P}_{2}\right) \geq \mathrm{kl}\left(\mathbb{E}_{1}[Z], \mathbb{E}_{2}[Z]\right),
	$$
	where $\mathbb{E}_{1}$ and $\mathbb{E}_{2}$ are the expectations under $\mathbb{P}_{1}$ and $\mathbb{P}_{2}$ respectively.
\end{fact}

\begin{fact}[Lemma 5, \cite{domingues2021episodic}]
	\label{fct:div_dec}
	Let $\mathcal{M}$ and $\mathcal{M}^{\prime}$ be two MDPs that are identical except for their transition probabilities, denoted by $ P_{h}$ and $P_{h}^{\prime}$, respectively. Assume that we have $\forall(s, a)$, $P_{h}(\cdot \mid s, a) \ll P_{h}^{\prime}(\cdot \mid s, a).$ Then, for any stopping time $\tau$ with respect to $\left(I_k\right)_{k \geq 1}$ that satisfies $\mathbb{P}_{\mathcal{M}}[\tau<\infty]=1$
	$$
	\mathrm{KL}\left(\mathbb{P}_{\mathcal{M}}, \mathbb{P}_{\mathcal{M}^{\prime}}\right)=\sum_{(s,a,h) \in\mathcal{S}\times\mathcal{A}\times[H-1]} \mathbb{E}_{\mathcal{M}}\left[N^{\tau}_h(s,a)\right] \mathrm{KL}\left(P_{h}(\cdot \mid s, a), P_{h}^{\prime}(\cdot \mid s, a)\right).
	$$
\end{fact}
\subsection{Missing Proofs in Section \ref{sec:dis}}
\textbf{Proof of Proposition \ref{prop:v_com}}
\begin{proof}
Recall that
\[ \eu{x,P}=\sum_{i\in[n]} e^{\beta x_i}P_i =\brk{P \circ e^{\beta x}} \]
We can prove the above inequalities by induction. We only show the proof for $\beta>0$. Assume that
\[ V'_{h+1}\ge V''_{h+1}\ge \erm{\tilde{\nu}_{h+1}}\ge V_{h+1}\ge V^*_{h+1}. \]
$(a) \Longleftarrow$ induction $V'_{h+1}>V''_{h+1}$ \\
$(b) \Longleftarrow$  $b''_h=|e^{\beta (H+1-h)}-1|c_h<|e^{\beta H}-1|c_h=b'_h$ \\
$(c) \Longleftarrow$   $\erm{\tilde{\nu}_{h+1}(s)}\leq V''_{h+1}(s)$ for all $s\in\cS$ is equivalent to $\eu{\tilde{\nu}_{h+1}(s)} \leq e^{\beta V''_{h+1}(s)}$. Given the linearity of EU, we have
\begin{align*}
    \eu{ \brk{\hat{P}_h \tilde{\nu}_{h+1}} }-\brk{\hat{P}_h e^{\beta V''_{h+1}}}=\brk{\hat{P}_h \eu{\tilde{\nu}_{h+1}}}-\brk{\hat{P}_h e^{\beta V''_{h+1}}}\leq0.
\end{align*}
On the other hand, 
\begin{align*}
    \eu{O_{c_h}\prt{ \brk{\hat{P}_h \tilde{\nu}_{h+1}} }}-\eu{ \brk{\hat{P}_h \tilde{\nu}_{h+1}} } &\leq L(\text{EU},H-h) \norm{ O_{c_h}\prt{ \brk{\hat{P}_h \tilde{\nu}_{h+1}} } - \brk{\hat{P}_h \tilde{\nu}_{h+1}} }_{\infty} \\
    &\leq L(\text{EU},H-h) c_h \leq b''_h.
\end{align*}
Therefore,
\begin{align*}
    &\eu{O_{c_h}\prt{ \brk{\hat{P}_h \tilde{\nu}_{h+1}} }}\leq \eu{ \brk{\hat{P}_h \tilde{\nu}_{h+1}} }+b''_h   \leq \brk{\hat{P}_h e^{\beta V''_{h+1}}}+b''_h  
    \\
    &\Longrightarrow \erm{\tilde{\eta}_h} \leq \frac{1}{\beta}\log\prt{\brk{\hat{P}_h e^{\beta V''_{h+1}}}+b''_h} 
\end{align*}
$(d) \Longleftarrow$ Since $\norm{\brk{\tilde{P}_h \tilde{\nu}_{h+1}}-\brk{\hat{P}_h \tilde{\nu}_{h+1}}}_{\infty}\leq\norm{\tilde{P}_h-\hat{P}_h}_1\leq c_h$, we have $O_{c_h}\prt{ \brk{\hat{P}_h \tilde{\nu}_{h+1}} }\succeq\brk{\tilde{P}_h \tilde{\nu}_{h+1}}$.  $\erm{\tilde{\nu}_{h+1}(s)}\ge V_{h+1}(s)$ for all $s\in\cS$ implies $\eu{\tilde{\nu}_{h+1}(s)} \ge e^{\beta V_{h+1}(s)}$.
\begin{align*}
    \eu{\tilde{\eta}_h}&=\eu{O_{c_h}\prt{ \brk{\hat{P}_h \tilde{\nu}_{h+1}} }} \ge \eu{ \brk{\tilde{P}_h \tilde{\nu}_{h+1}} } \\
    &\ge  \brk{\tilde{P}_h \circ\eu{\tilde{\nu}_{h+1}} } \ge \brk{\tilde{P}_h \circ e^{\beta V_{h+1}}}\\
    &= \eu{V_{h+1},\tilde{P}_h}.
\end{align*}
$(c+d) \Longleftarrow$
\begin{align*}
\brk{\prt{\tilde{P}_h-\hat{P}_h} e^{\beta V''_{h+1}}} &= \eu{V''_{h+1},\tilde{P}_h}-\eu{V''_{h+1},\hat{P}_h} \\
&\leq L(\text{EU},H-h)\norm{\prt{V''_{h+1},\tilde{P}_h}-\prt{V''_{h+1},\hat{P}_h}}_{\infty}\\
&\leq L(\text{EU},H-h)\norm{\tilde{P}_h-\hat{P}_h}_{1} \\
&\leq |e^{\beta (H+1-h)}-1|c_h
\end{align*}
$(e) \Longleftarrow$
\begin{align*}
    \brk{\tilde{P}_h e^{\beta V_{h+1}}} \ge \brk{P_h e^{\beta V_{h+1}}} \ge \brk{P_h e^{\beta V^*_{h+1}}}
\end{align*}
Observe that
\begin{align*}
    Q'_h&=\min\cbrk{H+1-h,r_h+\frac{1}{\beta}\log\prt{\brk{\hat{P}_h e^{\beta V'_{h+1}}}+b'_h}} \\
    &\ge \min\cbrk{H+1-h,r_h+\frac{1}{\beta}\log\prt{\brk{\hat{P}_h e^{\beta V''_{h+1}}}+b''_h}}=Q''_h \\
    &\ge \erm{\tilde{\eta}_h} \\
    &\ge \min\cbrk{H+1-h,r_h+\frac{1}{\beta}\log\prt{\brk{\tilde{P}_h e^{\beta V_{h+1}}}}}=Q_h,
\end{align*}
which implies that
\begin{align*}
    V'_h=\max_a Q'_h(\cdot,a) \ge \max_a Q''_h(\cdot,a)=V''_h \ge \max_a\erm{\tilde{\eta}_h(\cdot,a)}=\erm{\tilde{\nu}_h} \ge \max_a Q_h(\cdot,a)=V_h.
\end{align*}
\end{proof}

\section{Additional Property of EntRM}
\label{app:property}
We state some  lemmas about the monotonicity-preserving property and their proofs here. The results hold for general risk measures satisfying the monotonicity-preserving property. 

\begin{lemma}
\label{lem:mp_cor}
    Let $\rho$ be a risk measure satisfying \textbf{(I)}. For any $F$ and $G$ such that $\rho(F)<\rho(G)$ and $0\leq\theta^{\prime}<\theta\leq1$,
    \[ \rho(\theta F+(1-\theta)G) < \rho(\theta^{\prime} F+(1-\theta^{\prime})G). \]
\end{lemma}
\begin{proof}
    Let $\tilde{\theta}=\frac{\theta^{\prime}}{\theta^{\prime}+1-\theta}\in[\theta^{\prime},\theta]$ and $\bar{\theta}=\theta-\theta^{\prime}\in(0,1]$. It holds that
    \begin{align*}
        \theta F+(1-\theta)G = \bar{\theta} F+(1-\bar{\theta})(\tilde{\theta} F+(1-\tilde{\theta})G)\\
        \theta^{\prime} F+(1-\theta^{\prime})G = \bar{\theta} G+(1-\bar{\theta})(\tilde{\theta} F+(1-\tilde{\theta})G).
    \end{align*}
    The result follows from \textbf{(I)}
    \[ \rho(\bar{\theta} F+(1-\bar{\theta})(\tilde{\theta} F+(1-\tilde{\theta})G)) < \rho(\bar{\theta} G+(1-\bar{\theta})(\tilde{\theta} F+(1-\tilde{\theta})G)). \]
\end{proof}

\begin{lemma}
\label{lem:ind_mul}
	Let $\rho$ be a risk measure satisfying \textbf{(I)} and $n\ge2$ be an arbitrary integer.  If $\rho(F_i)\ge\rho(G_i), \forall i \in [n]$ (and $\rho(F_j)\neq\rho(G_j)$ for some $j \in [n]$) then $\rho\left(\sum^n_{i=1} \theta_i F_i\right)\ge(>)\rho(\sum^n_{i=1} \theta_i G_i)$ for any $\theta \in \Delta_n$ (and $\theta_j \neq0$).
\end{lemma}
\begin{proof}
	The proof follows from induction. Note that $\sum^n_{i=1} \theta_i F_i=\theta_1 F_1 +(1-\theta_1)\sum^n_{i=2}\frac{\theta_i}{1-\theta_1}F_i$ and $\sum^n_{i=2}\frac{\theta_i}{1-\theta_1}F_i \in \mathscr{D}$,
	therefore by Lemma  \ref{lem:mp_cor} we have $\rho(\sum^n_{i=1} \theta_i F_i) \ge \rho(\theta_1 G_1 +\sum^n_{i=2} \theta_i F_i)$. Suppose that for some $k\in[n-1]$ it holds that $\rho(\sum^n_{i=1} \theta_i F_i) \ge \rho(\sum^k_{i=1}\theta_i G_i +\sum^n_{i=k+1} \theta_i F_i)$. Since
	\begin{align*}
	\sum^k_{i=1}\theta_i G_i +\sum^n_{i=k+1} \theta_i F_i &= \theta_{k+1}F_{k+1}+\sum^k_{i=1}\theta_i G_i +\sum^n_{i=k+2} \theta_i F_i \\
	&=\theta_{k+1}F_{k+1}+(1-\theta_{k+1})\left[\sum^k_{i=1}\frac{\theta_i}{1-\theta_{k+1}} G_i +\sum^n_{i=k+2} \frac{\theta_i}{1-\theta_{k+1}} F_i\right]
	\end{align*} and $\frac{1}{1-\theta_{k+1}}\brk{\sum^k_{i=1}\theta_i G_i +\sum^n_{i=k+2} \theta_i F_i} \in \mathscr{D}$, it follows that
	$$
	\rho\prt{\sum^n_{i=1} \theta_i F_i} \ge \rho\prt{\sum^k_{i=1}\theta_i G_i +\sum^n_{i=k+1} \theta_i F_i} \ge\rho\prt{\sum^{k+1}_{i=1}\theta_i G_i +\sum^n_{i=k+2} \theta_i F_i}.
	$$
	The induction is completed. If in addition  $\rho(F_j)>\rho(G_j)$ for some $j\in[n]$,  the proof follows analogously by replacing the inequality to the strict inequality and the fact that $\theta_j>0$. 
\end{proof}

\begin{lemma}[Monotonicity-preserving under pairwise transport]
	\label{lem:smp}
	Let $\rho$ be a risk measure satisfying the monotonicity-preserving property. Suppose $n\ge2$ and  $ (F_i)_{ i \in [n]}$ satisfies $\rho(F_1)\leq\rho(F_2)...\leq\rho(F_n)$. For any $\theta,\theta^{\prime} \in \Delta_n$ and any $1\leq i <j \leq n$ such that
	\[ \begin{cases}
	\theta^{\prime}_i\leq\theta_i, \\
	\theta^{\prime}_j\ge\theta_j,  \\
	\theta^{\prime}_k=\theta_k, \quad k \neq i,j
	\end{cases}\]
	It holds that $\rho(\sum^n_{i=1} \theta_i F_i) \leq \rho(\sum^n_{i=1} \theta^{\prime}_i F_i)$.
\end{lemma}
\begin{proof}
	Observe that
	\begin{align*}
	\sum^n_{k=1} \theta^{\prime}_k F_k &= \theta^{\prime}_i F_i + \theta^{\prime}_j F_j +\sum_{k\neq i,j} \theta^{\prime}_{k} F_k
	= \theta^{\prime}_i F_i + \theta^{\prime}_j F_j + \sum_{k\neq i,j}\theta_{k} F_k \\
	&= (\theta^{\prime}_i F_i + \theta^{\prime}_j F_j) + (1-\theta_i -\theta_j)\sum_{k\neq i,j}\theta_{k} F_k.
	\end{align*}
	By Lemma \ref{lem:mp_cor}, it suffices to prove $\rho(\frac{1}{\theta_i +\theta_j}(\theta^{\prime}_i F_i + \theta^{\prime}_j F_j)) \ge \rho(\frac{1}{\theta_i +\theta_j}(\theta_i F_i + \theta_j F_j))$. The result follows from the definition and the fact that $\rho(F_i)\leq\rho(F_j)$ and $\theta^{\prime}_i\leq\theta_i$.
\end{proof}
\begin{lemma}[Monotonicity-preserving under block-wise transport]
	\label{lem:smp_mul}
Suppose $n\ge2$ and  $ (F_i)_{ i \in [n]}$ satisfies $\rho(F_1)\leq\rho(F_2)...\leq\rho(F_n)$. It holds that $\rho(\sum^n_{i=1}\theta_i F_i)\leq\rho(\sum^n_{i=1}\theta^{\prime}_i F_i)$ for any $\theta,\theta^{\prime}\in\Delta_n$ satisfying $\exists k\in[n], \theta^{\prime}_i\leq\theta_i $ if $i\leq k$ and $\theta^{\prime}_i\ge \theta_i $ otherwise.
\end{lemma}
\begin{proof}
Fix $k\in[n]$.  We rewrite the assumption imposed to $\theta^{\prime}$  as $\theta^{\prime}_i=\theta_{i}-\delta_i$ for $i\leq k$ and $\theta^{\prime}_i=\theta_{i}+\delta_i$ for $i> k$, where each $\delta_i\ge0$. It will be shown that there exists a sequence $\{\theta^l\}_{l\in[k]}$ satisfying $\theta^0=\theta$ and $\theta^k=\theta^{\prime}$ such that $\rho(\theta^l)\leq\rho(\theta^{l+1})$, then the proof shall be completed. 
	
sequence is constructed as follows: at the $l$-th iteration, we transport  probability mass $\delta_l$ of $\theta_l$ to the probability mass of $k+1,...,n$. Specifically, we start from moving to the least number $i_l\ge i_{l-1}$ that satisfy $\theta^{l-1}_{i_l} < \theta^{\prime}_{i_l}$ and sequentially move to the  next one if there is remaining mass. The iteration stops until all the mass $\delta_l$ are transported. Repeating the procedure for $k$ times we obtain $\theta^k=\theta^{\prime}$. The inequality $\rho(\theta^l)\leq\rho(\theta^{l+1})$ for each iteration follows from Lemma \ref{lem:smp}.
\end{proof}
Recall that the distributional optimism operator $\rO_c^1:\sD(\cS)\mapsto \sD(\cS)$ over space of PMFs with level $c$ and future return $\nu\in\sD^{\cS}$ as 
\[	\rO_c^1\prt{\widehat{P},\nu}\triangleq	\arg\max_{P\in B_1(\widehat{P},c)}U_{\beta}([P\nu]).\]
By Lemma \ref{lem:smp_mul}, $\rO_c^1\prt{\widehat{P},\nu}$ can be computed as follows
\begin{itemize}
    \item sort $\nu$ in the ascending order such that $U_{\beta}(\nu^1)\leq U_{\beta}(\nu^2)\cdots \leq U_{\beta}(\nu^S)$
    \item permute $\hat{P}$ in the order of $\nu$
    \item move  probability mass $\frac{c}{2}$ of the first $S-1$ states sequentially to the $S$-th state
\end{itemize}
The computational complexity of the three steps are $O(S\log(S))$, $O(S)$, and $O(S)$. Therefore the computational complexity of applying  $\rO_c^1$ in Line 6 of Algorithm \ref{alg:RODI-MB} is only $O(S\log(S))$.

\vskip 0.2in
\bibliography{sample}

\end{document}